\documentclass[letterpaper,11pt]{article}

\usepackage[margin=1in]{geometry} 
\usepackage{changepage}
\usepackage{graphicx} 
\usepackage{natbib} 
\usepackage{amssymb,amsmath,amsthm}
\usepackage{xspace} 
\usepackage{paralist}
\usepackage{xcolor}
\usepackage{multirow}

\usepackage[hidelinks]{hyperref}

\newcommand{\eps}{\varepsilon}

\usepackage{mathtools}
\DeclarePairedDelimiter\parens{\lparen}{\rparen}
\DeclarePairedDelimiter\abs{\lvert}{\rvert}
\DeclarePairedDelimiter\norm{\lVert}{\rVert}

\DeclarePairedDelimiter\braces{\lbrace}{\rbrace}

\newcommand\Var{\mathbb{V}}

\newcommand\opt{\mathit{OPT}}
\newcommand\R{\mathbb{R}}
\newcommand\RR{\mathcal{R}}
\newcommand\PP{\mathcal{P}}
\newcommand\QQ{\mathcal{Q}}

\DeclareMathOperator{\poly}{poly}
\DeclareMathOperator{\polylog}{polylog}

\DeclareMathOperator{\vol}{vol}

\let\Pr\relax\DeclareMathOperator*{\Pr}{\mathrm{Pr}}
\let\Var\relax\DeclareMathOperator*{\Var}{\mathbb{V}}

\DeclareMathOperator*{\E}{\mathbb{E}}

\DeclareMathOperator{\diam}{diam}
\DeclareMathOperator{\erfc}{erfc}

\renewcommand{\epsilon}{\varepsilon}

\allowdisplaybreaks[1]

\usepackage{microtype}
\usepackage{graphicx}
\usepackage{subfigure}
\usepackage{booktabs}

\usepackage{algorithmic}
\usepackage{algorithm}

\usepackage{amsmath}
\usepackage{amssymb}
\usepackage{mathtools}
\usepackage{amsthm}
\usepackage{thm-restate}
\usepackage{enumitem}
\usepackage[capitalize,noabbrev]{cleveref}

\renewcommand{\citet}{\cite}

\usepackage{microtype}
\usepackage{graphicx}
\usepackage{subfigure}
\usepackage{booktabs}
\usepackage{paralist}

\usepackage{amsmath}
\usepackage{amssymb}
\usepackage{mathtools}
\usepackage{amsthm}

\usepackage[capitalize,noabbrev]{cleveref}

\theoremstyle{plain}
\newtheorem{theorem}{Theorem}[section]
\newtheorem{proposition}[theorem]{Proposition}
\newtheorem{lemma}[theorem]{Lemma}
\newtheorem{corollary}[theorem]{Corollary}
\theoremstyle{definition}
\newtheorem{definition}[theorem]{Definition}

\theoremstyle{remark}

\usepackage[textsize=tiny]{todonotes}

\begin{document}

\title{Optimal Dimension-Free Sampling for Regularized Classification 
}

\author{
   Meysam Alishahi
   \thanks{University of Utah, USA; 
} 
   \and
   Alexander Munteanu
   \thanks{TU Dortmund, Germany; A.M. was supported by the German Research Foundation (DFG) - grant MU 4662/2-1 (535889065) and by the TU Dortmund - Center for Data Science \& Simulation (DoDaS).
} 
   \and
   Simon Omlor
   \thanks{TU Dortmund, Germany; S.O. was supported by the German Research Foundation (DFG) - project no. 535889065.
} 
   \and
   Jeff M. Phillips
   \thanks{University of Utah, USA; J.M.P. supported by NSF 2115677, 2311954, and 2421782, and Simons Foundation
MPS-AI-00010515.
}
}

\date{\today}

\maketitle

\begin{abstract} 
We prove optimal sampling bounds achieving $(1\pm\varepsilon)$-relative error for a broad class of Lipschitz continuous classification loss functions under various regularization terms. This includes important functions such as logistic and sigmoid loss, hinge loss, and ReLU loss, as prominent and popular representative examples.
In particular, we prove $k^2/\varepsilon^2$ upper and lower bounds for $\|\cdot\|_2/k$ regularization, and $k/\varepsilon^2$ upper and lower bounds for $\|\cdot\|_1/k$ regularization. For $\|\cdot\|_2^2/k$ regularization, the sampling complexity depends mainly on a bounded derivative property: if $|g'(x)|\leq g(x)$, and $g(0)>0$, and $g$ is monotonic or convex, then it admits linear in $k$ sampling complexity; otherwise the general bound is $k^2/\varepsilon^2$. However, if $g(0)=0$, our results indicate that no dimension-free bounds are possible, and even sublinear bounds are ruled out. All upper bounds are complemented by matching lower bounds up to polylogarithmic terms.
Moreover, our work relies conceptually and algorithmically on simple uniform or (squared) norm sampling and hereby improves over recent cubic $k^3/\varepsilon^2$ sensitivity sampling bounds of  (\citeauthor{AlishahiP24}, ICML'24). This is achieved by refined arguments involving higher moment bounds and empirical process analyses to avoid overcounting that appears in the de-facto standard VC-dimension and sensitivity framework.
\end{abstract}

\clearpage

\tableofcontents

\clearpage

\section{Introduction}
\label{sec:intro}

Sample complexity of linear classifiers is one of the most central questions in machine learning. It quantifies how many samples (observations) of data are needed to achieve certain guarantees in accuracy.  
These bounds assume data $(Z,y) \sim \mathcal{P}$ drawn {i.i.d.} from an unknown distribution $\mathcal{P}$ with $Z = \{z_1, \ldots, z_m\} \subset \R^d$ and labels $y \in \{-1,+1\}^m$. We simplify to row vectors $a_i = y_i z_i \in \R^d$, and say $a \sim \mathcal{P}$.
Our ultimate goal is to approximate a function $f(x) = \int g(\langle a, x \rangle)\,d \mathcal P(a) + \RR(x)/k$, with probability $1-\delta$ within a $(1\pm\eps)$-relative approximation, for $\eps,\delta > 0$. Thereby, $g(\cdot)$ is a monotonic, Lipschitz, and often convex loss function, such as hinge, sigmoid, ReLU, or logistic loss, and $\RR(x)/k$ is a regularization term with $\RR(x) \in \{ \|x\|_2^2, \|x\|_2, \|x\|_1\}$.
We aim to do this with as few weighted samples $\{a_i\}_{i=1}^m$ as possible from a properly designed distribution $\QQ$, and reweighting them with weights $\{w_i\}_{i=1}^m$ to obtain $f_S(x)=\frac{1}{m}\sum_{i=1}^m w_i g(\langle a_i, x \rangle) + {\RR(x)}/{k}$.  The objective is the following guarantee to hold with probability $1-\delta$:
\begin{equation*}
    \forall x \in \mathbb{R}^d\colon \left|f(x) - f_S(x) \right| \leq \eps f(x)\,.
\end{equation*}
A special case is uniform sampling, where $\QQ =\PP$ and all weights are set to $w_i=1$.
Setting $f(x) = f_0(x) + \RR(x)/k$, since the regularization appears in both $f(x)$ and its approximation $f_S(x)$, the guarantee is equivalent to 
\begin{equation}\label{eq:mainapprox}
    \forall x \in \mathbb{R}^d\colon \left|f_0(x)-\left(\frac{1}{m}\sum_{i=1}^m w_i g(\langle a_i, x \rangle) \right) \right| \leq \varepsilon f(x),
\end{equation}
which will be used prevalently throughout our paper. We now give an overview of classic and modern related work before getting to our contribution.

\paragraph{Classic related work.}
Classically, for 0-1 loss (which counts the number of misclassifications), let $\xi$ be the Bayes' optimal rate for linear classifiers; which is, for the best choice of linear classifier, the probability a new data point will be misclassified by it.  
In this setting, the core bounds were established by \citet{vapnik2015uniform} to depend on $\eps$, the tolerated misclassification rate above $\xi$; on 
$\delta$, the probability of exceeding that $\xi+\eps$ rate; and 
on the VC-dimension, which for linear classifiers is $d+1$ \citep{KearnsV94}. The tight bound, shown by \citet{li2001improved} requires 
\begin{align*}
    m = \Theta\left(\frac{1}{\eps^2}\Big(d + \log(1/\delta)\Big)\right) \text{ samples.} 
\end{align*}
However, in practice one solves for a linear classifier not in this 0-1 setting, but by a proxy monotonic, often convex, loss function $g(\langle a, x \rangle)$, such as hinge, sigmoid, ReLU, or logistic loss. Partially towards this question \citet{bartlett2002rademacher} developed the use of Rademacher complexity, and then \citet{bartlett2005local} established a bound that instead achieves an \emph{additive $\eps$-approximation of the true loss function} itself (on which gradient descent is run). Their bound also depends on a maximum radius $\|a\|_2 \leq D$, control on the Lipschitz constant $L$ of the loss function $g$, an upper bound of $\|x\|_2 \leq R$ on the scale of the normal direction used in the loss function ($x$ are the learned parameters).
This sort of bound uses
\begin{align*}
m = \Theta\left(\frac{1}{\eps^2}\Big(L^2 D^2 R^2 + \log(1/\delta)\Big)\right)  \text{ samples.}
\end{align*}
The $L$-Lipschitz property is standard (and basically disallows loss functions from growing too quickly) and a scale parameter $D$ is necessary when sampling uniformly from an unknown distribution, otherwise a single point can affect the results significantly. While $\|x\|_2 \leq R$ leads to a bounded loss function, it is not the way practitioners most commonly formulate and solve linear classification. Moreover, while these bounds are independent of the dimension $d$, they only provide an additive error. Thus it is not scale-free, and means as gradient descent approaches the minimum, the relative error induced by this analysis increases.

\paragraph{Contemporary related work.}
More recently, through the perspective of coresets,
this problem was revisited by \citet{munteanu2018coresets}; and they achieved a \emph{relative $(1\pm\eps)$-approximation} bound. It also removed the need to bound the radius of data $D$ and scale of parameters with $R$, but instead introduced a data dependent parameter $\mu(A)$. This $\mu(A)$ roughly captures how separable the data is. They introduced a sampling bound, using VC-dimension arguments, of $m = \tilde{O}(\mu(A)^3 d^3 / \eps^4)$, which was improved to $m = \tilde O(\mu(A)^2 d/ \eps^2)$ by \citet{MaiMR21}, and $m = \tilde{O}(\mu(A) d^3 / \eps^2)$ by \citet{MunteanuOP22}, 
and finally to linear $m=\tilde O(\mu(A) d/\eps^2)$ in \citep{MunteanuO24}. Note that these coreset approaches do not sample uniformly from the unknown distribution, but rather assume there is already a large enough representative sample, that is further reduced to the target size. Moreover, they also do not subsample uniformly from that larger sample. Instead they sample using importance sampling via square-root leverage scores, $\ell_1$ Lewis weights, or a mixture of $\ell_1$ and $\ell_2$ leverage scores, and re-weight the sampled points for their use in the approximation.
Also, this reintroduces the dependence on the dimension $d$, which while natural from VC-dimension theory, may be seen as unusual in the context of modern billion parameter models \citep{ShoeybiEtAl19,BrownMRSKDNSSAA20,NarayananSCLPKV21}, or whole genome data analysis \citep{TeschkeIM24}.

Additional work from the coreset perspective, \citet{SamadianPMIC20} and \citet{TolochinskyJF22} both consider a variant that adds a regularization term $\RR(x)/k$ for $\RR(x) \in \{ \|x\|_2^2, \|x\|_2, \|x\|_1\}$. As a result, no bound is required on the norm of $x$. Instead, it is constrained implicitly as part of the loss function. While \citet{TolochinskyJF22} drew a sample proportional to the leverage scores, \citet{SamadianPMIC20} only need a uniform sample; using a VC-dimension-based argument they require $O(dk \log k/\eps^2)$ samples (assuming $D = 1$).  
After this, \citet{AlishahiP24} revived the Rademacher complexity approach, and this led to dimension-free bounds of $m = O(k^3/\eps^2)$ (for $D = 1$); so the dependence on $k$ has increased, but there is no more dependence on the dimension $d$.

\paragraph{In this paper we resolve the sample complexity of regularized linear classifiers} \!\!\!\!(up to $\polylog(k)$ factors), i.e. the number of samples that are required to achieve the guarantee given in \Cref{eq:mainapprox}.
To this end, we leverage the weighted sampling scheme of coresets. Our weights are simple, corresponding to a balanced mixture of uniform and norm-sampling, essentially proportional to $1 + \|a\|_2$, when $\int \|a\|_2\,d\PP(a)= B<\infty$. The intuition is that points of high norm tend to dominate the approximation error. The addition of norm-sampling allows us to control their contribution.
Further if in addition $\|a\|_2 \leq D$ then uniform i.i.d. sampling from $\PP$ provides the same asymptotic upper bounds when the average norm $B$ is replaced by the maximum norm $D$ in the sample complexity. 

Our paper will demonstrate the sample complexity story is more intricate than has been revealed up to this point. The tight answer depends on a combination of both the loss function and the regularization term.  
Our first results apply widely to loss functions $g$ that are only required to be $L$-Lipschitz and obtain $k^2/\eps^2$ upper bounds. We then improve to linear bounds of $k/\eps^2$ in settings that do not have quadratic lower bounds. For a few combinations of loss and regularizer, we need to restrict further to loss functions with bounded derivative $g'$ (specifically $|g'(r)| < g(r)$ for all $r\in \R$), in cases where this is necessary to improve quadratic upper bounds to $k/\eps^2$.
As applications of these bounds and for proving lower bounds, we choose more specific examples of highly popular loss functions for regression and classification tasks that are also widely studied in the aforementioned previous work:
\begin{compactitem}
    \item logistic loss: $g(r)=\ln(1+\exp(-r))$,
    \item sigmoid loss: $g(r)=1/(1+\exp(r))$,
    \item hinge (or SVM) loss: $g(r)=\max\{0,1-r\}$,
    \item ReLU loss: $g(r)=\max\{0,-r\}$.
\end{compactitem}
We note that all these functions are monotonic and all except sigmoid loss are convex.

Our sampling complexity bounds (which we state formally in Section \ref{sec:results}) will also depend on the choice of regularization, and we consider $\RR(\cdot)\in\{\norm{\cdot}_2,\norm{\cdot}_2^2,\norm{\cdot}_1\}$, which correspond to the arguably most popular norm-, ridge-, and LASSO-regularization. In each case except for ReLU with $\RR(x) = \|x\|_2^2$, we show the bound must be either $m= \tilde \Theta(k^2/\eps^2)$ or $m = \tilde \Theta(k / \eps^2)$.
Ultimately our bounds  will be tight up to poly-logarithmic factors in $k$ for each combination. For ReLU with $\RR(x) = \|x\|_2^2$, we show that surprisingly $\Omega(n \log n)$ samples are required, i.e., no dimension-free, nor sublinear sample size are possible.

\paragraph{Further related work.}\label{app:furtherrelatedwork}
Our work combines Rademacher complexity \cite{AlishahiP24} with empirical processes, by deriving higher moment bounds on a Rademacher process \cite{CohenP15,MaiMR21} and using the classic higher moment method \cite{AlonS08,ClarksonW09,Woodruff14}.
Further, our work continues a recent series of works that achieve optimal linear complexities (whenever possible) through empirical process analyses of sensitivity sampling for $\ell_p$-subspace embeddings \citep{woodruffyasuda23,MunteanuO24}, $k$-means/median clustering \citep{BansalCPSS24}, and sparsification of sums of semi-norms, and generalizations of linear models \citep{JambulapatiLLS23,JambulapatiLLS24}.

\section{Sampling, algorithms, and applications}

\paragraph{Sampling distribution.}
The sampling distribution we analyze is proportional to a score function $s(a)$, with $S=\int s(a)\,d\mathcal P$. We then adjust the results via importance or rejection sampling. Specifically, we sample in our paper from a mixture of this distribution with an equal proportion of uniform sampling.  
More formally, our technical derivations focus on the importance sampling version where we set $d\mathcal Q = (\frac{s}{2S} + \frac{1}{2})\,d\mathcal P=(\frac{s+S}{2S})\,d\mathcal P$ and the weights become $w_i=\frac{2S}{s(a_i)+S}$. The function $s(a)$ will be chosen to be the (squared) Euclidean norm which augments the uniform part and will allow us to control data points with high Euclidean norm, as these tend to dominate the approximation error. We comment on how to realize this distribution in the paragraphs below.

\paragraph{PAC learning with relative error.}
Under the assumption that $\Pr_{a\sim\PP}(\|a\|_2 > D)=0$, which (as discussed before) is common and arguably required in all uniform i.i.d. sample complexity literature, our techniques also work for this classic PAC learning setting where we can only draw samples uniformly from $\PP$; we extend it by giving \emph{multiplicative} rather than the standard additive guarantees. Our claim is formalized and corroborated in \Cref{cor:uniformsampling}.

\paragraph{Sampling algorithms.}
Returning to sampling from probability distribution $\QQ$, the simplest way to realize this, is through rejection sampling. I.e. we take $m$ uniform samples $a_i\sim \PP$, and reject each with probability $\frac{1}{2}-\frac{s(a_i)}{2\hat S}$.
However, implementing this online requires prior knowledge of $\hat S=\sum_{j=1}^m s(a_j)$ (or a $(1\pm\eps)$-approximation), which is not always available.

Alternatives have been developed in data stream algorithms: in reservoir sampling one only needs to feed the current $s(a_i)$ to the sampler, which in our case is simply $s(a_i)=\|a_i\|_2+1$ (or similar), and the algorithm samples proportionally to $s(a_i)/\sum_{j=1}^i s(a_j)$. At the end of the sampling process we have $m$ samples with distribution $s(a_i)/\hat S$, where $\hat S = \sum_{j=1}^m s(a_j)$ \cite{chao82,Efraimidis15}, and $\hat S/m = (1\pm\eps) S$ for sufficiently large $m$ (see \Cref{lem:S_estimate,lem:w_estimate}). Another alternative (that even admits distributions over additive single-coordinate updates rather than insertion of entire points) is turnstile row-norm sampling \citep{Mahabadi2020,MunteanuO24turnstile}. 
Using either of the above streaming sampling algorithm as black box, we can realize $\QQ$ by taking a uniform sample $a\sim \PP$, and then tossing a coin. If it turns heads, we keep it right away. If it turns tails, we feed it to the sampling algorithm. The final sample is the union of the direct samples and the ones that are accepted by the sampling black box.

\paragraph{Finite data or coresets.}
In the \emph{coreset} setting, \citep{Phillips17,MunteanuS18},
we are given a finite amount of data, and we can simply calculate $S$ (or estimate it using standard streaming algorithms \cite{Woodruff14}); then our theorems apply directly without additional measures. They thus yield optimal coresets for a variety of regularized classification loss functions improving over previous works in this setting \citep{SamadianPMIC20,TolochinskyJF22,AlishahiP24}.

\paragraph{Distributional setting.}
In the distributional setting we only have access to samples from the unknown distribution $\PP$, and we require additionally $P(\|a\|_2 > D)=0$ for some term $D<\infty$.  
As discussed before, this is not needed for sampling from $\QQ$, but for setting the weights, it is required to estimate $S$. Since the final estimate of $S$ will be a $(1\pm\eps)$-approximation, we actually feed $2s(a)$ to the sampler. The total sample size suffers an \emph{additive} $O(D\log(1/\delta)/\eps^{2})$ compared to our main theorems \Cref{thm:norm_sampling,thm:boundedderivativethm,thm:generallonethm}, see \Cref{lem:S_estimate,lem:w_estimate}.
We note that $S = B+1$. Then 
if $D = O(B)$, this is a lower order term in sample size.

\paragraph{Kernels and bounded spaces.}  
Most data distributions naturally have some bound $B$, or $D$ on the (average) norm of data points. If for no other reason, then values are standardized or normalized before learning. Layer-normalization \cite{ba2016layer} is, for instance, a very common element of neural networks which leads to $B=D=1$.
This is also made explicit in kernel methods \cite{ScholkopfS01,hofmann2008kernel} 
where a kernel support-vector-machine (SVM) is mathematically equivalent to linear classification operating in a (perhaps infinite-dimensional) kernel Hilbert space $\mathcal{H}_K$.
Similarly, support vector data descriptor (SVDD) \citep{TaxD04} is equivalent to the kernel smallest enclosing ball problem \citep{KrivosijaM19}. For a wide variety of common kernels $K(a,x)$ (e.g., 
normalized angular $K(a,x) = \langle a,x \rangle / (\|a\| \|x\|)$; 
Gaussian $K(a,x) = \exp(-\|a-x\|^2/\sigma^2)$; 
Laplace $K(a,x) = \exp(-\|a-x\|/\sigma)$), then again $\|a\|_{\mathcal{H}_K} = K(a,a) = 1$, and hence we have $B=D=1$. When reproducing kernels are used, then this space $\mathcal{H}_K$ has no bound on the dimension, and a dimension-free bound (such as ours) is necessary for a finite sample complexity bound. In all of these cases, our analysis directly applies with uniform random sampling.

\section{Our results}
\label{sec:results}

We refer to \Cref{tab:results} for an overview of our contributions when applied to popular example functions.

\renewcommand{\arraystretch}{1.175}
\begin{table*}[b]
  \caption{\textbf{Overview of our results.} All bounds have $\varepsilon^{-2}$ dependence, and $\widetilde k \coloneqq k(\log k)^3\log\log k$. The best and only previous dimension-free result was $\tilde O(k^3/\eps^2)$ by \citet{AlishahiP24}.
    }
  \label{tab:results}
    \vskip -0.1in
  \begin{center}
      \begin{sc}
        \begin{tabular}{l|cc|cc|cc|cc}
        \toprule
        $\RR(\cdot)$  & \multicolumn{2}{c|}{logistic} & \multicolumn{2}{c|}{sigmoid} & \multicolumn{2}{c|}{hinge} & \multicolumn{2}{c}{ReLU} \\
        
        ~ & upper & lower & upper & lower & upper & lower & upper & lower \\
        \midrule
        $\norm{\cdot}_2^2$ & $\widetilde k$ & $k/\log k$ & $\widetilde k$ & $k/(\log k)^3$ & $k^2$ & $k^2$ & - & $n\log n$ \\
        $\norm{\cdot}_1$ & $\tilde k$ & $k\log k$ & $\tilde k$ & $k/\log k$ & $\tilde k$ & $k\log k$ & $\tilde k$ & $k\log k$ \\
        $\norm{\cdot}_2$ & $k^2$ & $k^2$ & $k^2$ & $k^2$ & $k^2$ & $k^2$ & $k^2$ & $k^2$ \\
        \bottomrule
    \end{tabular}
      \end{sc}
  \end{center}
  \vskip -0.1in
\end{table*}

\subsection{Upper bounds}
\paragraph{Quadratic upper bounds.}
Our first bound can be viewed as a refinement (and in our view, simplification) of that by \citet{AlishahiP24}; it is allowed by the choice of sampling proportional to $\|a\|_2+1$, and the employment of more powerful bounds. We first symmetrize the approximation error, leading to an upper bound in terms of a canonical Rademacher process. Then we leverage Lipschitz contraction bounds of \citet{LT1991} and a special variant of Khintchine's inequality due to \citet{Kahane:CRASP-259-2577} to derive a higher moment bound on the Rademacher process. This leads to our approximation result by employing the higher moment method \citep{AlonS08}, i.e., an application of Markov's inequality to a higher power of the random variable.

\begin{restatable}{theorem}{normsampling}
\label{thm:norm_sampling}
    Let $f(x)\coloneqq\int g(\langle a,  x\rangle)\,d\mathcal{P}(a)+\RR(x)/k$  such that $g$ is $L$-Lipschitz. Let $s(a) \geq \norm{a}_2 + 1$ for all $a \in \mathbb{R}^d$. Let $T\subseteq \R^d$, and $\opt_T=\inf_{x\in T} f(x)$. Set $C=(SL)^2$.
    There exists
    \begin{align*}
        m_0 = \begin{cases}
            O(C \varepsilon^{-2} k \ln(\delta^{-1}) / \opt_T), \quad\;\; 
            \text{ for } \RR(\cdot)=\norm{\cdot}_2^2 \\
            O(C \varepsilon^{-2} k^2 \ln(\delta^{-1})), \quad\;\; \text{ for } \RR(\cdot)\in\{\norm{\cdot}_2,\norm{\cdot}_1\},
        \end{cases}
    \end{align*}
    such that if we sample $m\geq m_0$
    points according to $\mathcal{Q}$, then with probability at least $1- \delta$ we have
\[
    \forall x \in T \colon \left|f_0(x)-\frac{1}{m}\sum_{i=1}^m w_i g(\langle a_i , x \rangle) \right| \leq \varepsilon f(x)\,.
\]
\end{restatable}

We note that in full generality it holds that $\opt \geq \Omega(1/k)$ (see \Cref{lem:C=opt}) if $g(0)>0$, which implies that both bounds of \Cref{thm:norm_sampling} are quadratic in $k$ in the worst case.
Surprisingly, this $k^2$ upper bound is already optimal in several cases, see \Cref{tab:results,sec:LB_main}. This improves over the $k^3$ sensitivity sampling bounds of \cite{AlishahiP24}. Also note that in the case $\RR(\cdot)=\norm{\cdot}_2^2$, we can handle subsets $T\subseteq \R^d$ with linear $k$ dependence whenever $f(x)\geq OPT_T$ is bounded below by an absolute constant for all $x\in T$.

Moreover, in contrast to sensitivity sampling that was used in previous literature to obtain dimension-free bounds \cite{AlishahiP24}, our theorem allows to set $s(a)=\|a\|_2+1$. This corresponds to row-norm sampling, which is almost as simple to implement as uniform sampling, e.g., using a weighted reservoir sampler \cite{chao82}, and does not require estimating sensitivities. Further, unlike sensitivity sampling, setting $s(a)=\|a\|_2+1$ suffices to ensure that $S = B + 1$ is independent of $k$ for our row norm sampling.

\paragraph{Linear upper bounds.}
We identify two settings for which we can conduct improved analyses to obtain linear $k$ dependence.

\textbf{Setting 1:} $\norm{\cdot}_2^2$-regularization when $g$ is either convex or concave or monotonic and more crucially has a bounded derivative, i.e., $|g'|\leq g$. This applies to our examples of logistic and sigmoid loss, which decay sufficiently slowly towards zero.

\begin{restatable}{theorem}{boundedderivative}
\label{thm:boundedderivativethm}
Let $f(x)\coloneqq\int g(\langle a,  x\rangle)\,d\mathcal{P}(a)+\Vert x \Vert_2^2/k$  such that $g$ is $L$-Lipschitz, and either convex or concave or monotonic, and for all $r \in \mathbb{R}$ it holds that $|g'(r)| \leq g(r)$. Let $s(a) \geq \norm{a}^2_2 + 2$ for all $a \in \mathbb{R}^d$. Set $C = SBL^2/g(0)$.
There exists $m_0 = O(C \varepsilon^{-2} k\ln(C \eps^{-1}k\ln(\delta^{-1}))^3\ln(\delta^{-1}\ln(SLk/g(0)))) = \tilde O(C \varepsilon^{-2} k\ln(\delta^{-1}))$, such that if we sample $m\geq m_0$ points according to $\mathcal{Q}$, then with probability at least $1-\delta$ we have that
\[
    \forall x \in \mathbb{R}^d\colon \left|f_0(x)-\left(\frac{1}{m}\sum_{i=1}^m w_i g(\langle a_i,  x\rangle)\right) \right| \leq \varepsilon f(x).
\]
If additionally $B=O(\log(k))$ holds, then choosing $C=SL/g(0)$ suffices for the same guarantee.
\end{restatable}

As argued before, it often holds that $B=O(1)$ in typical applications, which is in favor of the improved constants.

\textbf{Setting 2:} The second setting with linear $k$ dependence is for general Lipschitz functions with $\norm{\cdot}_1$-regularization.

\begin{restatable}{theorem}{generallonethm}
\label{thm:generallonethm}
Let $f(x)\coloneqq\int g(\langle a,  x\rangle)\,d\mathcal{P}(a)+\Vert x \Vert_1/k$ such that $g$ is $L$-Lipschitz. Let $s(a) \geq \norm{a}_2+1$ for all $a \in \mathbb{R}^d$. Set $C = SL$.
There exists $m_0 = O(C \varepsilon^{-2} k\ln(C \eps^{-1}k\ln(\delta^{-1}))^3\ln(\delta^{-1}\ln(\eps^{-1}BLk))) = \tilde O(C \varepsilon^{-2} k\ln(\delta^{-1}))$, such that if we sample $m\geq m_0$ points according to $\mathcal{Q}$, then with probability at least $1-\delta$ it holds for all ${x\in \{x\in \R^d \mid f(x)\leq g(0)/\eps\}}$ that 
\[
    \left|f_0(x)-\left(\frac{1}{m}\sum_{i=1}^m w_i g(\langle a_i, x\rangle)\right) \right| \leq \varepsilon f(x) \,.
\]
If $g$ is homogeneous, then it is sufficient to set $m_0 = O(C \varepsilon^{-2} k\ln(\delta^{-1})\ln(C \eps^{-1}k\ln(\delta^{-1}))^3)$ to get the same result for all $x\in\R^d$ simultaneously.
\end{restatable}
We comment on the restriction to $x\in \{x\in \R^d \mid f(x)\leq g(0)/\eps\}$ in the case of a non-homogeneous function $g$. First, this set contains and thus allows to approximate the optimal value, since $\opt\leq f(0)=g(0)$. Second, all four example functions can be split into $g(r)=h(r)+b(r)$, where $h(r)$ is homogeneous, and $b(r)$ is bounded by $g(0)$. Whenever such a decomposition exists, then the guarantee of \cref{thm:generallonethm} extends to all $x\in\R^d$ (see \Cref{lem:convergencel1}).

Both the above improvements are achieved via refined empirical process analyses. The aforementioned symmetrization argument relates the relative error up to an additional small constant to a canonical Gaussian process. The Gaussian process induces a pseudo-metric, which can be bounded in terms of the \emph{metric entropy} $\mathcal E$ and the \emph{metric diameter} $\mathcal D$ by Sudakov's tail bound, also known as Dudley's bound \citep{Dudley2016,LT1991}. Recently, \citet{woodruffyasuda23} transformed this classic result into a higher moment bound involving the ratio of the two quantities. The remaining challenge is to bound $\mathcal E, \mathcal D$ appropriately and use the higher moment method again to obtain the final approximation via Markov's inequality.

More details on the general outline of empirical process analyses and on our specific techniques to handle the two different settings are provided in the technical overview, and the full details can be found in \Cref{sec:improvingtolinear}, \ref{sec:boundedderiveative} resp. \ref{sec:linearboundl1}.

\paragraph{Relative error PAC learning with uniform sampling.}
Our upper bound theorems apply to the classic problem under plain uniform sampling $a\sim \PP$ as special case, resolving the sample complexity, and matching our lower bounds.
\begin{restatable}{corollary}{uniformsampling}
\label{cor:uniformsampling}
Let $f(x)\coloneqq\int g(\langle a,  x\rangle)\,d\mathcal{P}(a)+\RR(x)/k$ such that $g$ is $L$-Lipschitz. Assume that $\Pr_{a\sim P}(\|a\|_2 > D) = 0$. Let $T\subseteq \R^d$. If we sample a number of $m$ points according to $\mathcal{P}$, then with probability at least $1-\delta$
\[
    \forall x\in T\colon \left|f_0(x)-\left(\frac{1}{m}\sum_{i=1}^m g(\langle a_i, x\rangle)\right) \right| \leq \varepsilon f(x) \,,
\]
holds under the same restrictions on $T$, conditions on $g$, and sizes of $m$, as stated in \Cref{thm:norm_sampling,thm:generallonethm}, with every occurrence of $S$ replaced by $D$. In \Cref{thm:boundedderivativethm}, $S$ must be replaced by $D^2$.
\end{restatable}
\begin{proof}
The result follows directly by setting $s(a)=D+1\geq \|a\|_2 + 1$, resp. $s(a)=D^2+2\geq \|a\|_2^2 + 2$ and calling
\Cref{thm:norm_sampling,thm:boundedderivativethm,thm:generallonethm}.
For this choice of $s(a)$ it holds that $S=\int s(a)\,d\PP(a) = D+1$, resp. $S=D^2+2$, $\QQ(a)=(\frac{s(a)}{2S} +\frac 1 2) \PP(a) = \PP(a)$, and $w_i = \frac{2S}{s(a_i)+S} = 1$, which corresponds exactly to uniform sampling of $a\sim\PP$.
\end{proof}

\subsection{Lower bounds}\label{sec:LB_main}
\paragraph{Quadratic lower bounds.}
As mentioned before when treating the quadratic $k^2/\eps^2$ upper bound, it is already optimal for $\norm{\cdot}_2$-regularization, indicating that this regularizer is the weakest among our examples.
However, $\norm{\cdot}_2^2$-regularization also exposes weaknesses if the function $g$ is not well-behaved (in particular if its derivative is not uniformly bounded by its value $|g'|\leq g$ as in \Cref{thm:boundedderivativethm}). In fact, the combinations with hinge and ReLU loss functions do not have this property that seems required for better than quadratic sampling complexity, as our quadratic lower bounds match our corresponding upper bounds:
\begin{restatable}{theorem}{thmquadraticlowerbounds}
    \label{thm:allquadraticlowerbounds}
    Consider any of the following combinations of loss functions and regularizers:
    \begin{compactitem}
      \item logistic/sigmoid with $\RR(\cdot)=\norm{\cdot}_2$, or
      \item hinge/ReLU with $\RR(\cdot)\in\{\norm{\cdot}_2,\norm{\cdot}_2^2\}$\,.
    \end{compactitem}
    There exists a distribution $\mathcal P$ over vectors $a\in\R^d$ such that $m = \Omega (k^2/\eps^2)$ samples are required to satisfy \Cref{eq:mainapprox}.
\end{restatable}

\paragraph{Linear lower bounds.} We also complement our linear sampling complexity results by matching lower bounds.

\begin{restatable}{theorem}{thmlinearlowerbounds}\label{thm:alllinearlowerbounds}
    Suppose we want to satisfy \Cref{eq:mainapprox} for some combination of loss function and regularizer.
    In each of the following settings there exists a distribution $\mathcal P$ over vectors $a\in\R^d$ such that the stated lower bounds on the number of samples $m$ hold:
    \begin{compactenum}
        \item For either ReLU/hinge/logistic loss with $\RR(\cdot)=\norm{\cdot}_1$ regularization,
        $m = \Omega (\eps^{-2} k\log k)$ is required,
        \item For either sigmoid loss with $\RR(\cdot)=\norm{\cdot}_1$, or logistic loss with $\RR(\cdot)=\norm{\cdot}_2^2$ regularization,
        $m = \Omega (\eps^{-2} k/\log k)$ is required,
        \item For sigmoid loss with $\RR(\cdot)=\norm{\cdot}_2^2$ regularization,
        $m = \Omega (\eps^{-2} k/(\log k)^3)$ is required.
    \end{compactenum}
\end{restatable}

\paragraph{Impossibility results.}
Finally, recall that our bound in \Cref{thm:norm_sampling} for handling $\norm{\cdot}_2^2$-regularization is sensitive to the optimal value $\opt$ and even \Cref{thm:boundedderivativethm} is sensitive to the value of $g(0)$, and it is required that these values are bounded away from $0$. This, however, is not the case with the ReLU function since $f(0)=\operatorname{ReLU}(0)+\|0\|_2^2/k=g(0)=0$, which makes the upper bound infinite. We corroborate this phenomenon by two lower bounds. The first is conceptually simple and yields $\Omega(d\log d)$, showing impossibility of dimension-free sampling complexity. The second extends and refines the idea and yields $\Omega(n\log n)$, for a distribution $\PP$ with arbitrarily large support size $n \geq d+1$.

\begin{restatable}{theorem}{thmReLUsquared}\label{thm:ReLUsquared}
    Consider the ReLU function where $g(r)=\max\{0,-r\}$, and the $\RR(x)=\|x\|_2^2$ regularizer.
    For any $n\geq d$ there exist distributions supported on $n$ points in $\R^d$, such that with high probability $m=\Omega(n\log n)$ samples are required to satisfy \Cref{eq:mainapprox}.
\end{restatable}

We note that $O(n)$ is clearly sufficient \emph{if} the support is known, for instance in a finite data setting. However, if we can access data items only through sampling from an unknown distribution it is known that with high probability $\Omega(n \log n)$ draws are necessary to see all $n$ distinct items by the coupon collector's theorem \citep{ErdosR61}.

\section{Technical overview}
\subsection{Upper bounds}
Even though none of our analyses are conducted in this way, we would like to point the interested reader to \Cref{lem:cuttingparameters}, which may be of independent interest. It states that for $\RR(x)\in\{\norm*{\cdot}_1,\norm*{\cdot}_2\}$, the analysis can be performed for any $k'\in \mathbb{N}$ assuming $B'=L'=1$. The resulting bounds work with arbitrary values $B,L,k$ just by substituting $k'=LBk$.

All upper bounds start with a standard symmetrization \citep[cf.][]{CohenP15}, that relates the relative error
$$E \coloneqq\sup_{x \in \mathbb{R}^d}\left|f_0(x)-\left(\frac{1}{m}\sum_{i=1}^m w_i g(a_i x) \right) \right|\Big /f(x)$$
we aim to bound, to the supremum of an empirical process.
I.e., up to absolute constants, we show for i.i.d. Rademacher $\sigma_i\sim U\{-1,+1\}$ or Gaussian $\sigma_i\sim N(0,1)$ variables, that
\begin{align}\label{main:error_process_bound}
    \mathbb{E}(E)
    &\leq \mathbb{E} \sup_{x \in \mathbb{R}^d}\left|\frac{1}{m}\sum_{i=1}^m  \sigma_i w_i g(a_i x)\right| \Big/ f(x) \,.
\end{align}
All remaining efforts aim at bounding the $\ell$-th moment of the empirical process by $(\eps/2)^\ell$. Using Markov's inequality and $\ell = O(\ln(1/\delta))$, such a bound allows us to conclude
    \begin{align*}
        \Pr(E \geq \varepsilon) 
        = \Pr(E^\ell \geq \varepsilon^\ell) 
        \leq {\mathbb{E}(E^\ell)}/{\varepsilon^\ell}
        \leq 2^{-\ell}
        \leq \delta\,,
    \end{align*}
known as higher moment method \cite{AlonS08}.

\paragraph{Quadratic upper bounds.}
Towards \Cref{thm:norm_sampling}, we first bound the $\ell$-th moment essentially by
\begin{equation}\label{main:momentbound}
    \mathbb{E}(E^\ell) \leq \left(\frac{SL\sqrt{\ell}}{\sqrt{m}}\cdot \sup_{x \in \mathbb{R}^d} \left(\frac{\norm{x}_2}{\opt+ \RR{(x)}/k}\right)\right)^\ell\,.
\end{equation}
The main steps in this analysis are taken by a careful application of Talagrand's Lipschitz Contraction (\Cref{lem:expsup}) followed by Cauchy-Schwartz and the Khintchine-Kahane inequalities (see \Cref{cor:khintchine}).
This gives essentially
\begin{align*}
    \E\left|\sum_{i=1}^m  \sigma_i w_i g(a_i x)\right|^\ell 
    &{\leq} \left( L \norm{x}_2 \right)^\ell\; \mathbb{E}\left( \ell\,\sum_{i=1}^m  w_i^2 \left\|a_i\right\|^2_2 \right)^{\ell/2} \,.
\end{align*}
We can see that the norm of a point has a high impact on the bound of their error contribution. Due to our choice of the sampling distribution and corresponding weights, we can control the norm of data points up to a factor $S$, i.e., $w_i \|a_i\|_2 \leq \frac{S}{\|a_i\|_2}\cdot \|a_i\|_2 = S$. It then remains to bound the supremum term in \Cref{main:momentbound}. We get in \Cref{lem:supbound} that
\[
    \sup_{x \in \mathbb{R}^d} \left(\frac{2\norm{x}_2}{\opt+ \RR{(x)}/k}\right)
        \leq
        \begin{cases}
            \sqrt{k/\opt}, \text{ if }\RR(\cdot) = \norm{\cdot}_2^2\\
            2k,\; \text{ if }\RR(\cdot) \in \{ \norm{\cdot}_2,\norm{\cdot}_1\}
        \end{cases}
\]
and note that it leads to tight bounds in the general case when $\RR(\cdot) \in \{ \norm{\cdot}_2,\norm{\cdot}_2^2\}$. By plugging this into \Cref{main:momentbound} we get up to absolute constants
\begin{align*}
    \mathbb{E}( E^\ell )
    &\leq \left(\frac{(SL)^2k^{2/p}{\ell}}{{m}\opt^{{p-1}}}\right)^{\ell/2} ,
\end{align*}
where $p\in \{1,2\}$ is the exponent of the regularizer. By our choice of $m$ in \Cref{thm:norm_sampling}, it yields $\mathbb{E}( E^\ell ) \leq \left({\eps}/{2}\right)^\ell$ in both cases, which concludes the proof.

Since $\opt\geq \Omega(1/k)$, this yields only quadratic bounds in $k$ for both $p\in \{1,2\}$. But for $\RR(\cdot) = \norm{\cdot}_2^2$ it is useful in special cases where $OPT\geq\Omega(1)$, and leaves hope for a refined analysis. For $\RR(\cdot) = \norm{\cdot}_1$, the bound is an immediate consequence of the inter-norm inequality $\|x\|_1\geq \|x\|_2$, which leaves sufficient gap for further improvement.

\paragraph{Linear upper bounds.} ~\\[-10pt]

\textbf{Setting 1:} For $\RR(\cdot) = \norm{\cdot}_2^2$, \Cref{thm:norm_sampling} is tight in general. In \Cref{thm:boundedderivativethm}, we improve it for monotonic, or convex functions that in particular do not decay too quickly towards zero; formally, we require that $|g'(r)|\leq g(r)$ for all $r\in \R$.

We bound the error akin to \Cref{main:error_process_bound} by a Gaussian process, and use the framework of empirical process analysis. Sudakov's resp. Dudley's tail bound \citep{Dudley2016,LT1991} was recently reformulated by \citep{woodruffyasuda23} to obtain a higher moment bound 
\begin{equation}\label{main:dudleymomentbound}
     \E[\sup\nolimits_{t\in T}\abs{X_t}^\ell] \leq (2\mathcal E)^\ell (\mathcal E/\mathcal D) + O(\sqrt \ell \mathcal D)^{\ell}
\end{equation}
in terms of the \emph{metric entropy} $\mathcal E$ and the \emph{metric diameter} $\mathcal D$ of a pseudo-metric $d_X$ induced by the Gaussian process $(X_t)_{t\in T}$. It thus remains to bound $\mathcal E, \mathcal D$ appropriately. To this end, the analysis is partitioned into slices $T_\alpha=\{x\in \R^d \mid \alpha \leq f(x) \leq 2\alpha\}$. First fix a specific value, say $\alpha=1$. The diameter for $T_\alpha$ can essentially be related to the approximation ratio $G=\sup_{x\in T_\alpha} f_S(x)/f(x)$ and the sensitivity of our samples, i.e.,
$$\mathcal D = \sup\limits_{s, t\in T_\alpha} d_X(s,t) \leq G\alpha\left(\frac{1}{m} \sup_{x \in T_\alpha, j \in [m]} \frac{w_j g(a_j x)}{f(x)}\right)^{1/2};$$
and we thus bound the sensitivities by $SBL^2k/g(0)$. Given the common case that $B=O(\log(k))$, we also derive another bound that yields $SLk/g(0)$, which is linear and optimal in the problem specific constants, not only in $k$.

To bound the metric entropy, we first relate the metric $d_X$ to an $\ell_\infty$ norm $\|\cdot\|_X$ in the subspace of our weighted samples 
\begin{align*}
    d_X(s,t)\leq \left(G\alpha \max\nolimits_{j \in [m]} w_j |\langle a_j , s-t \rangle|^2  L/m \right)^{1/2} .
\end{align*}
This step relies on either monotonicity or convexity, and in particular it requires the additional assumption that $|g'(r)|\leq g(r)$ for all $r\in \R$. So this works for logistic or sigmoid loss but we will see later that our examples without this assertion (hinge, ReLU) have $k^2$ lower bounds.

Next, we bound the Gaussian complexity $\E_{g\sim\mathcal N(0,I_d)}\norm*{g}_X$ of the norm $\norm*{\cdot}_X$. Using that $\max_{j\in[m]} w_j |\langle a_j , g \rangle| \leq {S} \max_{j\in[m]} |\langle \frac{a_j}{\norm*{a_j}_2} , g \rangle|$, it reduces to bounding the largest in a collection of $m$ standard half-Gaussians which is known to be in the order of $O(\sqrt{\log(m)})$. Then, Sudakov's dual-minorization bound \citep[cf.][]{BLM1989} yields an upper bound on the number of $\|\cdot\|_X$-balls (thus $d_X$-balls) of large radii $t>1/m$ to cover the current set $T_\alpha$, i.e., 
\begin{align*}
    \log E(T_\alpha, \norm*{\cdot}_X, t)\leq {O(G^2\alpha^2 SLk\ln(m)/m)}/{t^2}\,.
\end{align*}
Due to the $1/t^2$ dependence, it is better to use a standard $t$-net construction for small radii $0<t<1/m$. This yields an $O(m\log(1/t))$ dependence, i.e., 
\begin{align*}
        \log E(T_\alpha, \norm*{\cdot}_X, t)\leq O(m \ln(t^{-1}G\alpha\sqrt{SLk\ln(m)/m})).
\end{align*}
This enables us to calculate the entropy integral by splitting at $\lambda=1/m$, and using the two bounds for $0<t\leq\lambda$ resp. $\lambda\leq t\leq \mathcal D$. We can cut off the calculation at the diameter, because for larger radii the covering number is trivially $1$ and its logarithm becomes $0$.
This argumentation yields up to absolute constants and polylogarithmic terms that
\begin{align*}
    \mathcal E = &\int_0^\infty \sqrt{\log E(T_\alpha, \norm*{\cdot}_X, t)}dt 
    \leq \tilde O(G\alpha\sqrt{SLk/m})\,.
\end{align*}
Choosing $m$ carefully, as stated in \Cref{thm:boundedderivativethm}, balances both our bounds on $\mathcal D$ and $\mathcal E$ to be in $O(\eps G \alpha)$, so we can use the higher moment method on \Cref{main:dudleymomentbound} to obtain the desired approximation via Markov's inequality. To complete the argument, we union bound over all $T_\alpha$ for $1\geq \alpha \geq \opt \geq \Omega(1/k)$ in powers of two. For the remaining $x\in T=\{x\in\R^d\mid f(x)\geq 1\}$ we invoke \Cref{thm:norm_sampling}. Since $\opt_T = \inf_{x\in T} f(x) \geq 1$, it yields a linear bound in $k$.\\

\textbf{Setting 2:} For ${\RR(\cdot) = \norm{\cdot}_1}$ we do not need further assumptions to get a linear bound in $k$. The outline is similar to the case $\RR(\cdot) = \norm{\cdot}_2^2$ and we discuss only the most significant changes to obtain \Cref{thm:generallonethm}.
Using the relationship between the diameter and sensitivity, we need to bound the sensitivities, which are $O(SLk)$ in the new setting and yield
$$\mathcal D = \sup\nolimits_{s, t\in T_\alpha} d_X(s,t) \leq G\alpha\left({SLk}/{m} \right)^{1/2}.$$

Instead of embedding the metric into a norm, we have the square root of the  $\ell_\infty$-norm. Unlike the previous analysis, this requires squaring the radii before applying Sudakov's bounds.
For this not to harm the complexities in subsequent steps, we need to counteract the squaring by proving a bound on the covering numbers that is only linear in $\log(m)/t$, i.e., 
\begin{align*}
        \log E(B_1, B_\infty(A), t) &
        \leq O(S)\log(m)/t
        \,,
\end{align*}
by chaining the calculation into 
$$\log E(B_1, B_2(V^T), \sqrt{t})+\log E(B_2(V^T), B_\infty(A), \sqrt{t}),$$
where $B_p(A)=\{x\in\R^d\mid \|Ax\|_p \leq 1\}$ and $V^T\in \R^{m\times d}$ is an orthonormal basis for the subspace spanned by $A\in \R^{m\times d}$.
We apply Sudakov's minorization bound again, but with radius $\sqrt{t}$, to the second term.
For the first term, we employ Sudakov's dual minorization. We need to bound the Gaussian complexity $\E_{g\sim\mathcal N(0,I_m)}\norm*{Vg}_X$, which na{\"i}vely results in a prohibitive $O(\sqrt{\log(d)})$.
To circumvent this, we leverage the fact that we work in an $m$ dimensional subspace and the additive variance property of squared Gaussians, to aggregate the rows $v_i$ resp. their leverage scores $\|v_i\|_2^2$ into a matrix $W\in\R^{2m\times m}$ where each row norm is $\|w_i\|_2^2\leq 1$. Notably, this is the \emph{opposite} to flattening lemmas \citep[cf.][]{woodruffyasuda23} and yields $\E_{g\sim\mathcal N(0,I_m)}\norm*{Vg}_X \leq \E_{g}\norm*{Wg}_X \leq O(\sqrt{\log(m)})$. With the new bounds on covering numbers, and incorporating the additional factors of our embedding of $d_X$ into $\|\cdot\|_X$, the metric entropy integral calculation works exactly as before.

Choosing $m$ carefully, as stated in \Cref{thm:generallonethm}, balances both our bounds on $\mathcal D$ and $\mathcal E$, so we can use the higher moment method on \Cref{main:dudleymomentbound} to conclude the error bound via Markov's inequality. We complete the argument again by a union bound over all $T_\alpha$, for $g(0)/\eps \geq \alpha \geq \opt \geq \Omega(1/k)$ in powers of two.

The analysis is so far restricted to $x\in \{x\in \R^d \mid f(x)\leq g(0)/\eps\}$, but we cannot invoke \Cref{thm:norm_sampling} since it would yield $k^2$ dependence. 
However, this set contains and thus allows to approximate the optimal value, since $\opt\leq f(0)=g(0)$.
Furthermore, consider our four example functions: ReLU is homogeneous, for which analyzing one fixed value, say $\alpha = 1$ suffices to extend to all $x\in \R^d$ simply by scaling. All other three functions can be split into $g(r)=h(r)+b(r)$, where $h(r)$ is homogeneous, and $b(r)$ is bounded by $g(0)$. Whenever such a decomposition applies, we argue in \Cref{lem:convergencel1}, that the guarantee of \cref{thm:generallonethm} extends to all $x\in\R^d$ since the error can be separated by the triangle inequality. The homogeneous part works for $x\in\R^d$ as explained in the case of ReLU and the bounded part has a maximum error of $O(g(0))$ which can be charged for $\eps f(x)$, since we need this special treatment only for $x\in\R^d$ that obey $f(x)\geq g(0)/\eps$.
Sigmoid loss is bounded in $[0,1]$, so $h(x)=0$ and $b(x)=g(x)$. Logistic and hinge loss, both allow for $h(x)=ReLU(x)$, and $b(x)=g(x)-ReLU(x)\in [0,g(0)]$ \citep[cf.][]{MaiMR21}.

\subsection{Lower bounds}
\label{app:lowerbounds_techoverview}
\paragraph{Quadratic lower bounds.}
Our constructions for logistic and sigmoid are similar up to constants, and hold only against $\RR(\cdot) = \norm{\cdot}_2$ regularization. The input is a uniform distribution supported on $n=\Theta((k/\eps)^2)$ standard basis vectors $e_i$. Any subsample of $m\leq n/2$ elements will miss at least half of the support; let them be indexed by $i\in [n/2]$. The adversarial query is $x=\sum_{i=1}^{n/2} e_i$. Then the original loss is bounded by a constant, say $g(0)/2$, and the regularizer contributes $\|x\|_2/k = \Theta(\sqrt{n}/k) \approx 1/\eps$, so $f(x)=\frac{g(0)}{2}+\frac{1}{\eps}$. However, since our sample misses the indexes of the query, regularization is preserved, but the main loss grows to a significantly larger constant $g(0)>(1+\eps)g(0)/2$. We then show that if the approximation of $f_S(0)\approx \frac{1}{m}\sum w_i \in [1-\eps, 1+\eps]$, then the approximation of $f_S(x)>(1+\eps)f(x)$, thus fails on either of $x' \in \{0,x\}$.

Our lower bounds against hinge and ReLU loss work for both $\RR(\cdot) \in\{\norm{\cdot}_2,\norm{\cdot}_2^2\}$, however, for ReLU with $\RR(\cdot) = \norm{\cdot}_2^2$ we will see worse results later. The construction is very similar, but fixes one index to $1$, so $a_i = e_n + e_i/\sqrt{2}$, and accordingly $x=e_n + \sum_{i=1}^{(n-1)/2} e_i/\sqrt{2(n-1)}$ are scaled to constant norm. The regularizer is roughly $1/k$ and our construction ensures that the original loss is roughly $2\eps/k$, while the approximated loss drops to $0$. But then the approximation error is only the original loss, i.e., $|f_S(x)-f(x)| = |f(x)| = 2\eps/k > \eps (1+2\eps)/k = \eps f(x)$, which concludes \cref{thm:allquadraticlowerbounds}. In particular, note how this exploits the fact that hinge and ReLU drop to $0$ with slope $1$, and thus \emph{do not} have the nice property $|g'(r)|\leq g(r)$ as in \Cref{thm:boundedderivativethm}, which allows linear in $k$ upper bounds for logistic and sigmoid loss with $\RR(\cdot) = \norm{\cdot}_2^2$ regularization.

\paragraph{Linear lower bounds.}
All linear in $k$ lower-bounds follow the same high level template. We build a distribution supported on $N = \Theta(k)$ carefully chosen points. Then for each point $a_j$ we build an adversarial query $x^{(j)}$ whose loss is highly sensitive to the contribution of that single point. Every query vector essentially tests whether the sampler estimated the contribution of the $j$-th point accurately. Concretely, an i.i.d. importance sampler induces a multinomial count $N_j$ for each point $a_j$. Since the sampling scores  $s(\cdot)$ are normalized, by Markov's inequality at least a constant fraction of supported points must have comparatively small score (e.g. $s_j\le 2$), hence their expected counts satisfy $\mu_j=\mathbb{E}[N_j] = O(m/k)$. 
The query $x^{(j)}$ is designed so the failure occurs whenever we have a a constant-factor deviation of $N_j$ from its mean. By Feller's inequality (\Cref{lem:feller}), an anti-concentration bound, each deviation occurs at least with probability roughly $p = \exp\left(-\varepsilon^2\Omega(m/k)\right)$. Leveraging negative association of multinomial coordinates, the probability that \emph{none} of these deviations happens over $\Omega(k)$ disjoint indices is at most $(1-p)^{\Omega(k)}$. 
Requiring overall failure probability at most $\delta$ forces $p \lesssim \delta/k$, which yields the stated  $\Omega(\varepsilon^{-2}k\log k)$-type sample lower bounds.
For ReLU loss, we place mass on the signed basis vectors $\{\pm e_j\}_{j=1}^k$ and test at $x=-\sigma e_j$, which isolates
the contribution of a single point and yields the $\Omega(\varepsilon^{-2}k\log k)$ bound for $\norm*{\cdot}_1/k$ regularization.
For hinge and logistic with $\norm*{\cdot}_1/k$, we transfer this bound via a reduction from the ReLU problem (see \Cref{lem:reduction}): as the dot product grows, both (scaled) losses converge to ReLU, and with homogeneous regularizers, which are not affected by scaling, this allows us to reduce from the ReLU hard instance (\Cref{thm:alllinearlowerbounds} part (1)).
For logistic with $\norm*{\cdot}_2^2/k$ 
and sigmoid with $\norm*{\cdot}_1/k$ or $\norm*{\cdot}_2^2/k$, we use a $k$-point construction $e_i+e_{k+1}\in \R^{k+1}$ (for $i\in[k]$) together with a family of queries $x^{(j)}=\alpha(-2e_j+e_{k+1})$ that make the dot product of one point $-\alpha$ and the remaining points have $+\alpha$. Choosing $\alpha=\Theta(\log(k))$ makes $g(\alpha)= 1/k$, keeping $f(x^{(j)})$ small while the estimate becomes sensitive to the single count $N_j$ providing a sufficient condition for failure.
For logistic, the exact identity $g(-\alpha)-g(\alpha)=\alpha$ amplifies deviations and yields the $\Omega(\varepsilon^{-2}k/\log k)$ bound.
For sigmoid, the symmetry $g(-\alpha)=1-g(\alpha)$ gives weaker amplification, leading to
$\Omega(\varepsilon^{-2}k/\log k)$ for $\norm*{\cdot}_1/k$ and
$\Omega(\varepsilon^{-2}k/\log^3 k)$ for $\norm*{\cdot}_2^2/k$ (\Cref{thm:alllinearlowerbounds} parts (2) and (3)).

\paragraph{Impossibility results.}
A simple $\Omega(d\log d)$ bound shows no dimension-free sampling is possible for the combination of ReLU with $\RR(\cdot)=\norm*{\cdot}_2^2$. Consider the uniform distribution over $d$ standard basis vectors in $\R^d$. Suppose our sample misses one of them, say $e_1$. If we query $x=\eta e_1$ for a tiny $\eta >0$, the loss becomes $\eta/d$, its approximation is $0$ and the regularizer is $\eta^2/k$. When we let $\eta$ be small enough, the regularizer has no effect, so $|f_S(x)-f(x)| = |f(x)| = \eta/d > \eps ( \eta/d+\eta^2/k ) = \eps f(x)$. Thus, all $d$ points need to be in the sample, and $\Omega(d\log d)$ samples are required to succeed this task by the coupon collector's theorem \citep{ErdosR61}.

Our $\Omega(n\log n)$ lower bound in \Cref{thm:ReLUsquared} refines the above idea of a \emph{scaling mismatch} near the origin and highlights a limitation of the sampling model to provide a uniform relative-error guarantee for sublinear sample size.
The proof builds a geometric hard instance by placing $n$ points $a_i\in \R^d, i\in [n]$ along the moment curve, so their convex hull is a cyclic polytope. As a result, each point can be separated from all others by a supporting hyperplane. Equivalently, for every point $a_j$ there exists a query direction $x^{(j)}$ such that $a_j$ has negative dot product while all other points have non-negative dot products. For ReLU, this makes the loss at $x^{(j)}$ depend essentially on the contribution of that single point.  
Ridge regularization does not prevent this isolation, because we can scale the query by a tiny factor $\eta>0$.
The ReLU term scales as $\Theta(\eta)$, while the regularization scales as $\Theta(\eta^2)$.
In particular, for the isolating direction $x^{(j)}$ we can choose $\eta\to 0$ so that $\eta^2 / g(\eta c_j) \to 0$, and the regularization becomes negligible compared to the (already tiny) ReLU loss. Writing $N_j$ for the number of times $a_j$ appears in the $m$ i.i.d. samples, in the limit $\eta\to 0$, the required relative-error guarantee at $x_j$ essentially reduces to $|N_j-\mathbb{E}N_j|\le \varepsilon\,\mathbb{E}N_j$. Because scores are normalized, many points receive only $O(1/n)$ sampling probability in each draw. Therefore, ensuring relative error requires a coupon collector number of samples $m=\Omega(\varepsilon^{-2}n\log n)$.

\clearpage

\bibliographystyle{plainnat}
\bibliography{references}

\newpage
\appendix
\onecolumn
\allowdisplaybreaks

\section{Preliminaries}

\subsection{Setting}

We study the problem, where we are given i.i.d. sampling access a distribution $\mathcal{P}$ over $\mathbb{R}^d$. Further we fix a loss function $g\colon \mathbb{R} \rightarrow \mathbb{R}_{\geq 0}$.
We will consider the regularized loss function $f\colon\mathbb{R}^d \rightarrow \mathbb{R}_{\geq 0}$
\[
    f(x)\coloneqq\int g(\langle a,  x\rangle)\,d\mathcal{P}(a)+\frac{\mathcal{R}(x)}{k}
\]
with regularization parameter $k \in \mathbb{N}$ and one of the following regularizers $\mathcal{R}(x) \in \{ \Vert x \Vert_2^2, \Vert x \Vert_2, \Vert x \Vert_1 \}$.
We further set $ f_0(x)\coloneqq\int g(\langle a,  x\rangle)\,d\mathcal{P}(a)$, and $g_0(r)=g(r)-g(0)$. We note that for many important functions, it holds that $g(0)\leq 1$.

For our sampling algorithms, we will define a suitable function $s\colon \mathbb{R}^d \rightarrow [0, \infty)$ such that $S=\int s(a)\,d\mathcal{P}(a) < \infty$. Given approximation and confidence parameters $\varepsilon, \delta \geq 0$, our goal is to draw a small number $m \in \mathbb{N}$ of points $a_1 , \ldots, a_m \in \mathbb{R}^d$ via a probability distribution $\mathcal{Q}$ with density $d\mathcal{Q}(a)=(\frac{s(a)}{2S} + \frac{1}{2})\, d\mathcal{P}(a) = \frac{s(a)+S}{2S}\,d\mathcal{P}(a)$ such that if we reweight the contributions of point $a_i$ by $w_i=w(a_i)=\frac{2S}{s(a_i)+S}$, then it holds with probability at least $1-\delta$ that
\[
    \forall x \in \mathbb{R}^d\colon \left|f_0(x)-\frac{1}{m}\sum_{i=1}^m w_i g(\langle a_i,  x\rangle) \right| \leq \varepsilon f(x)\,.
\]
We note that $\frac{1}{m}\sum_{i=1}^m w_i \leq 2$, since each $w_i = \frac{2S}{s(a_i)+S}\leq \frac{2S}{S}=2$.

We can draw from $\mathcal Q$ via a mixture of uniform and rejection sampling. I.e., we draw $a$ from $\mathcal P$ and throw a coin. If the coin turns heads then we keep the sample and we are done. On tails, we decide with probability $s(a)/S$ to keep the sample or otherwise we reject the sample.

\subsection{Assumptions and parameters}\label{sec:assumptions}
We make the following assumptions on the distribution $\mathcal P$ and the loss function for some parameters $L \geq 1, B\geq 0$.
\begin{enumerate}[label={[A\arabic{*}]}]
    \item $g$ is $L$-Lipschitz \label{ass:L}
    \item $\mathbb{E}_{a \sim \mathcal{P}}(\norm{a}_2) = B < \infty$ \label{ass:B}
\end{enumerate}

We note that only in one case of $\RR(x)=\|x\|_2^2$, and $g(x)$ with bounded derivatives (this will be made clear in \Cref{sec:boundedderiveative}), we will need the slightly stronger assumption that $\mathbb{E}_{a \sim \mathcal{P}}(\norm{a}^2_2)= B$. A similar assumption as Assumption~\ref{ass:B} is then implied by Jensen's inequality since $\mathbb{E}_{a \sim \mathcal{P}}(\norm{a}_2)^2\leq \mathbb{E}_{a \sim \mathcal{P}}(\norm{a}^2_2)\leq B$. Thus $\mathbb{E}_{a \sim \mathcal{P}}(\norm{a}_2)\leq \sqrt{B}$.

We note that in any case, we will use the upper bound $B\leq S$. This is valid since we will ensure that $s(a)\geq \|a\|_2+1$, resp. $s(a)\geq \|a\|_2^2+2$. So in the first case we have $S=B+1\geq B$, and in the second case, we have that either $B\geq 1$, so both $\sqrt{B}\leq B \leq B+2 = S$, or if $B<1$, then $B<\sqrt{B}<1<B+2=S$.

These fairly standard assumptions ensure that the regularization has an effect. The constants $L$ and $B$ will show up in our bounds. We note that we often have $L=B=1$. In parts of the analysis it is even possible to simply assume $L=B=1$, irrespective of the actual value and conduct the analysis with $k'=LBk$. This is shown in \Cref{lem:cuttingparameters} below to hold if $\mathcal{R}(x)$ is sub-multiplicative. It could also be adapted to work for $\mathcal{R}(x)=\Vert x \Vert_2^2$ with $\Vert x \Vert_2 \leq 1$, if $f(x)\geq c_0$ for all $x \in \mathbb{R}^d$.

\begin{lemma}\label{lem:cuttingparameters}
    Let $f(x)\coloneqq\int g(\langle a,  x\rangle)\,d\mathcal{P}(a)+\mathcal{R}(x)/k$ and assume that $g$ is $L$-Lipschitz, $\mathbb{E}_{a \sim \mathcal{P}}(\norm{a}_2)\leq B$ and $\mathcal{R}(Bx)\leq B\mathcal{R}(x)$. Then there exists a function $g_2=g/L$ that is $1$-Lipschitz, such that if
    \[
        \forall x \in \mathbb{R}^d\colon \left|\int g_2(\langle a/B,  x\rangle)\,d\mathcal{P}(a)-\frac{1}{m}\sum_{i=1}^m w_i g_2(\langle a_i/B,  x\rangle) \right| \leq \varepsilon \left(\int g_2(\langle a/B,  x\rangle)\,d\mathcal{P}(a) + \frac{\mathcal{R}(x)}{LBk}\right)\,,
    \]
    then it also holds that
    \[
        \forall x \in \mathbb{R}^d\colon \left|f_0(x)-\frac{1}{m}\sum_{i=1}^m w_i g(a_i x)\right| \leq \varepsilon f(x)\,.
    \]
\end{lemma}

\begin{proof}
    We set $g_2(x)=g(x)/L$ and assume that for all $x \in \mathbb{R}^d$ it holds that
    \[
        \left|\int g_2(\langle a/B,  x\rangle)\,d\mathcal{P}(a)-\frac{1}{m}\sum_{i=1}^m w_i g_2(\langle a/B,  x\rangle) \right| \leq \varepsilon \left(\int g_2(\langle a/B,  x\rangle)\,d\mathcal{P}(a) + \frac{\mathcal{R}(x)}{LBk}\right).
    \]
    Now, let $x \in \mathbb{R}^d$ be arbitrary.
    We set $x'=Bx$ and note that this is a one-to one-mapping.
    Then we have that
    \begin{align*}
        \left|f(x)-\frac{1}{m}\sum_{i=1}^m w_i g(a_i x)\right| 
        &= \left|\int Lg_2(\langle a/B, x'\rangle)\,d\mathcal{P'}(a)-\frac{L}{m}\sum_{i=1}^m w_i g_2(\langle a_i/B,  x'\rangle) \right| \\
        &\leq \varepsilon L \left(\int g_2(\langle a/B,  x'\rangle)\,d\mathcal{P}(a) + \frac{\mathcal{R}(x')}{LBk}\right)\\
        &\leq \varepsilon \left(\int g( a  x)\,d\mathcal{P}(a) + \frac{\mathcal{R}(x)}{k}\right).
    \end{align*}
\end{proof}

The following lemma states that we can estimate $S$ up to a $(1\pm\eps)$ factor based on a small uniform sample. This will not be needed for taking importance samples but it will only be needed for determining the weights using some $\hat S = \Theta(S)$ in the distributional setting and requires that the distribution is bounded, i.e., $\Pr(\|a\|_2 > D)=0$. We note that such assumption is also standard whenever uniform sampling from an unknown distribution is involved. We note that in a finite setting such as coresets, $S$ can be determined exactly simply by summing the (squared) row norms and does not require this lemma.
 \begin{lemma}\label{lem:S_estimate}
    Assume we have i.i.d. sampling access to an unknown distribution $\PP$ with $\Pr(\|a\|_2 > D)=0$. Set $s(a)=\|a\|_2^p+1$, for $p\in\{1,2\}$. If $m\geq \frac{D^p\ln(1/\delta)}{\eps^2}$, then we have for $S = \int s(a_i)\,d\PP(a) $ with probability $1-\delta$ that the empirical estimator satisfies
    \[
        \left| \frac{1}{m}\sum_{i=1}^m s(a_i) - S \right| \leq \eps S\,.
    \]
 \end{lemma}
 \begin{proof}
    First note that $\E_{a_i\sim \PP}\left( \frac{1}{m}\sum_{i=1}^m s(a_i) \right) = S$, and recall that $S \geq 1$. We also calculate $$\Var(s(a_i)) = \Var(\|a_i\|_2^p) \leq \E(\|a_i\|_2^{2p}) = 
    \int \|a_i\|_2^{2p}\,d\PP(a) \leq D^p\int \|a_i\|_2^p \,d\PP(a) \leq D^pS.$$
    
    By an application of Bernstein's inequality, we thus get
    \begin{align*}
        \Pr\left( \left| \frac{1}{m}\sum_{i=1}^m s(a_i) - \int s(a_i)\,d\PP(a) \right| \geq \eps S \right)
        &\leq \exp\left( -\frac{\eps^2m^2S^2}{2mD^pS +\frac{2}{3} D^p m\eps S} \right) \\
        &\leq \exp\left( -\frac{\eps^2m}{3D^p} \right) \leq \delta
    \end{align*}\qedhere
 \end{proof}

\begin{lemma}\label{lem:w_estimate}
    Let $\eps\in(0,1/2]$, and let $\hat S \in [1-\eps, 1+\eps]\cdot S$. Let $w(a) = \frac{2S}{s(a)+S}$, and $w'(a) = \frac{2\hat S}{s(a)+\hat S}$. Then it holds that $$w'(a)\in [1-\eps, 1+\eps]\cdot w(a).$$
\end{lemma}
\begin{proof}
    We let $\eta \in [-\eps, \eps]$. By assumption it holds that $\abs{\eta}\leq \eps \leq 1/2$. Now fix $\hat S = (1+\eta) S$.

    We have that
    \begin{align*}
        \frac{|w'(a) - w(a)|}{w(a)}
        &=\left|\frac{w'(a)}{w(a)} - 1\right| 
        =\left| \frac{s(a)+S}{2S} \cdot \frac{(1+\eta)2S}{s(a)+(1+\eta)S} - 1 \right| \\
        &=\left| \frac{s(a) + (1+\eta)S + \eta s(a)}{s(a)+(1+\eta)S} - 1 \right| 
        =\left| \frac{\eta s(a)}{s(a)+(1+\eta)S} \right|
        = |\eta| \cdot \left| \frac{s(a)}{s(a)+(1+\eta)S} \right|
        \leq \eps\,,
    \end{align*}
    where the last step follows since $2S \geq (1+\eta)S \geq S/2 > 0$.
\end{proof}

\subsection{Bounds on the optimal value}
Parts of our analysis will depend on the optimal value $\opt_T=\inf_{x \in T} f(x)$. In particular this will be the case when the regularizer is $\RR(x)=\norm{\cdot}_2^2$, but we will also require a lower bound for $\RR(x)=\norm{\cdot}_1$. To handle these cases, we will need the following lemmas.

\begin{lemma}\label{lem:C=opt}
    For any $T\subseteq \R^d$, let $\opt_T=\inf_{x \in T} f(x)$. It holds that $$\forall x \in T \colon f(x) \geq \frac{\opt_T+\RR(x)/k}{2}.$$
\end{lemma}
\begin{proof}
    We have for every $x\in \R^d$ that $f(x) = f_0(x) + \RR(x)/k \geq \RR(x)/k$. Moreover we have by definition that $\forall x\in T \colon f(x) \geq \opt_T$. It follows for all $x\in T$ that $({\opt_T+\RR(x)/k})/{2} \leq 2 f(x) / 2 = f(x)$.
\end{proof}

We will denote the global optimum for $T=\R^d$ by $\opt \coloneqq \opt_{\R^d}$, and note that $OPT_T\geq OPT$ holds for any $T\subseteq\R^d$.
We can derive a general bound on $\opt$ based on $g(0)$, since in this case the regularization term is $R(0)/k = 0$.

\begin{lemma}\label{lem:lb-opt}
    If $\RR(\cdot)=\norm{\cdot}_2^2$ then $\opt \geq g(0)^2/(4(LB)^2k)$.
    If $\RR(\cdot)\in\{\norm{\cdot}_2, \norm{\cdot}_1\}$ then $\opt \geq g(0)/(LBk)$.
\end{lemma}
\begin{proof}
We first bound $f_0(x)$ using a Lipschitz expansion of $g$ around $0$.  
\begin{align*}
 f_0(x) 
 &= 
 \int g(\langle a, x \rangle)\, d\PP(a) 
 \\ & \geq 
 \int \max\{0, (g(0) - L \|a\|_2\|x\|_2) \}\, d\PP(a)  & \text{Assumption \ref{ass:L}; Cauchy-Schwarz}
 \\ & \geq 
 \max\{0, g(0) - L \|x\|_2 \int \|a\|_2\, d\PP(a) \}  
 \\ & \geq 
 \max\{0, g(0) - L B \|x\|_2 \}.   & \text{Assumption \ref{ass:B}}
\end{align*}
For $\RR(x)=\|x\|_2^2$, we let $r = \|x\|_2$.  Then setting $h(r) = \max\{0, g(0) - LB r\} + r^2/k$, we have that 
\[
\opt 
= 
\inf_{x \in \R^d} f(x)
= 
\inf_{x \in \R^d} f_0(x) + \|x\|_2^2/k
\geq 
\inf_{x \in \R^d} h(\|x\|_2)
\geq
\inf_{r \in \R} h(r).
\]
Using the derivative of $h(r)$, we observe two options for the minimizing $r$, depending on how the $\max$ within $h$ evaluates. If $r \leq g(0)/(LB)$ then $h$ is minimized at $r=LBk/2$ with $h(LBk/2)=g(0)-(LB)^2k/4 \geq (LB)^2k/4$. Otherwise if $r \geq g(0)/(LB)$, then $h$ is minimized on the boundary where $r=g(0)/(LB)$, and $h(g(0)/LB) = \frac{g(0)^2}{(LB)^2 k}$. The final claimed bound of $\opt$ takes the minimum of these two options, and under the previous parameter assumptions, we get that 
\begin{align*} \label{eq:opt-bound}
\opt \geq \min\{(LB)^2k/4, g(0)^2/((LB)^2k)\} \geq g(0)^2/(4(LB)^2k). \end{align*}
For the remaining cases $\RR(\cdot)\in\{\norm{\cdot}_2, \norm{\cdot}_1\}$, we note that $\|x\|_1 \geq \|x\|_2$. It thus suffices to prove a bound for $\RR(\cdot)=\norm{\cdot}_2$. A similar analysis as above but with $h(r) = \max\{0, g(0) - LB r\} + r/k$ yields $\opt \geq g(0)/(LBk)$ in that case.
\end{proof}

\section{A quadratic upper bound depending on OPT}
In this section, we prove a higher moment bound that will turn out to be quadratic in $k$ in the worst case. However, we will see that in some settings considered in this paper, this bound is already tight and improves over previous cubic bounds.

On the other hand, for $\RR(\cdot)=\norm{\cdot}_2^2$, if we get the promise that $\opt$ is bounded below by a constant, then it yields a linear $k$-dependence, which will be helpful to handle marginal cases in our improved upper bound analysis \Cref{sec:boundedderiveative}.

\subsection{Analysis}
The proof of \Cref{thm:norm_sampling} is divided into multiple parts for the sake of a clear presentation.
Our first lemma states that the expected value of our subsampled loss function is the same as its original value.
The proof is straightforward but we include it for completeness and since it is an important piece of the puzzle.

\begin{lemma}\label{lem:expected}
    For any $x \in \mathbb{R}^d$ it holds that 
    \[
        \mathbb{E}_{a \sim \mathcal{Q}^m}\left(\frac{1}{m}\sum_{i=1}^m w_i g(a_i x) + \mathcal{R}(x)/k\right)=f(x)\,.
    \]
\end{lemma}

\begin{proof}
    First note that
    \begin{align*}
        \mathbb{E}_{A \sim \mathcal{Q}^m}\left(\frac{1}{m}\sum_{i=1}^m w_i g(a_i x)\right)
        &=\frac{1}{m}\sum_{i=1}^m \mathbb{E}_{a_i \sim \mathcal{Q}}(w_i g(a_i x))\\
        &= \mathbb{E}_{a \sim \mathcal{Q}}(w(a) g(a x))\\
        &=\int g(a x) w(a) \,d\mathcal{Q}(a)\\
        &=\int g(a x) \frac{2S}{s(a)+S} \cdot \frac{s(a)+S}{2S} \,d\mathcal{P}(a)
        =f_0(x)\,.
    \end{align*}
    Thus, we have that
    \begin{align*}
        \mathbb{E}_{A \sim \mathcal{Q}^m}\left(\frac{1}{m}\sum_{i=1}^m w_i g(a_i x) + \RR{(x)}/k\right)
        &=\mathbb{E}_{A \sim \mathcal{Q}^m}\left(\frac{1}{m}\sum_{i=1}^m w_i g(a_i x)\right)  + \RR{(x)}/k \\
        &= f_0(x) + \RR{(x)}/k
        = f(x)\, .
    \end{align*}
\end{proof}
    
    We proceed by considering the relative approximation error
    $$E\coloneqq\sup_{x \in \mathbb{R}^d}\left|f(x)-\left(\frac{1}{m}\sum_{i=1}^m w_i g(a_i x) + \mathcal{R}(x)/k\right) \right|\Big /f(x)\,.$$
    
    We first show that for a suitable choice of $\ell$, bounding the $\ell$-th moment of the error term suffices to conclude the theorem.

    \begin{lemma}\label{lem:probbound}
    Assume that $\mathbb{E}(E^\ell) \leq (\varepsilon/2)^\ell $ for some even integer $\ell \geq 2\ln(\delta^{-1})$. Then with probability at least $1-\delta$ we have that
    $E < \varepsilon$
\end{lemma}

\begin{proof}
    Using Markov's inequality we have that 
    \begin{align*}
        \Pr(E \geq \varepsilon) 
        = \Pr(E^\ell \geq \varepsilon^\ell) 
        \leq \frac{\mathbb{E}(E^\ell)}{\varepsilon^\ell}
        \leq 2^{-\ell}
        \leq \exp(-2\ln(2)\ln(\delta^{-1}))
        < \delta\,.
    \end{align*}
\end{proof}

It remains to bound the $\ell$-th moment $\mathbb{E}(E^\ell)$. Our first attempt is a Rademacher complexity bound that follows from a standard symmetrization argument involving Rademacher random variables.

\begin{lemma}\label{lem:Ebound}
    Let $\ell \in \mathbb{N}$ be an even integer. Then we have that
    \[
    \mathbb{E}( E^\ell  )\leq 2^\ell \cdot \mathbb{E}_{a \sim \mathcal{Q}^m, \sigma \sim \mathcal{U}\{-1, 1\}^m} \left(\sup_{x \in \mathbb{R}^d}\left|\frac{1}{m}\sum_{i=1}^m \sigma_i w_i g(a_i x) \right|^\ell \Big/ f(x)^\ell\right).
    \]
\end{lemma}

\begin{proof}
    We first show that $\mathbb{E} (E^\ell) $ is bounded by the expected difference between two independently drawn samples $a,a'\sim \mathcal{Q}^m$ with their corresponding weights $w, w' \in\mathbb{R}^m$. From \Cref{lem:expected}, we know that $\mathbb{E}_{a \sim \mathcal{Q}^m}(f_0(x) - \frac{1}{m}\sum_{i=1}^m w_i g(a_i x)) = 0$. We thus get
    \begin{align*}
        \mathbb{E} (E^\ell)
        &=\mathbb{E}_{a \sim \mathcal{Q}^m}\left(\sup_{x \in \mathbb{R}^d}\left|f(x)-\left(\frac{1}{m}\sum_{i=1}^m w_i g(a_i x)+ \RR{(x)}/k\right) \right|^\ell\Big/ f(x)^\ell \right)\\
        &=\mathbb{E}_{a \sim \mathcal{Q}^m}\left(\sup_{x \in \mathbb{R}^d}\left|f_0(x)-\frac{1}{m}\sum_{i=1}^m w_i g(a_i x)\right|^\ell \Big/ f(x)^\ell \right) \\
        &=\mathbb{E}_{a \sim \mathcal{Q}^m}\left(\sup_{x \in \mathbb{R}^d}\left|f_0(x)-\frac{1}{m}\sum_{i=1}^m w_i g(a_i x)- \mathbb{E}_{a' \sim \mathcal{Q}^m}\left(f_0(x)-\frac{1}{m}\sum_{i=1}^m w_i' g(a_i'x)))\right) \right|^\ell\Big/ f(x)^\ell \right) \\
        &= \mathbb{E}_{a \sim \mathcal{Q}^m}\left(\sup_{x \in \mathbb{R}^d}\left| \mathbb{E}_{a' \sim \mathcal{Q}^m}\left(\frac{1}{m}\sum_{i=1}^m w_i g(a_i x)- \frac{1}{m}\sum_{i=1}^m w_i' g(a_i' x) \right)\right|^\ell\Big/ f(x)^\ell \right)\\
        &\leq \mathbb{E}_{a, a' \sim \mathcal{Q}^m} \left(\sup_{x \in \mathbb{R}^d}\left|\frac{1}{m}\sum_{i=1}^m w_i g(a_i x)- \frac{1}{m}\sum_{i=1}^m w_i' g(a_i' x))\right|^\ell\Big/ f(x)^\ell \right).
    \end{align*}
    The last inequality follows by Jensen's inequality since $\sup$ is a convex function. 
    
    Next, by a standard symmetrization argument \citep{LT1991,CohenP15}, we bound the last term back in terms of a single sample but with random Rademacher sign variables. Recall that $\ell$ is an even number. We have that 
    \begin{align*}
        \mathbb{E} (E^\ell) &\leq \mathbb{E}_{a, a' \sim \mathcal{Q}^m} \left(\sup_{x \in \mathbb{R}^d}\left|\frac{1}{m}\sum_{i=1}^m w_i g(a_i x)- \frac{1}{m}\sum_{i=1}^m w_i' g(a_i' x))\right|^\ell\Big/ f(x)^\ell \right) \\
        &=\mathbb{E}_{a, a' \sim \mathcal{Q}^m, \sigma \sim\mathcal{U}\{-1, 1\}^m} \left(\sup_{x \in \mathbb{R}^d}\left|\frac{1}{m}\sum_{i=1}^m \sigma_i(w_i g(a_i x)- w_i' g(a_i' x))\right|^\ell\Big/ f(x)^\ell \right)\\
        &\leq  \mathbb{E}_{a, a' \sim \mathcal{Q}^m, \sigma \sim\mathcal{U}\{-1, 1\}^m} \left(\sup_{x \in \mathbb{R}^d}\left(\left|\frac{1}{m}\sum_{i=1}^m \sigma_i w_i g(a_i x) \right| + \left|\frac{1}{m}\sum_{i=1}^m \sigma_i w_i' g(a_i' x))\right|\right)^\ell\Big/ f(x)^\ell \right)\\
        &\leq 2^\ell \cdot \mathbb{E}_{a \sim \mathcal{Q}^m, \sigma \sim\mathcal{U}\{-1, 1\}^m} \left(\sup_{x \in \mathbb{R}^d}\left|\frac{1}{m}\sum_{i=1}^m \sigma_i w_i g(a_i x) \right|^\ell \Big/ f(x)^\ell\right).
    \end{align*}
\end{proof}
    
    As a final step, we need to control the $\ell$-th moment of our Rademacher bound. To this end, we set $$\Lambda\coloneqq \sup_{x \in \mathbb{R}^d}\left|\frac{1}{m}\sum_{i=1}^m \sigma_i w_i g(a_i x) \right| / f(x).$$
    
    We need to upper bound $\mathbb{E}(\Lambda^\ell)$. To this end, we will use a Lipschitz contraction bound (\Cref{lem:expsup}) proven in \citep{AlishahiP24} and the Khinchine-Kahane inequality (\Cref{prop:kahane,cor:khintchine}), which extends Khintchine's classic inequality \citep{Khintchine1923} from sequences of complex numbers to elements of Banach spaces.

    \begin{lemma}\label{lem:lambdabound}
        If $s(a) \geq \norm{a}_2$ holds for all $a \in \mathbb{R}^d$ then we have that
        \[
            \mathbb{E}(\Lambda^\ell) \leq \left(\frac{8\sqrt{\ell}}{\sqrt{m}}\right)^\ell \left( \frac{g(0)^\ell}{\opt^\ell} + (SL)^\ell \sup_{x \in \mathbb{R}^d} \left(\frac{2\norm{x}_2}{\opt+ \RR{(x)}/k}\right)^\ell \right)\,.
        \]
    \end{lemma}
    \begin{proof}
    Recall that $\ell$ is an even number, and $s(a)\geq \|a\|_2 \geq 0$.
    Then we have $w_i=2S/(s(a_i)+S)\leq \min\{2S/\|a\|_2,2\}$ and thus 
    \begin{align*}
        \mathbb{E}(\Lambda^\ell)
        &= \mathbb{E} \sup_{x \in \mathbb{R}^d}\left|\frac{1}{m}\sum_{i=1}^m  \sigma_i w_i g(a_i x)\right| ^\ell\Big/ f(x)^\ell \\
        &= \left(\frac{1}{m}\right)^\ell \mathbb{E} \sup_{x \in \mathbb{R}^d}\left|\sum_{i=1}^m  \sigma_i w_i (g(0)+g_0(a_i x))\right| ^\ell\Big/ f(x)^\ell \\
        &\overset{Jensen}{\leq} \left(\frac{2}{m}\right)^\ell \mathbb{E} \sup_{x \in \mathbb{R}^d} \left(\left| \sum_{i=1}^m  \sigma_i w_i g(0) \right|^\ell + \left|\sum_{i=1}^m  \sigma_i w_i g_0(a_i x)\right|^\ell\right)\Big/ f(x)^\ell \\
        &\overset{\Cref{lem:expsup}}{\leq} \left(\frac{2}{m}\right)^\ell \mathbb{E} \sup_{x \in \mathbb{R}^d} \left((2g(0))^\ell \left| \sum_{i=1}^m  \sigma_i \right|^\ell + (2L)^\ell\left|\sum_{i=1}^m  \sigma_i w_i a_i x\right|^\ell\right)\Big/ f(x)^\ell\\
        &\overset{CSI,\,\Cref{lem:C=opt}}{\leq} \left(\frac{4}{m}\right)^\ell \left( \frac{g(0)^\ell\,\mathbb{E}\left| \sum_{i=1}^m  \sigma_i \right|^\ell}{\opt^\ell} + L^\ell\;\mathbb{E} \sup_{x \in \mathbb{R}^d} \left(\frac{2\norm{x}_2}{\opt + \RR{(x)}/k}\right)^\ell \left\|\sum_{i=1}^m  \sigma_i w_i a_i\right\|_2^\ell\right)\\
        &\overset{\Cref{cor:khintchine}}{\leq} \left(\frac{4}{m}\right)^\ell \ell^{\ell/2} \left( \frac{g(0)^\ell\, {m}^{\ell/2}}{\opt^\ell} + L^\ell\sup_{x \in \mathbb{R}^d} \left(\frac{2\norm{x}_2}{\opt+ \RR{(x)}/k}\right)^\ell \mathbb{E}\left( \sum_{i=1}^m  w_i^2 \left\|a_i\right\|^2_2 \right)^{\ell/2} \right) \\
        &\leq \left(\frac{4\sqrt{\ell}}{m}\right)^\ell \left( \frac{g(0)^\ell{m}^{\ell/2}}{\opt^\ell} + L^\ell\sup_{x \in \mathbb{R}^d} \left(\frac{2\norm{x}_2}{\opt+ \RR{(x)}/k}\right)^\ell (2S)^\ell m^{\ell/2} \right) \\
        &\leq \left(\frac{8\sqrt{\ell}}{\sqrt{m}}\right)^\ell \left( \frac{g(0)^\ell}{\opt^\ell} + (SL)^\ell\sup_{x \in \mathbb{R}^d} \left(\frac{2\norm{x}_2}{\opt+ \RR{(x)}/k}\right)^\ell \right) \\
    \end{align*}
    Here the first inequality follows by Jensen's inequality. The next inequality uses $w_i\leq 2$ and
    further the application of \cref{lem:expsup} uses $\phi_i(r)= w_ig_0(r/w_i)$, $t_i= w_i a_ix$ and $T=\{ Ax \mid x \in \mathbb{R}^d \}$. Note that $T$ is not bounded but since $ w_i g_0(a_i x)/f(x)^\ell$ goes to $0$ for $\Vert x \Vert_2 \rightarrow \infty $ we can apply the lemma to any bounded subset of $T$ to get the unbounded version. 
    The subsequent (in)equalities follow in elementary steps by linearity of expectation, the Cauchy-Schwarz inequality, \Cref{lem:C=opt}, followed by the Khintchine-Kahane inequality (see \Cref{cor:khintchine}) on both expected values and $w_i\leq 2S/\|a_i\|_2$.
\end{proof}

It now remains to bound the supremum term of \Cref{lem:lambdabound} depending on the choice of the regularizer. We have the following lemma.

\begin{lemma}\label{lem:supbound}
    It holds that
    \[
        \sup_{x \in \mathbb{R}^d} \left(\frac{2\norm{x}_2}{\opt+ \RR{(x)}/k}\right)
        \leq
        \begin{cases}
            \sqrt{k/\opt} &, \text{ if }\RR(\cdot) = \norm{\cdot}_2^2\\
            2k &, \text{ if }\RR(\cdot) \in \{ \norm{\cdot}_2,\norm{\cdot}_1\}
        \end{cases}
    \]
\end{lemma}
\begin{proof}
    For the first item, we need a bound on $\sup_{x \in \mathbb{R}^d} \frac{2\norm{x}_2}{\opt+ \norm{x}_2^2/k } $. To this end, we consider the function $h(r)=\frac{r}{\opt+r^2/k}=\frac{1}{\opt/r+r/k}$ for $r \geq 0$. Note, that it is maximized if $\opt/r+r/k$ is minimized, which is the case when its derivative satisfies $1/k-\opt/r^2=0$. Rearranging leads to $r=\sqrt{k\opt}$.
    
    This implies that
    \begin{align}
        \sup_{x \in \mathbb{R}^d} \frac{2\norm{x}_2}{\opt+ \norm{x}_2^2/k } \leq \frac{2\sqrt{k\opt}}{{\opt}+{\opt}}= \frac{\sqrt{k}}{\sqrt{\opt}}\, .
    \end{align}

    For the second item, recall that $\|x\|_1\geq \|x\|_2$ and $\opt>0$. Thus 
    \begin{align*}
        \sup_{x \in \mathbb{R}^d} \left(\frac{2\norm{x}_2}{\opt+ \norm{x}_1/k}\right)
        &\leq \sup_{x \in \mathbb{R}^d} \left(\frac{2\norm{x}_2}{\opt+ \norm{x}_2/k}\right) 
        \leq \sup_{x \in \mathbb{R}^d} \left(\frac{2\norm{x}_2}{\norm{x}_2/k}\right)
        = 2k\,.
    \end{align*}
\end{proof}
    
\subsection{Proof of the main sampling result}

\normsampling*

\begin{proof}
    First assume that $\RR(\cdot)\in\{\norm{\cdot}_2,\norm{\cdot}_1\}$.
    By chaining the arguments of Lemmas~\ref{lem:Ebound}, \ref{lem:lambdabound} and \ref{lem:supbound}, and using that $\opt\geq g(0)/(LBk)$ by \Cref{lem:lb-opt}, we get for sufficiently large absolute constants $C''>C'>1$ that 
    \begin{align*}
        \mathbb{E}( E^\ell )
        &\leq \left(\frac{C'\sqrt{\ell}}{\sqrt{m}}\right)^\ell \left( \frac{g(0)^\ell}{\opt^\ell} + (SLk)^\ell \right)
        \leq \left(\frac{C'\sqrt{\ell}}{\sqrt{m}}\right)^\ell \Big( (BLk)^\ell + (SLk)^\ell \Big)
        \leq \left(\frac{C''SLk\sqrt{\ell}}{\sqrt{m}}\right)^\ell \,.
    \end{align*}
        
    As a consequence, setting $\ell= 2\lceil\ln(\delta^{-1})\rceil$ and $m={(C'')^2 S^2 L^2 k^2 \ell}/\varepsilon^2 = O(S^2L^2\eps^{-2}k^2\log(\delta^{-1}))$, the final moment bound becomes $\mathbb{E}(E^\ell)\leq (\eps/2)^\ell$ and we can conclude by applying \Cref{lem:probbound}.

    For the case $\RR(\cdot)=\norm{\cdot}_2^2$, we can use the appropriate bound of \Cref{lem:supbound} and the lower bound $\opt\geq g(0)^2/(4(BL)^2k)$ of \Cref{lem:lb-opt}. We thus get for sufficiently large absolute constants $C''>C'>1$ that 
    \begin{align*}
        \mathbb{E}( E^\ell )
        &\leq \left(\frac{C'\sqrt{\ell}}{\sqrt{m}}\right)^\ell \left( \frac{g(0)^\ell}{\opt^\ell} + SL^\ell \frac{k^{\ell/2}}{\opt^{\ell/2}} \right)
        = \left(\frac{C'\sqrt{\ell}}{\sqrt{m}}\right)^\ell \left( \frac{g(0)^\ell}{\opt^{\ell/2}} \cdot \frac{1}{\opt^{\ell/2}} + (SL)^\ell\frac{k^{\ell/2}}{\opt^{\ell/2}} \right) \\
        &\leq \left(\frac{C'\sqrt{\ell}}{\sqrt{m}}\right)^\ell
        \left( {(2BL)^\ell} \cdot \frac{k^{\ell/2}}{\opt^{\ell/2}} + (SL)^\ell\frac{k^{\ell/2}}{\opt^{\ell/2}} \right)
        \leq \left(\frac{C''SL\sqrt{k\ell}}{\sqrt{m\opt}}\right)^\ell \,.
    \end{align*}
    Setting $\ell= 2\lceil\ln(\delta^{-1})\rceil$ and $m={(C'')^2 S^2 L^2 k \ell}/(\varepsilon^2 \opt) = O(S^2 L^2  \eps^{-2} k \log(\delta^{-1})/\opt)$, the final moment bound becomes $\mathbb{E}(E^\ell)\leq (\eps/2)^\ell$ and we can conclude by applying \Cref{lem:probbound}.
\end{proof}

\section{Improving the upper bound to linear}\label{sec:improvingtolinear}
In this section, we present an improved analysis that yields linear dependence on $k$. This is achieved in two different settings:
\begin{enumerate}
    \item $\norm{\cdot}_2^2$-regularization when $g$ is either monotonic or convex/concave and has a bounded derivative, i.e., $|g'|\leq g$,
    \item general functions with $\norm{\cdot}_1$-regularization.
\end{enumerate}

We start wit a few preliminaries on Gaussian processes and their analysis.

\subsection{Moment bound for Gaussian processes}

The following moment bound follows from the so-called Dudley's theorem and is one of our main analysis tools: 

\begin{lemma}[\citealp{woodruffyasuda23}, slightly modified]
\label{lem:moment-boundentropy}
Let $(X_t)_{t\in T}$ be a Gaussian process with pseudo-metric $d_X(s,t)\coloneqq \norm*{X_s - X_t}_2 = \sqrt{\E (X_s - X_t)^2}$.
Further, we let $\mathcal \diam(T) \coloneqq \sup\braces*{d_X(s,t) : s, t\in T}\leq \mathcal D$ be a bound on the diameter of $d_X$ on $T$ and $\int_0^\infty \sqrt{\log E(T, d_X, u)}~du \leq \mathcal E$ be a bound on the metric entropy. Then
\[
     \E[\sup\nolimits_{t\in T}\abs{X_t}^\ell] \leq (2\mathcal E)^\ell (\mathcal E/\mathcal D) + O(\sqrt \ell \mathcal D)^{\ell}.
 \]
\end{lemma}

The only modification is that we consider a more general $\Lambda$ than \citet{woodruffyasuda23}. However, we stress that their proof remains unchanged.
The moment bound can be used for a suitably large choice of $\ell$ to obtain low failure probability bounds using Markov's inequality.
Our task reduces to obtaining best possible bounds for the diameter $\mathcal D$ and the entropy $\mathcal E$ to allow a small sample size.

\subsection{Sudakov's bounds on the covering numbers of \texorpdfstring{$\ell_2$}{L2} balls}
In this section we give some basic definitions and results regarding covering numbers.
\begin{definition}
    Let $(V,\norm*{\cdot})$ be a normed space and $W\subseteq V$. Then $C\subseteq W$ is a $t$-covering of $W$ of size $n=|C|$ if $\forall w\in W \, \exists c\in C\colon \|w-c\|\leq t$. The covering number of covering $W$ with balls of radius $t$ with respect to $\|\cdot\|$ is defined as $E(W,\|\cdot\|, t)=\min\{n\in \mathbb{N}\mid \exists t\text{-covering of } W \text{of size }n\}$.
\end{definition}

\begin{definition}
Let $\norm*{\cdot}_X$ be a norm. Then, we define
\[
    M_X \coloneqq {\E_{g\sim\mathcal N(0,I_d)}\norm*{g}_X}\,.
\]
\end{definition}

The following results will help us bounding covering numbers for the cases where $T$ are balls with respect to the Euclidean or the $\ell_1$ norm.
Bounds on $M_X$ imply bounds for covering the Euclidean ball by $\norm*{\cdot}_X$-balls using the following result:

\begin{proposition}[Sudakov minoration, \citealp{Sudakov1971}, cf. Proposition 4.1 of \citealp{BLM1989}]
\label{prop:sudakov}
Let $\norm*{\cdot}_X$ be a norm, and let $\norm*{\cdot}_{X^*}$ be the dual norm with respect to the Euclidean inner product. Let $B_2\subseteq\mathbb R^d$ denote the Euclidean unit ball in $d$ dimensions. Then
\[
    \log E(\norm*{\cdot}_X, B_2, t) \leq O(1)\frac{M_{X^*}^2}{t^2}\,.
\]
\end{proposition}

\begin{proposition}[Dual Sudakov minoration, Proposition 4.2 of \citealp{BLM1989}]
\label{prop:dual-sudakov}
Let $\norm*{\cdot}_X$ be a norm, and let $B_2\subseteq\mathbb R^d$ denote the Euclidean unit ball in $d$ dimensions. Then
\[
    \log E(B_2, \norm*{\cdot}_X, t) \leq O(1)\frac{M_X^2}{t^2}\,.
\]
\end{proposition}

We will also need a standard covering argument:
\begin{lemma}
\label{lem:ballcover}
  \[
    \log E(B_2, \norm*{\cdot}_2, t) \leq d\log(1+2/t)\,.
  \]  
\end{lemma}
\begin{proof}
    Consider a maximal set $C=\{x_1,\ldots,x_2\}\subseteq B_2$ such that $\min_{i\neq j}\|x_i-x_j\|_2>t/2$. Then $C$ is a $t$-covering of $B_2$, otherwise there exists $x\in B_2\setminus C$ such that $\forall c\in C \colon \|x-c\|_2 \geq t$ contradicting the maximality of $C$.   
    We have that $\bigcup (x_i+t/2) B_2 \subseteq (1+t/2) B_2$. Thus, we have by a standard volume argument $E(B_2,\norm*{\cdot}_2, t)\leq \frac{\vol((1+t/2) B_2)}{\vol{((t/2) B_2)}}\leq (1+2/t)^d$. The Lemma follows by taking logs on both sides.
\end{proof}

\subsection{Sampling bound for monotonic, convex or concave functions with bounded derivative and ridge regularization}\label{sec:boundedderiveative}

In this subsection we study the setting where 
\[
    f(x)=\int g(\langle a,  x\rangle)\,d\mathcal{P}(a)+\frac{\Vert x \Vert_2^2}{k}
\]
for some $k \in \mathbb{N}$. Set $f_0(x)=\int g(\langle a,  x\rangle)\,d\mathcal{P}(a)$. Let $g: \mathbb{R}\rightarrow \mathbb{R}_{\geq 0}$ be convex or concave or monotonic and satisfy $|g'(r)|\leq g(r)$ for all $r \in \mathbb{R}$.

Further, we assume that $g$ is $L$-Lipschitz (Assumption~\ref{ass:L}) and that $\mathbb{E}_{a \sim \mathcal{P}}(\norm{a}_2^2)= B < \infty$ for some $L\geq 1, B\geq 0$, which is slightly stronger as Assumption~\ref{ass:B} before as discussed in \Cref{sec:assumptions}. We will use $s(a)\geq \|a\|_2^2+2$, which is squared norm sampling and can be handled analogously to the previous norm sampling. Indeed, it implies $s(a)\geq \|a\|_2+1$ as before and under the strengthened assumption, if we set $s(a)= \|a\|_2^2+2$, we have that $B\leq S = B + 2$, is in particular independent of $k$.

Fix some subset $T\subseteq \R^d$. Given a sample $S$ we further set $G_S=\sup\nolimits_{x \in T}1+|f_{S}(x)-f_0(x)|/f(x)$ where $f_S(x)=\frac{1}{m}\sum_{i=1}^m w_i g(a_i x)$. Note that $\sqrt{G_S}\leq G_S$ since $G_S \geq 1$. 

We will frequently use the following bound.
\begin{lemma}\label{lem:GSbound}
    It holds for any $T\subseteq \mathbb{R}^d$ and all $x \in T$ that $\frac{f_S(x)}{f(x)} \leq G_S.$
\end{lemma}
\begin{proof}
    It always holds that 
    \[
        \frac{f_S(x)}{f(x)} - \frac{f_0(x)}{f(x)} \leq \left| \frac{f_S(x)}{f(x)} - \frac{f_0(x)}{f(x)} \right|\,.
    \]
    Thus,
    \[
        \frac{f_S(x)}{f(x)} \leq 1 + \left| \frac{f_S(x)}{f(x)} - \frac{f_0(x)}{f(x)} \right| \leq \sup\nolimits_{x \in T} 1+\frac{|f_{S}(x)-f_0(x)|}{f(x)} = G_S\,.
    \]
\end{proof}

A final (very weak) bound will be required later to claim that $G_S\leq \poly(m)$. 
\begin{lemma}\label{lem:GSpolym}
    It holds that
    \[
        G_S \leq \begin{cases}
            {18(LS)^3 k}/{g(0)^2}, \quad\quad \text{ for } \RR(\cdot)=\norm{\cdot}_2^2\\
            6LSk, \quad\quad\;\;\;\;\; \text{ for } \RR(\cdot)\in\{\norm{\cdot}_2,\norm{\cdot}_1\}.
        \end{cases}
    \]    
\end{lemma}
\begin{proof}
    We will see in the course of the analysis, that since $\|x\|_1\geq\|x\|_2$, it suffices to show the claim for $\|x\|_2^p$, for $p\in \{1,2\}$. We start with a common analysis and then make a case distinction depending on $p$ in the final derivations.

    Recall that $w_i = \frac{2S}{s(a_i)+S} \leq \min\{2, 2S/\|a_i\|_2\}$. The two bounds of \Cref{lem:lb-opt} combined yield $f(x)\geq \frac{g(0)^p}{4^{p-1}(BL)^pk}$.

    It thus holds that 
    \begin{align*}
        G_S
        &= 1+\sup\nolimits_{x \in T} \frac{|f_{S}(x)-f_0(x)|}{f(x)} \\
        &\leq 1+\sup\nolimits_{x \in T} \frac{|f_{0}(x)|}{f(x)} + \sup\nolimits_{x \in T} \frac{|f_{S}(x)|}{f(x)} \\
        &\leq 2 + \sup\nolimits_{x \in T} \frac{|f_{S}(x)|}{f(x)} \\
        &\leq 2 + \sup\nolimits_{x \in T} \frac{1}{m}\sum_{i=1}^m w_i \frac{g(0)+g_0(a_i x)}{f(x)} \\
        &\leq 2 + \frac{2g(0)}{f(x)} + \sup\nolimits_{x \in T} \frac{L}{m}\sum_{i=1}^m w_i \frac{\|a_i\|_2 \|x\|_2}{f(x)} \\
        &\leq 2 + \frac{2\cdot4^{p-1}(LB)^pk}{g(0)^{p-1}} + 2\sup\nolimits_{x \in T} \frac{LS}{m}\sum_{i=1}^m \frac{\|x\|_2}{f(x)}\,.
    \end{align*}

    For $\RR(\cdot)\in\{\norm{\cdot}_2,\norm{\cdot}_1\}$, we have that $f(x)\geq \frac{\|x\|_1}{k}\geq \frac{\|x\|_2}{k}$. The above derivation with $p=1$ yields
    \begin{align*}
        G_S
        &\leq 2 + {2LBk} + 2{LSk} \leq 6 {LSk}\,.
    \end{align*}

    For $\RR(\cdot)=\norm{\cdot}_2^2$, we consider two sub-cases.   
    If $\|x\|_2 \geq 1$, then $\|x\|_2^2 \geq \|x\|_2$ and we can simply apply the same continuation with $p=1$ to get $G_S\leq 6 {LSk}$.
    
    In the remaining case $\|x\|_2 \leq 1$, we need to set $p=2$. Our derivation yields
    \begin{align*}
        G_S
        &\leq 2 + \frac{8(LB)^2k}{g(0)} + \frac{2LS}{f(x)} 
        \leq 2 + \frac{8(LB)^2k}{g(0)} + \frac{8LS(LB)^2 k}{g(0)^2} 
        \leq \frac{18(LS)^3 k}{g(0)^2} \,.
    \end{align*}
\end{proof}

A standard symmetrization argument yields the following bound that relates higher moments of the error $E=G_S-1$ to a Gaussian process defined over $T$.

\begin{lemma}\label{lem:errorboundbyGp}
    It holds that $\E(G_S-1)^\ell\leq ({2\pi})^{\ell/2} \E(\Lambda^\ell)$ where $\Lambda=\sup_{t \in T}\left|\frac{1}{m}\sum_{i=1}^m \sigma_i w_i g(a_i t)\right| / f(t) $, and $\sigma_i$ are i.i.d standard Gaussian random variables.
\end{lemma}
\begin{proof}
    The first part is verbatim to \Cref{lem:Ebound} leading to Rademacher variables $\sigma_i$. We then convert from Rademacher to Gaussian random variables, which takes an additional factor $(\frac{\pi}{2})^{\ell/2}$, see Eq. (4.8) of \citep{LT1991}.
\end{proof}

We will analyze the error in slices of similar function values. To this end, we  set $T_\alpha=\{ x \in \mathbb{R}^d ~|~ f(x) \in [\alpha, 2\alpha) \}$.
We note that for any $x \in T_\alpha$ it holds that $\Vert x \Vert_2^2 \leq 2 \alpha k$ and thus $T_\alpha \subseteq (2 \alpha k)^{1/2}B_2$.
We consider the canonical Gaussian process $X_s=\frac{1}{m}\sum_{i=1}^m \sigma_i w_i g(a_i x)$ for $x \in T_\alpha$ for some fixed $\alpha \in \mathbb{R}$, where $\sigma_i$ denote i.i.d. standard Gaussian random variables.

This process induces a metric $d_X$ over $T_\alpha$ that is given by
\begin{align*}
    d_X(x, x')&=\left(\E_{\sigma} \left(\frac{1}{m}\sum_{i=1}^m  \sigma_i w_i g(a_i x) - \frac{1}{m}\sum_{i=1}^m  \sigma_i w_i g(a_i x')  \right)^2\right)^{1/2}\\
    &= \left(\frac{1}{m^2}\sum_{i=1}^m  w_i^2 \left(g(a_i x) - g(a_i x')\right)^2  \right)^{1/2}\,,
\end{align*}
where we used the fact that $\sigma_i$ are independent zero-mean and unit-variance random variables. This implies for $i\neq j$ that $\E(\sigma_i\sigma_j)=\E(\sigma_i)\E(\sigma_j)=0$. Thus only the squared terms survive and we can use ${\E(\sigma_i^2)}=1$ to conclude.

\subsubsection{Bounding the diameter}\label{sec:diam_boundedderiveative}
We first relate the diameter to the sensitivity. This will be reused in the case of $\norm*{\cdot}_1$ regularization.

\begin{lemma}\label{lem:diamsensbound}
    It holds that $\sup\limits_{s, t\in T_\alpha} d_X(s,t) \leq 2G_S \alpha\left(\frac{2}{m} \sup_{x \in T_\alpha, j \in [m]}\frac{w_j g(a_j x)}{f(x)}\right)^{1/2}
    $.
\end{lemma}
\begin{proof}
    Note that for all $x\in\mathbb{R}^d$ we have ${g(a_i s)}\geq 0$. Then 
    \begin{align*}
        \sup_{s,t\in T_\alpha} d_X(s, t)
        &=\sup_{s,t\in T_\alpha} \left(\frac{1}{m^2}\sum_{i=1}^m  w_i^2 \left(g(a_i s) - g(a_i t)\right)^2  \right)^{1/2} \\
        &=\sup_{s,t\in T_\alpha} \left( \frac{1}{m^2}\sum_{i=1}^m  w_i^2 \left(g(a_i s)^2 - \;2\,g(a_i s)g(a_i t) + g(a_i t)^2 \right) \right)^{1/2} \\
        &\leq \sup_{s,t\in T_\alpha} \left( \frac{1}{m^2}\sum_{i=1}^m  w_i^2 \left(g(a_i s)^2 + g(a_i t)^2 \right) \right)^{1/2} \\
        &= \sup_{x\in T_\alpha} \left( \frac{2}{m^2}\sum_{i=1}^m  w_i^2 g(a_i x)^2 \right)^{1/2} \\
        &\leq \sup_{x \in T_\alpha}\left(\frac{2}{m} \max_{j \in [m]}w_j g(a_j x) \frac{1}{m} \sum_{i=1}^m w_i g(a_i x)\right)^{1/2} \\
        &\leq \sup_{x \in T_\alpha}\left(\frac{2}{m}\max_{j \in [m]}w_j g(a_j x) \cdot G_S f(x) \right)^{1/2}\\
        &= \sup_{x \in T_\alpha}\left(\frac{2}{m}\max_{j \in [m]}\frac{w_j g(a_j x)}{f(x)} \cdot G_S f(x)^2 \right)^{1/2}\\
        &\leq \,2G_S \alpha\left(\frac{2}{m} \sup_{x \in T_\alpha,j \in [m]}\frac{w_j g(a_j x)}{f(x)}\right)^{1/2}.
    \end{align*}
\end{proof}

It remains to bound the sensitivity.
\begin{lemma}\label{lem:sensitivitybound}
    Let $w_i=\frac{2S}{s(a_i)+S}$. Then for any $i \in [m]$ and any $x \in T_\alpha$ it holds that
    \begin{align*}
        \frac{w_i g(a_ix)}{f(x)} \leq \frac{16S BL^2 k}{g(0)}.
    \end{align*}
\end{lemma}
\begin{proof}
    Recall that $s(a_i)\geq \|a_i\|^2_2+1\geq \max\{\|a_i\|_2,1\}$. Also recall that $\|x\|_2\leq \sqrt{2\alpha k}$ and $f(x)\geq \alpha \geq \frac{g(0)^2}{4(BL)^2k}$. Thus, it holds that
    \begin{align*}
        \frac{w_i g(a_ix)}{f(x)}&= \frac{w_i(g(0)+g_0(a_i x))}{f(x)} 
        \leq \frac{2S}{s(a_i)+S} \cdot \frac{g(0)+L |a_i x|}{f(x)} \\
        &\overset{CSI}{\leq} \frac{2S}{S} \frac{g(0)}{f(x)} + \frac{2S}{\|a_i\|_2} \frac{L \|a_i\|_2 \|x\|_2}{f(x)} 
        \leq 2 \left(\frac{4(BL)^2 k }{g(0)} + \frac{S L \|x\|_2}{f(x)} \right) \\
        &\leq 2 \left(\frac{4(BL)^2 k }{g(0)} + \frac{SL \sqrt{2k}}{\sqrt{\alpha}} \right) 
        \leq 2 \left(\frac{4(BL)^2 k }{g(0)} + \frac{4SBL^2 k}{g(0)} \right) 
        \leq \frac{16S BL^2 k}{g(0)}\,. \qedhere
    \end{align*}
\end{proof}

The following lemma shows that the sensitivity bound (and thus using  \Cref{lem:diamsensbound} also the diameter bound) can be improved to be linear when we have a small $B=O(\log k)$.
\begin{lemma}\label{lem:sensitivitybound2}
    Let $w_i=\frac{2S}{s(a_i)+S}$ and assume that $g(r) \geq |g'(r)|$ for all $r\in \mathbb{R}$. Then for any $i \in [m]$ and any $x \in T_\alpha$ it holds that
    \begin{align*}
        \frac{w_i g(a_ix)}{f(x)} \leq 4SL\max\left\{k, \frac{\exp(B)}{g(0)} \right\}
    \end{align*}
\end{lemma}
\begin{proof}
    Recall that $s(a_i)\geq \|a_i\|^2_2+2\geq \|a_i\|^2+1 \geq \max\{\|a_i\|_2,1\}$.
    First, note that our assumption on the derivative implies that $g(r)\geq g(0)\exp(-r)$.
    
    To see this, assume $g(0)>0$, otherwise the statement is trivial.
    Start with $g(r) \geq |g'(r)| \geq -g'(r)$, thus $ \frac{g'(r)}{g(r)} \geq -1$. It follows that 
    \begin{align*}
        \ln(g(r))-\ln(g(0)) = \int_0^{r} \frac{g'(x)}{g(x)}\,dx \geq - \int_0^{r} 1 = -r\,.
    \end{align*}
    Taking the exponential function on both sides and rearranging, gives
    \begin{align}\label{eq:expgbound}
        \frac{g(r)}{g(0)}  &\geq \exp(-r) \nonumber\\
        \Longleftrightarrow\quad g(r) &\geq {g(0)}\exp(-r)\,.
    \end{align}
    
    Now, for $\|x\|_2\leq 1$ we have that
    \begin{align*}
        f(x) &= \int g(\langle a,  x\rangle)\,d\mathcal{P}(a)+\frac{\Vert x \Vert_2^2}{k} \\
        &\overset{\cref{eq:expgbound}}{\geq} \int \exp(-\langle a,  x\rangle )g(0)\,d\mathcal{P}(a) \\
        &\overset{Jensen}{\geq} g(0)\exp \left(-\int |\langle a,  x\rangle |\,d\mathcal{P}(a) \right) \\
        &\overset{CSI}{\geq} g(0)\exp \left(-\int\Vert a \Vert_2 \Vert x \Vert_2\,d\mathcal{P}(a) \right)\geq g(0)\exp(-B).
    \end{align*}
    and thus
    \begin{align*}
        \frac{w_i g(a_ix)}{f(x)}&= \frac{w_i(g(0)+g_0(a_i x))}{f(x)} 
        \leq \frac{2S}{s(a_i)+S} \cdot \frac{g(0)+L |\langle a_i, x\rangle|}{f(x)} \\
        &\overset{CSI}{\leq} \frac{2g(0)+2SL\|x\|_2}{g(0)\exp(-B)}       \leq 2 \left(\exp(B) + \frac{SL}{g(0)\exp(-B)} \right) 
        \leq \frac{2SL\exp(B) }{g(0)}\,. 
    \end{align*}
    For any $x$ with $\|x\|_2\geq 1$ we have that $\|x\|_2^2\geq \|x\|_2$. Therefore $f(x)= f_0(x)+\|x\|_2^2/k\geq f_0(x)+\|x\|_2/k$. Thus $f(x)\geq g(0)/(LBk)$ by \Cref{lem:lb-opt}. Then,
    \begin{align*}
        \frac{w_i g(a_ix)}{f(x)}&= \frac{w_i(g(0)+g_0(a_i x))}{f(x)} 
        \leq \frac{2S}{s(a_i)+S} \cdot \frac{g(0)+L |a_i x|}{f(x)} \\
        &\overset{CSI}{\leq} \frac{2S}{S} \frac{g(0)}{f(x)} + \frac{2S}{\|a_i\|_2} \frac{L \|a_i\|_2 \|x\|_2}{f(x)} 
        \leq 2 BL k + 2 S L k \leq 4 SLk\,.
    \end{align*}
    \qedhere
\end{proof}

We conclude the following two bounds on the diameter.
\begin{corollary}\label{cor:diambound}
    It holds that $\sup\limits_{s, t\in T_\alpha} d_X(s,t) \leq 12 G_S \alpha\left(\frac{S B L^2 k}{mg(0)}\right)^{1/2} $.
\end{corollary}
\begin{proof}
    The bound follows directly from using the sensitivity bound of \Cref{lem:sensitivitybound} with \cref{lem:diamsensbound}
\end{proof}

\begin{corollary}\label{cor:diambound2linear}
    It holds that $\sup\limits_{s, t\in T_\alpha} d_X(s,t) \leq 6 G_S \alpha\left(\frac{S L (k + \exp(B))}{mg(0)}\right)^{1/2} $.
\end{corollary}
\begin{proof}
    The bound follows directly from using the sensitivity bound of \Cref{lem:sensitivitybound2} with \cref{lem:diamsensbound}
\end{proof}

\subsubsection{Relating the metric to a norm}
To bound the entropy, we first bound the metric $d_X$ by a norm $\norm*{\cdot}_X$ that will be used in subsequent steps.

\begin{lemma}\label{lem:metricbound3}
    Assume that $g$ is either convex or concave or monotonic and for all $r$ we have that $|g'(r)|\leq g(r)$. Then for any $s, t \in T_\alpha$ it holds that
    \begin{align*}
        d_X(s,t)\leq \left(4G_S\alpha \max_{j \in [m]} w_j |\langle a_j , s-t \rangle|^2  L/m \right)^{1/2} .
    \end{align*}
\end{lemma}

\begin{proof}
    We have 
    \begin{align*}
        d_X(s,t)^2 
        & = \frac{1}{m^2}\sum_{i=1}^m  w_i^2 \left(g(a_i s)-g(a_i t)\right)^2 \\
        & = \frac{1}{m^2}\sum_{i=1}^m  w_i^2 \left|g(a_i s)-g(a_i t)\right| \cdot|g(a_i s)-g(a_i t)| \\
        &= \frac{1}{m^2}\sum_{i=1}^m  w_i^2 \left| \int_{a_it}^{a_is} g'(x) \,dx \right| \cdot|g(a_i s)-g(a_i t)| \\
        &\leq \frac{1}{m^2}\sum_{i=1}^m  w_i^2 |a_i s-a_i t| \cdot \left( \max_{x\in[{a_it},{a_is}]} \left|g'(x)\right| \right) \cdot|g(a_i s)-g(a_i t)| \tag{*} \label{eq:test}
    \end{align*} 
    If the function is convex or concave, then the derivative is monotonic and the maximum is attained on the boundary of the interval. Thus,
    \begin{align*}
        \eqref{eq:test}&\leq \frac{1}{m^2}\sum_{i=1}^m  w_i^2 |a_i s-a_i t| \left( |g'(a_is)| + |g'(a_it)|\right) \cdot|g(a_i s)-g(a_i t)| \tag{**} \label{eq:test2} 
    \end{align*} 
    Alternatively, if the function itself is monotonic we first use the derivative bound and then the maximum is again attained on the boundary of the interval. Thus,
    \begin{align*}
        \eqref{eq:test}&\leq \frac{1}{m^2}\sum_{i=1}^m  w_i^2 |a_i s-a_i t| \cdot \left( \max_{x\in[{a_it},{a_is}]} \left|g(x)\right| \right) \cdot|g(a_i s)-g(a_i t)| \\
        &\leq \frac{1}{m^2}\sum_{i=1}^m  w_i^2 |a_i s-a_i t| \left( |g(a_is)| + |g(a_it)|\right) \cdot|g(a_i s)-g(a_i t)| \tag{**}
    \end{align*} 
    Finally we can finish both derivations as follows.
    \begin{align*}
        \eqref{eq:test2}&\leq \frac{L}{m^2}\sum_{i=1}^m  w_i^2 |a_is-a_it|^2\left(\left|g(a_i s)\right|+\left|g(a_i t)\right|\right) \\
        &\leq \frac{L}{m}\max_{j \in [m]} w_j |\langle a_j , s-t \rangle|^2 \cdot \left(\frac{1}{m}\sum_{i=1}^m  w_i g(a_i s) + \frac{1}{m}\sum_{i=1}^m  w_i g(a_i t)\right) \\
        &\leq \frac{L}{m} \max_{j \in [m]} w_j |\langle a_j , s-t \rangle|^2 \cdot G_S (f(s) + f(t))\\
        &\leq 4 G_S \alpha \max_{j \in [m]} w_j |\langle a_j , s-t \rangle|^2 {L}/{m}.
    \end{align*} 
\end{proof}

\subsubsection{Bounding the covering numbers}
Clearly $\norm*{\cdot}_X:=\left(4 G_S \alpha \max_{j \in [m]} w_j \left| \langle a_j , \cdot\, \rangle \right|^2  L/m \right)^{1/2}$ defines a norm. 
We will use the following lemma to bound the covering numbers.

\begin{lemma}\label{lem:covnumbersboundedderviative}
    Assume that $g$ is either convex or concave or monotonic and for all $r$ we have that $|g'(r)|\leq g(r)$. Then for any radius $t > 0$ it holds that 
    \begin{align*}
        \log E(B_2, \norm*{\cdot}_X, t)\leq {O(G_S^2\alpha SL\ln(m)/m)}/{t^2}
    \end{align*}
    and
    \begin{align*}
        \log E(T_\alpha, \norm*{\cdot}_X, t)\leq {O(G_S^2\alpha^2 SLk\ln(m)/m)}/{t^2}\,.
    \end{align*}
\end{lemma}

\begin{proof}
    Since $s(a_j)\geq \|a_j\|_2^2+2\geq \|a_j\|_2^2$, and thus $w_j = \frac{2S}{s(a_j)+S}\leq 2S/\|a_j\|_2^2$, it follows that
    \begin{align}
        \E_{y\sim\mathcal N(0,I_d)}\norm*{y}_X \nonumber
        &= \E_{y\sim\mathcal N(0,I_d)}\left(4G_S\alpha \max_{j \in [m]} w_j |\langle a_j , y \rangle|^2 L/m\right)^{1/2}\\ \nonumber
        &\leq 2G_S \sqrt{\alpha L/m} \E_{y\sim\mathcal N(0,I_d)}\left(\max_{j \in [m]} w_j |\langle a_j , y \rangle|^2 \right)^{1/2}\\ \nonumber
        &= 2G_S \sqrt{\alpha L/m} \E_{y\sim\mathcal N(0,I_d)} \max_{j \in [m]} \sqrt{w_j} |\langle a_j , y \rangle| \\ \nonumber
        &\leq 2G_S \sqrt{2\alpha SL/m} \E_{y\sim\mathcal N(0,I_d)} \max_{j \in [m]} \left|\left\langle \frac{a_j}{\|a_j\|_2} , y \right\rangle\right| \\
        &\leq O(G_S\sqrt{\alpha SL\ln(m)/m})\,, \label{eq:levybound}
    \end{align}
    where the last inequality follows from \Cref{lem:expmaxGaussian} and the fact that each $Z_j = \left\langle \frac{a_j}{\|a_j\|_2} , y \right\rangle$ is identically distributed as a zero mean Gaussian with unit variance $\left\|\frac{a_j}{\|a_j\|_2}\right\|_2^2=1$ by self-stability of the Gaussian distribution.

    Finally, we can apply \Cref{prop:dual-sudakov} using \Cref{eq:levybound} as follows.
    \begin{align*}
        \log E(B_2, \norm*{\cdot}_X, t)\leq  O(G_S^2\alpha SL\ln(m)/m)/t^2
    \end{align*}  
    The second bound follows using that $T_\alpha \subseteq (2\alpha k)^{1/2}B_2$. Then
    \begin{align*}
        \log E(T_\alpha, \norm*{\cdot}_X, t)
        &\leq \log E((2 \alpha k)^{1/2}B_2, \norm*{\cdot}_X, t) \\
        &=\log E(B_2, \norm*{\cdot}_X, t/(2 \alpha k)^{1/2}) 
        \leq O(G_S^2\alpha^2 SLk\ln(m)/m)/t^2\,. \qedhere
    \end{align*}  
\end{proof}

For very small $t$ the following bound gives better dependencies.

\begin{lemma}\label{lem:coveringnumbersgridbound}
    It holds that
    \begin{align*}
        \log E(B_2, \norm*{\cdot}_X, t)\leq O(m \ln(t^{-1}G_S\sqrt{\alpha SL\ln(m)/m}))
    \end{align*}
    and
    \begin{align*}
        \log E(T_\alpha, \norm*{\cdot}_X, t)\leq O(m \ln(t^{-1}G_S\alpha\sqrt{SLk\ln(m)/m})).
    \end{align*}
\end{lemma}

\begin{proof}
Let $\{v_1, \dots, v_m\} \subset \mathbb{R}^d$ be an orthonormal basis of the subspace $V$ spanned by $a_1, \ldots, a_m$.
    We show that the following set is a covering of $B_2$ of size $|N|\leq\exp(O(m\ln(t^{-1}G_S\sqrt{\alpha SL\ln(m)/m})))$ with respect to $\norm*{\cdot}_X $.
    \[
        N=\left\{ \sum_{i=1}^m z_iv_i t/(G_S\sqrt{\alpha SL\ln(m)/m}) ~\bigg|~ z \in \mathbb{Z}^m, |z_i| \leq t^{-1}G_S\sqrt{\alpha SL\ln(m)/m} \right\} .
    \] 
    We fix an arbitrary element $x \in B_2$ and let $x'=\sum_{i=1}^m x_iv_i$ be the orthogonal projection of $x$ onto $V$.
    Further, we set $x_i''= \lfloor x_i' \cdot t^{-1}G_S\sqrt{\alpha SL\ln(m)/m} \rfloor \cdot t/(G_S\sqrt{\alpha SL\ln(m)/m})$. Note that $x_i'' \in N$ and $\|x'-x''\|_X \leq t$ hold by construction.
    Finally, since $x-x'$ is orthogonal to $V$, the triangle inequality yields that $$ \|x-x''\|_X = \|x-x'-(x''-x')\|_X \leq \|x-x'\|_X +\|x'-x''\|_X = \|x'-x''\|_X \leq t.$$
    The second bound follows from rescaling $t$ by a factor $(2\alpha k)^{-1/2}$ as in the previous lemma.
\end{proof}

\subsubsection{Bounding the metric entropy}
We are now ready to bound the entropy integral.

\begin{lemma}\label{lem:entropyboundedderviative}
    Assume that $g$ is either convex or concave or monotonic and for all $r$ we have that $|g'(r)|\leq g(r)$. Let $m\geq C\cdot \min\{SBL^2k/g(0), SL(k+\exp(B))/g(0)\}$, for a sufficiently large absolute constant $C>1$, and assume that $\alpha \leq 1$. Then it holds that $\int_0^\infty \sqrt{\log E(T_\alpha, \norm*{\cdot}_X, t)}\,dt \leq O(\ln(m)^{3/2}G_S\alpha\sqrt{SLk/m})$
\end{lemma}

\begin{proof}
    We split the entropy integral at some small $\lambda \ll 1$ into two parts. We use \Cref{lem:coveringnumbersgridbound} and the technical \Cref{lem:calc} to bound the first part and \Cref{lem:covnumbersboundedderviative} for the second. It is sufficient to integrate up to the diameter, since for larger radii, the covering number is simply $1$ and its logarithm becomes $0$. By \Cref{cor:diambound,cor:diambound2linear,lem:GSpolym} we have that $\mathcal D \leq G_S\alpha \leq \poly(m)$ for our choice of $m$ and assumption that $\alpha \leq 1$. 

    Then there exist constants $C_1,C_2$ such that
    \begin{align*}
        &\int_0^\infty \sqrt{\log E(T_\alpha, \norm*{\cdot}_X, t)}\,dt\\
        &= \int_0^{\lambda} \sqrt{\log E(T_\alpha, \norm*{\cdot}_X, t)}\,dt ~+ ~\int_{\lambda}^\mathcal{D} \sqrt{\log E(T_\alpha, \norm*{\cdot}_X, t)}\,dt\\
        &\leq \int_0^{\lambda} C_1(m \ln(t^{-1}G_S\alpha\sqrt{SLk \ln(m)/m}))^{1/2}\,dt + \int_{\lambda}^{\mathcal D} C_2 t^{-1}G_S\alpha\sqrt{SLk\ln(m)/m}\,dt\\
        &\leq O(\lambda) \sqrt{m \ln\left(\lambda^{-1}\right)}+O(\lambda) \sqrt{m \ln\left(G_S\alpha\sqrt{SLk \ln(m)/m}\right)} + O(\ln(\mathcal D) + \ln(\lambda^{-1}) )\, G_S\alpha\sqrt{SLk\ln(m)/m}
    \end{align*}
    Now setting $\lambda=\frac{1}{m}$, we can assert that the last term dominates and thus we get that
    \[
        \int_0^\infty \sqrt{\log E(T_\alpha, \norm*{\cdot}_X, t)}\,dt \leq O(G_S\alpha\ln(m)^{3/2}\sqrt{SLk/m})
    \]
\end{proof}

\subsubsection{Proof of the main sampling result}
We are now ready to prove the main theorem of this section.

\boundedderivative*

\begin{proof}
    We set $\ell=2 \lceil\ln(\delta^{-1})\rceil$, and $m=O(C \eps^{-2} {k \ell \ln(m)^3})$.
    Let $\alpha={C_0\beta}/{k} $ for a fixed $\beta \in \{2^0, 2 ,\dots, 2^{\lfloor \log(k/C_0) \rfloor}\}$, and $C_0$ to be chosen later.
    By Lemma \ref{lem:errorboundbyGp} it holds that $  \E(G_S-1)^\ell\leq ({2\pi})^{\ell/2} \E(\Lambda^\ell)$ where $\Lambda=\sup_{t \in T_\alpha} \left|\frac{1}{m}\sum_{i=1}^m \sigma_i w_i g(a_i t)\right| / f(t) $.
    Note that $X_t=\frac{1}{m}\sum_{i=1}^m \sigma_i w_i g(a_i t)$ defines a Gaussian process. We have that $\Lambda=\sup_{t \in T_\alpha} X_t/f(t) \leq \sup_{t \in T_\alpha}X_t/\alpha$.
    Using \Cref{cor:diambound} we get that the diameter is bounded by
    \[
        \mathcal D=12G_S \alpha\left(\frac{S B L^2 k}{mg(0)}\right)^{1/2} = \Theta( \eps G_S \alpha /\sqrt{\ell} )
    \]
    and using \Cref{lem:entropyboundedderviative} we get that the entropy is bounded by 
    \[
        \mathcal E= O(G_S\alpha \ln(m)^{3/2}\sqrt{SLk/m})=O(\eps G_S\alpha ).
    \]
    \Cref{lem:errorboundbyGp} relates the error term to the Gaussian process, and we continue with \Cref{lem:moment-boundentropy} to bound its $\ell$-th moments
\begin{align*}
        \E_S (G_S-1)^\ell 
        &\overset{Lem~\ref{lem:errorboundbyGp}}{\leq} \E_S({2\pi})^{\ell/2} \E_\sigma \abs{\Lambda}^\ell \\
        &\leq \E_S({2\pi})^{\ell/2} \alpha^{-\ell} \E_\sigma \sup\nolimits_{t\in T_\alpha}\abs{X_t}^\ell \\
        &\overset{Lem~\ref{lem:moment-boundentropy}}{\leq} \E_S({2\pi})^{\ell/2} \frac{(2\mathcal E)^\ell (\mathcal E/\mathcal D) + O(\sqrt \ell \mathcal D)^{\ell}}{\alpha^\ell} \\
        &\leq \E_S((2e)^{-1}\varepsilon G_S)^\ell 
        \leq \frac{\delta\eps^{\ell}}{2\cdot2^\ell} \E_S ((G_S-1)+1)^\ell \\
        &\leq\frac{\delta\eps^{\ell}}{2\cdot2^\ell} \E_S 2^\ell((G_S-1)^\ell+1)
        \leq \varepsilon^\ell \E_S (G_S-1)^\ell+\frac{\delta}{2} \varepsilon^\ell.
    \end{align*}
    By rearranging the terms, this implies that
    \begin{align*}
        \E_S (G_S-1)^\ell \leq \frac{\delta \varepsilon^\ell}{2(1-\varepsilon^\ell)}\leq \delta\varepsilon^\ell\, .
    \end{align*}
    
    Using Markov's inequality we get that
    \[
        \Pr(G_S-1\geq \varepsilon  ) = \Pr((G_S-1)^\ell\geq \varepsilon  ^\ell)\leq \frac{\delta\varepsilon^\ell}{\varepsilon  ^\ell}  = \delta\, .
    \]
    Thus, it holds with probability $1-\delta$ that
    \[
        \forall x\in T_\alpha \colon \left|\frac{\int g(a x)\,d\mathcal{P}(a) - \frac{1}{m}\sum_{i=1}^m w_i g(a_i x)}{f(x)}\right| \leq G_S-1\leq \eps\,.
    \]
    Recall that $f(x)\geq g(0)^2/(4(LB)^2k)$ by \Cref{lem:lb-opt}. We set $C_0=g(0)^2/(4(LB)^2)$ accordingly.
    Substituting $\delta$ by $\delta / \log((SL)^2k/C_0)$, we can use the union bound to apply the result to all $\alpha=C_0\beta/k$ for $\beta \in \{2^0, 2^1 ,\ldots, 2^{\lfloor \log((SL)^2k/C_0) \rfloor}\}$ simultaneously.
    
    For larger values of $\alpha$, we have $f(x) \geq \alpha \geq  C_02^{\lceil \log((SL)^2k/C_0) \rceil} /k \geq (SL)^2$. Note that we have $s(a)\geq  \|a\|_2^2 + 2 \geq \|a\|_2 + 1$. We can thus apply Theorem \ref{thm:norm_sampling} to $T=\{x\in\R^d\mid f(x)\geq (SL)^2\}$ and thus with $OPT_{T}\geq (SL)^2$, which yields $O((SL)^2 \varepsilon^{-2} k \ln(\delta^{-1}) / \opt_T) = O(\varepsilon^{-2} k \ln(\delta^{-1})) \leq O(m)$ sample size.

    The improved constant $C=SLk/g(0)$ follows verbatim by replacing \Cref{cor:diambound} with \Cref{cor:diambound2linear}, which gives $\mathcal D = \Theta(SL(k+\exp(B))/g(0))$ together with the additional assumption that $B= O(\log(k))$.
\end{proof}

\subsection{Sampling bound for \texorpdfstring{$\ell_1$}{L1} regularization}\label{sec:linearboundl1}

In this subsection we study the setting where $$f(x)=\int g(\langle a,  x\rangle)\,d\mathcal{P}(a)+\frac{\Vert x \Vert_1}{k}.$$
We work again with our standard $L$-Lipschitz and $B$-bounded distribution assumptions \ref{ass:L} and \ref{ass:B}, and set $s(a)\geq \|a\|_2 + g(0)$. We note that the analysis will be similar to the one in \Cref{sec:boundedderiveative}. Therefore, several details and definitions are borrowed from the previous sections and will not be explained again in all details.

Fix some subset $T\subseteq \R^d$. Given a sample $S$ recall that we set  $G_S=\sup\nolimits_{x \in T}1+|f_{S}(x)-f_0(x)|/f(x)$ where $f_S(x)=\frac{1}{m}\sum_{i=1}^m w_i g(a_i x)$. Note that $\sqrt{G_S}\leq G_S$ since $G_S \geq 1$. We also remind of the bound $f_S(x)/f(x) \leq G_S$, see \Cref{lem:GSbound}.

A standard symmetrization argument relates the error $E=G_S-1$ to a Gaussian process defined over $T$, see \Cref{lem:errorboundbyGp}. 
We then consider the Gaussian process $X_s=\left|\frac{1}{m}\sum_{i=1}^m \sigma_i w_i g(a_i s) \right| $ for $s \in T_\alpha=\{ x \in \mathbb{R}^d ~|~ f(x) \in [\alpha, 2\alpha) \}$ for some $\alpha \in \mathbb{R}$. 
We note that for any $x \in T_\alpha$ it holds that $\Vert x \Vert_1 \leq 2 \alpha k$ and thus $T_\alpha \subseteq 2 \alpha k B_1$.
The metric $d_X$ is again given by
\begin{align*}
    d_X(x, x')&=\left(\frac{1}{m^2}\sum_{i=1}^m  w_i^2 (g(a_i x) - g(a_i x'))^2  \right)^{1/2}\,.
\end{align*}
and thus, the diameter bound is similar to the one in \Cref{sec:diam_boundedderiveative}. 

\subsubsection{Bounding the diameter}

\begin{lemma}\label{lem:sensitivityboundlinear}
    Let $w_i=\frac{2S}{s(a_i)+S}$ and $s(a_i) \geq \Vert a_i \Vert_2$. Then for any $i \in [m]$ and any $x \in T_\alpha$ it holds that
    \begin{align*}
        \frac{w_i g(a_ix)}{f(x)} \leq 4SLk\,. 
    \end{align*}
\end{lemma}
\begin{proof}
    Recall that $s(a_i)\geq \|a_i\|^2_2+1\geq \max\{\|a_i\|_2,1\}$. Also recall that $f(x)\geq \frac{g(0)}{BLk}$, and $f(x)\geq \frac{\|x\|_1}{k}$. Thus, it holds that
    \begin{align*}
        \frac{w_i g(a_ix)}{f(x)}&= \frac{w_i(g(0)+g_0(a_i x))}{f(x)} 
        \leq \frac{2S}{s(a_i)+S} \cdot \frac{g(0)+L |a_i x|}{f(x)} \\
        &\overset{CSI}{\leq} \frac{2S}{S} \frac{g(0)}{f(x)} + \frac{2S}{\|a_i\|_2} \frac{L \|a_i\|_2 \|x\|_2}{f(x)} 
        \leq 2 \left({BL k } + \frac{S L \|x\|_2}{f(x)} \right) \leq 4{S L k}\,. 
    \end{align*}\qedhere
\end{proof}

We have the following bound on the diameter.
\begin{lemma}\label{lem:diambound_l1}
    It holds that $\sup\limits_{s, t\in T_\alpha} d_X(s,t) \leq 6G_S\alpha\left(\frac{SL k }{m}\right)^{1/2}$.
\end{lemma}
\begin{proof}
    By \cref{lem:diamsensbound} we have that $\sup\limits_{s, t\in T_\alpha} d_X(s,t) \leq 2G_S \alpha\left(\frac{2}{m} \sup_{x\in T_\alpha, j \in [m]}\frac{w_j g(a_j x)}{f(x)}\right)^{1/2}
    $. Using \Cref{lem:sensitivityboundlinear}, yields that
    \[
        \sup\limits_{s, t\in T_\alpha} d_X(s,t) \leq 6G_S \alpha\left(\frac{SLk}{m} \right)^{1/2}
    \]\qedhere
\end{proof}

\subsubsection{Relating the metric to a norm}
This time we do not need a proper norm, the square root of an $\ell_\infty$-norm will suffice for our purposes.
\begin{lemma}\label{lem:metricboundl1}
    Let $A$ be the matrix with rows $w_i a_i$. Then for any $s, t \in T_\alpha$ it holds that
    \begin{align*}
        d_X(s,t)\leq \left( \frac{4G_S \alpha L \Vert A(s-t)\Vert_\infty }{m}\right)^{1/2} .
    \end{align*}
\end{lemma}

\begin{proof}
    We have 
        \begin{align*}
        d_X(s,t)^2 
        & = \frac{1}{m^2}\sum_{i=1}^m  w_i^2 \left(g(a_i s)-g(a_i t)\right)^2 \\
        & \leq \frac{1}{m^2}\sum_{i=1}^m  w_i^2 \left|g(a_i s)-g(a_i t)\right| \cdot(g(a_i s)+g(a_i t)) \\
        &\leq \frac{1}{m^2}\sum_{i=1}^m  w_i^2 L|a_is-a_it| \cdot(g(a_i s)+g(a_i t)) \\
        &\leq \frac{L}{m^2}\sum_{i=1}^m  \max_{j\in [m]}w_j |a_j(s-t)|\cdot w_i (g(a_i s)+g(a_i t))  \\
        &\leq \frac{L}{m}\max_{j \in [m]} w_j |a_j(s-t)| \cdot \left(\frac{1}{m}\sum_{i=1}^m  w_i g(a_i s) + \frac{1}{m}\sum_{i=1}^m  w_i g(a_i t)\right) \\
        &\leq \frac{L}{m} \max_{j \in [m]} w_j |a_j(s-t)| \cdot G_S (f(s) + f(t))\\
        &\leq \frac{4G_S \alpha L}{m} \max_{j\in [m]} |(w_ja_j)(s-t)| \,.
    \end{align*}
\end{proof}

\subsubsection{Bounding the covering numbers}

We set $A \in \mathbb{R}^{m\times d}$ to be the matrix with rows $w_ia_i$ and for $p \in \mathbb{R}_{\geq0}$ we set $B_p(A)=\{x \in \mathbb{R}^d ~|~ \Vert Ax \Vert_p\leq 1 \}$, and $B_p=B_p(I_d)$. We first bound the covering numbers for covering $B_1$ with $B_\infty(A)$ balls.

\begin{lemma}\label{lem:covlemmal1linf}
    For any radius $t>0$ it holds that 
    \begin{align*}
        \log E(B_1, B_\infty(A), t)\leq  O(S)\frac{\ln(m)}{t}\,.
    \end{align*}
\end{lemma}

\begin{proof}
    Assume the rank of $A\in\R^{m\times d}$, is bounded by $m<d$.
    Let $A=U\Sigma V^T$ be the singular value decomposition. The matrix $V\in\R^{d\times m}$ is a matrix with orthonormal columns that span the same subspace as $A$.
    We first cover $B_1$ with $B_2(V^T)$ and then we cover $B_2(V^T)$ with $B_\infty(A)$.
    Note that $B_2(V^T)$ is isometric to the Euclidean ball in dimension $m$. 
    
    For the first part, we aim to use \Cref{prop:sudakov} with $\norm*{\cdot}_X=\norm*{\cdot}_1$. We note that the dual norm is $\norm*{\cdot}_{X^*}=\norm*{\cdot}_\infty$. 
    We thus need to bound
    $\E_{g\sim\mathcal N(0,I_m)}\norm*{Vg}_\infty$. However, we cannot use \Cref{lem:expmaxGaussian} directly since this is a $d$-dimensional Gaussian. Since it is projected to an $m$-dimensional subspace, there is hope that $O(\sqrt{\log(m)})$ will suffice:
    
    Each row corresponds to a Gaussian as $v_i g = \sum_{j=1}^m V_{ij} g_j \sim \|v_i\|_2 \cdot g',$ where $g'\sim N(0,1)$ In particular $|v_i g|^2 \sim \|v_i\|_2^2 \cdot g'^2$ follows a squared Gaussian distribution. It is also known that since $V$ is an orthonormal basis, its squared row norms are the statistical leverage scores, for which it is known \citep{Foster53} that they are bounded by $\|v_i\|_2^2\leq 1$ and sum to the rank, so $\sum_{i=1}^d \|v_i\|_2^2 \leq m < d.$

    Flattening lemmas take high leverage points and replicate them several times to replace one heavy coordinate by several ones with uniformly small leverage \cite{woodruffyasuda23}. Reversely to these result, we aim to combine rows in such a way that we end up with $\Theta(m)$ rows, each with $\Theta(1)$ leverage.

    To this end, consider \Cref{lem:expmaxsquared2}. Using this lemma, whenever, we have a group $G$ of rows, i.e., Gaussians $v_i g, i\in G$, we can replace them by a row vector $v\in \R^d$, i.e., by a Gaussian $vg$ that has variance $\|v\|_2^2 = \sum_{i\in G} \|v_i\|_2^2$ to aggregate the norm of the entire group into one single row. The lemma ensures that the expected maximum can only grow.
    
    Due to the aforementioned properties of leverage scores, we now put the rows of $V$ into at most $2m$ groups $G_i$ such that $1/2 \leq \sum_{j\in G_i} \|v_j\|_2^2 \leq 1$ (e.g. by greedy bin-packing). We use \Cref{lem:expmaxsquared2} with $\sigma_{\max}\leq 1$ to aggregate all vectors within $G_i$ to only one vector $w_i$ with $\|w_i\|_2^2=\sum_{j\in G_i} \|v_j\|_2^2 \leq 1$ and arrange them into a new matrix $W\in \R^{2m\times m}$.
    
    This yields 
    \begin{align*}
        \left( \E_{g\sim\mathcal N(0,I_m)}\norm*{Vg}_\infty \right)^2
        &\overset{Jensen}{\leq} \E_{g\sim\mathcal N(0,I_m)} {\max_{i \in [d]} | \langle v_i, g \rangle |^2} \\
        &\overset{~}{\leq} \E_{g\sim\mathcal N(0,I_m)} {\max_{i \in [2m]} | \langle w_i, g \rangle |^2} 
        = O({\ln(m)})\,,
    \end{align*}
    where we used Jensen's inequality followed by the above argument, and then followed by \cref{lem:expmaxGaussian}.
        
    We thus get from \Cref{prop:sudakov} that
    \begin{align*}
        \log E(B_1, B_2(V^T), t) \leq O(1)\frac{\ln(m)}{t^2}\,.
    \end{align*}
    
    Similarly, we have that $A_i = w_ia_i \leq 2S{a_i}/{\|a_i\|_2}$ and thus 
    \begin{align*}
        \E_{g\sim\mathcal N(0,I_d)}\norm*{Ag}_\infty
        &\leq 2S \E_{g\sim\mathcal N(0,I_d)} \max_{i \in m} \left| \left\langle \frac{a_i}{\|a_i\|_2}, g \right\rangle \right| =O(S \sqrt{\ln(m)}) \,.
    \end{align*}

    where the last inequality follows from \Cref{lem:expmaxGaussian} and the fact that each $Z_i = \left\langle \frac{a_i}{\|a_i\|_2} , g \right\rangle$ is identically distributed as a zero mean Gaussian with unit variance $\left\|\frac{a_j}{\|a_j\|_2}\right\|_2^2=1$ by self-stability of the Gaussian distribution.
    
    Using the dual Sudakov bound stated in \cref{prop:dual-sudakov}, we get that
    \begin{align*}
        \log E(B_2(V^T), B_\infty(A), t) \leq O(S)\frac{\ln(m)}{t^2}
    \end{align*}
    Combining both results implies that
    \begin{align*}
        \log E(B_1, B_\infty(A), t)
        &\leq \log E(B_1, B_2(V^T), \sqrt{t})+\log E(\sqrt{t}B_2(V^T), B_\infty(A), t) \\
        &= \log E(B_1, B_2(V^T), \sqrt{t})+\log E(B_2(V^T), B_\infty(A), \sqrt{t}) \\
        &\leq O(1)\frac{\ln(m)}{t}+O(S)\frac{\ln(m)}{t} =O(S)\frac{\ln(m)}{t}
    \end{align*}\qedhere
\end{proof}

Using this result, we bound in the following lemma the covering numbers for covering $T_\alpha$ with balls according to our metric.
\begin{lemma}\label{lem:covlemmal1new}
    For any radius $t>0$ 
    it holds that 
    \begin{align*}
        \log E(T_\alpha, d_X, t)\leq  O(1)\frac{G_S \alpha^2 S L k}{mt^2}\log(m)\,.
    \end{align*}
    and 
    \begin{align*}
        \log E(T_\alpha, d_X, t)\leq O(1)m\log\left( \frac{G_S \alpha^2 S L k }{m t^2}\right)\,.
    \end{align*}
\end{lemma}

\begin{proof}
    Recall that $T_\alpha \subseteq 2\alpha k B_1 $.
    Using \Cref{lem:metricboundl1,lem:covlemmal1linf}, we get that
    \begin{align*}
        \log E(T_\alpha, d_X, t)
        &\leq  \log E\left(2\alpha k B_1, \left( \frac{4G_S \alpha L \Vert A \cdot \Vert_\infty }{m}\right)^{1/2}, t\right)\\
        &= \log E\left(2\alpha k B_1,  B_\infty(A),   \frac{m t^2}{4G_S \alpha L  } \right)\\
        &= \log E\left(B_1, B_\infty(A),   \frac{m t^2}{8G_S \alpha^2 L k } \right)\\
        &\leq O(1)\frac{G_S \alpha^2 S L k}{mt^2}\log(m)\,.
    \end{align*}
    
    For the second part, assume the rank of $A\in\R^{m\times d}$, is bounded by $m<d$.
    Let $A=U\Sigma V^T$ be the singular value decomposition. The matrix $V\in\R^{d\times m}$ is a matrix with orthonormal columns that span the same subspace as $A$.
    We have $B_1 \subseteq B_2 \subseteq B_2(V^T)$ as the projection to the subspace can make the norm only smaller, i.e., $\|V^Tx\|_2\leq\|x\|_2\leq \|x\|_1\leq 1$.
    
    Recall that $w_i = \frac{2S}{s(a_i) + S}\leq \frac{2S}{\|a_i\|_2}$, and let $A'$ be the matrix whose rows equal $a'_i = a_i/\|a_i\|_2$. 
    
    For $x\in B_2(V^T)/(2S)$ it holds by definition that $\|V^Tx\|_2 \leq 1/(2S)$. Using $V^TV=I$, this implies
    \begin{align*}
        \|Ax\|_\infty
        &=\|U\Sigma V^Tx\|_\infty 
        =\|U\Sigma V^TVV^Tx\|_\infty \\
        &= \|AVV^Tx\|_\infty 
        \leq 2S \|A'VV^Tx\|_\infty \\
        &= 2S \max_{i \in[m]} x^TVV^Ta'_i
        \leq 2S \sup_{a \in B_2(V^T)} x^TVV^Ta'_i \\
        &= 2S\cdot x^TVV^Tx/\|V^Tx\|_2 
        = 2S \|V^Tx\|_2 \leq 1\,.
    \end{align*}
    where the second inequality follows since all $a_i' \in B_2(V^T)$.
    Thus, we have that $B_2(V^T)/(2S)\subseteq B_\infty(A)$.
    
    Consequently, it holds that $$\log E(B_1, B_\infty(A), t) \leq \log E(B_2(V^T), B_2(V^T), t/(2S)) \leq O(1)m\log(S/t),$$
    where the last inequality follows from the standard volume argument, see \Cref{lem:ballcover}.
    Similar to our previous bound, we arrive at
    \begin{align*}
        \log E(T_\alpha, d_X, t)
        &\leq \log E\left(B_1, B_\infty(A), \frac{m t^2}{8G_S \alpha^2 L k} \right)\\
        &\leq O(1)m\log\left( \frac{G_S \alpha^2 S L k}{m t^2}\right)\,.
    \end{align*}\qedhere
\end{proof}

\subsubsection{Bounding the metric entropy}
It remains to bound the entropy integral.
\begin{lemma}\label{lem:entropyl1}
    Let $m\geq \frac{CSLk}{\eps^2}$, for a sufficiently large absolute constant $C>1$, and assume that $\alpha \leq 1/\eps$.
    Then it holds that $\int_0^\infty \sqrt{\log E(T_\alpha, \norm*{\cdot}_X, t)}\,dt \leq O(\ln(m)^{3/2}G_S\alpha\sqrt{SLk/m})$
\end{lemma}

\begin{proof}
    Again we split the entropy integral at some small $\lambda \ll 1$ into two parts. We use \Cref{lem:covlemmal1new} and the technical \Cref{lem:calc}. It is sufficient to integrate up to the diameter, since for larger radii, the covering number is simply $1$ and its logarithm becomes $0$. By \Cref{lem:diambound_l1} we have that $\mathcal D \leq \eps G_S \alpha \leq G_S \leq \poly(m)$ for our choice of $m$ using \Cref{lem:GSpolym} and the assumption that $\alpha\leq 1/\eps$.

    Using Lemma \ref{lem:covlemmal1new} and the fact that $\sqrt{G_S}\leq G_S$, we get that
    \begin{align*}
        &\int_0^\infty \sqrt{\log E(T_\alpha, d_X, t)}\,dt \\    &\leq\int_0^{\lambda} O(1)\sqrt{m\log\left( \frac{G_S \alpha^2 S L k}{m t^2}\right)} \,dt
        ~+~\int_{\lambda}^{\mathcal D} O(1)\frac{G_S \alpha }{t}\sqrt{\frac{SLk}{m}\log(m)}\,dt \\
        &\leq\int_0^{\lambda} O(1)\sqrt{m\log\left( \frac{1}{t}\right)} \,dt
        + \int_0^{\lambda} O(1)\sqrt{m\log\left( \frac{G_S \alpha S L k }{m}\right)} \,dt 
        + \int_{\lambda}^{\mathcal D} O(1)\frac{G_S \alpha }{t}\sqrt{\frac{S L k}{m}\log(m)}\,dt \,. \\
        &\leq O(\lambda) \sqrt{m\log\left( \lambda^{-1} \right)}
        + O(\lambda)\sqrt{m\log\left( \frac{G_S \alpha S L k }{m}\right)} 
        + O(\log(\mathcal D)+\log(\lambda^{-1})){G_S \alpha }\sqrt{\frac{S L k}{m}\log(m)}\,dt \,.
    \end{align*}
    Now setting $\lambda=\frac{1}{m}$, we can assert that the last term dominates and thus we get that
    \[
        \int_0^\infty \sqrt{\log E(T_\alpha, \norm*{\cdot}_X, t)}\,dt \leq O(G_S\alpha\ln(m)^{3/2}\sqrt{SLk/m})
    \]
\end{proof}

\subsubsection{Proof of the main sampling result}
We are now ready to prove our general bound for $\ell_1$ regularized functions.

\generallonethm*

\begin{proof}
    We set $\ell=2 \lceil\ln(\delta^{-1})\rceil$, and $m=C \cdot \frac{k \ell \ln(m)^3}{\eps^2}$.
    Let $\alpha={C_0\beta}/{k} $ for some fixed $\beta \in \{2^0, 2 ,\dots, 2^{\lfloor \log(k/\eps C_0) \rfloor}\}$, and $C_0$ to be chosen later.
    By Lemma \ref{lem:errorboundbyGp} it holds that $  \E(G_S-1)^\ell\leq ({2\pi})^{\ell/2} \E(\Lambda^\ell)$ where $\Lambda=\sup_{t \in T_\alpha} \left|\frac{1}{m}\sum_{i=1}^m \sigma_i w_i g(a_i t)\right| / f(t) $.
    Note that $X_t=\frac{1}{m}\sum_{i=1}^m \sigma_i w_i g(a_i t)$ defines a Gaussian process. We have that $\Lambda=\sup_{t \in T_\alpha} X_t/f(t) \leq \sup_{t \in T_\alpha}X_t/\alpha$.
    Using \Cref{lem:diambound_l1} we get that the diameter is bounded by
    \[
        \mathcal D=6G_S\alpha\left(\frac{SL k }{m}\right)^{1/2} = \Theta( \eps G_S \alpha /\sqrt{\ell} )
    \]
    and using \Cref{lem:entropyl1} we get that the entropy is bounded by 
    \[
        \mathcal E= O(G_S\alpha \ln(m)^{3/2}\sqrt{SLk/m})=O(\eps G_S\alpha ).
    \]
    \Cref{lem:errorboundbyGp} relates the error term to the Gaussian process, and we continue with \Cref{lem:moment-boundentropy} to bound its $\ell$-th moments
    \begin{align*}
        \E_S (G_S-1)^\ell 
        &\overset{Lem~\ref{lem:errorboundbyGp}}{\leq} \E_S({2\pi})^{\ell/2} \E_\sigma \abs{\Lambda}^\ell \\
        &\leq \E_S({2\pi})^{\ell/2} \alpha^{-\ell} \E_\sigma \sup\nolimits_{t\in T_\alpha}\abs{X_t}^\ell \\
        &\overset{Lem~\ref{lem:moment-boundentropy}}{\leq} \E_S({2\pi})^{\ell/2} \frac{(2\mathcal E)^\ell (\mathcal E/\mathcal D) + O(\sqrt \ell \mathcal D)^{\ell}}{\alpha^\ell} \\
        &\leq \E_S((2e)^{-1}\varepsilon G_S)^\ell 
        \leq \frac{\delta\eps^{\ell}}{2\cdot2^\ell} \E_S ((G_S-1)+1)^\ell \\
        &\leq\frac{\delta\eps^{\ell}}{2\cdot2^\ell} \E_S 2^\ell((G_S-1)^\ell+1)
        \leq \varepsilon^\ell \E_S (G_S-1)^\ell+\frac{\delta}{2} \varepsilon^\ell.
    \end{align*}
    By rearranging the terms, this implies that
    \begin{align*}
        \E_S (G_S-1)^\ell \leq \frac{\delta \varepsilon^\ell}{2(1-\varepsilon^\ell)}\leq \delta\varepsilon^\ell\, .
    \end{align*}
    
    Using Markov's inequality we get that
    \[
        \Pr(G_S-1\geq \varepsilon  ) = \Pr((G_S-1)^\ell\geq \varepsilon  ^\ell)\leq \frac{\delta\varepsilon^\ell}{\varepsilon  ^\ell}  = \delta\, .
    \]
    
    Thus, it holds with probability $1-\delta$ that
    \[
        \forall x\in T_\alpha \colon \left|\frac{\int g(a x)\,d\mathcal{P}(a) - \frac{1}{m}\sum_{i=1}^m w_i g(a_i x)}{f(x)}\right| \leq G_S-1\leq \eps\,.
    \]
    This concludes the claim for $x\in T_\alpha$. If $g$ is homogeneous, then it suffices to apply the above for a fixed $\alpha=1$. The result extends to all $x\in\R^d$ simply by scaling. 
    
    Otherwise, recall that $f(x)\geq g(0)/(LBk)$ by \Cref{lem:lb-opt}. We set $C_0=1/(LB)$.
    Substituting $\delta$ by $\delta / \log(k/(\eps C_0))$, we can use the union bound to apply the result to all $\alpha=g(0)C_0\beta/k$ for $\beta \in \{2^0, 2^1 ,\ldots, 2^{\lfloor \log(k/(\eps C_0)) \rfloor}\}$ simultaneously. This handles all $x\in \{x\in \R^d \mid f(x)\leq g(0)/\eps\}$ with the stated sample size.
\end{proof}

For larger values of $\alpha$, not handled by \Cref{thm:generallonethm}, we have $f(x) \geq \alpha = g(0)C_02^{\lceil \log(k/(\eps C_0)) \rceil} /k \geq g(0)/\eps$. This enables a different argument, if $g(\cdot)$ can be decomposed into a homogeneous part and a bounded part (as all our example functions).
\begin{lemma}\label{lem:convergencel1}
    Assume that $g(x)=h(x)+b(x)$ such that $h(x)$ is homogeneous, i.e., for any $\lambda \in \mathbb{R}_{\geq 0}$ it holds that $h(\lambda x)=\lambda h(x)$.
    Further assume that $b(x)$ is bounded by $g(0)$, i.e., it holds for all $r\in \R$ that $b(r)\leq g(0)$.
    If 
    \[
        \left|f_0(x)-\left(\frac{1}{m}\sum_{i=1}^m w_i g(a_i x)\right) \right| \leq \eps f(x).
    \]
    holds for all $x\in \R^d$ such that $f(x)\leq g(0)/\eps$, then the inequality holds for all $x\in \R^d$ with $\eps'=4\eps$.
\end{lemma}

\begin{proof}
    We need to prove the claim for all $x\in\R^d$ for which it holds that $f(x)\geq g(0)/\eps$. Note that for the homogeneous part, if we have a relative error guarantee e.g. for $h(x)=1$, then it is also a relative error guarantee for $\lambda h(x) = h(\lambda x)$ for any scaling $\lambda\geq 0$. Thus the guarantee extends to all $x\in \R^d$. Recall that $w_i=\frac{2S}{s(a_i)+S}\leq 2$. By our assumption and the triangle inequality, we have
    \begin{align*}
        &\left|f_0(x)-\left(\frac{1}{m}\sum_{i=1}^m w_i g(a_i x)\right) \right| \\
        &= \left|\int h(ax)+b(ax)\,d\PP(a)-\left(\frac{1}{m}\sum_{i=1}^m w_i (h(a_ix)+b(a_ix))\right) \right| \\
        &\leq \left|\int h(ax)\,d\PP(a)-\left(\frac{1}{m}\sum_{i=1}^m w_i h(a_ix)\right) \right|
        + \left|\int b(ax)\,d\PP(a)-\left(\frac{1}{m}\sum_{i=1}^m w_i b(a_ix)\right) \right|\\
        &\leq \eps h(x) + \left| \int b(ax)\,d\PP(a) \right| + \left|\frac{1}{m}\sum_{i=1}^m w_i b(a_ix)\right| \\
        &\leq \eps h(x) + g(0) + 2g(0) \leq \varepsilon f(x) + 3\varepsilon f(x) = 4\eps f(x).
    \end{align*}\qedhere
\end{proof}

\section{Lower bounds}\label{sec:LB}
In this section, we prove various lower bounds.
\subsection{Quadratic lower bounds}\label{sec:LBquadratic}
We prove a series of quadratic lower bounds summarized in the following theorem:

\thmquadraticlowerbounds*

\begin{proof}
    The proof is treated in \Cref{lem:lowerbound_logreg} for logistic loss and \Cref{cor:lowerbound_sigmoid} for sigmoid loss, both with $\RR(\cdot)=\norm{\cdot}_2$. The remaining combinations are treated in \Cref{lem:lowerboundallhinges} for hinge loss, and in \Cref{lem:lowerboundallrelus} for ReLU loss. Both results work with either of the regularizers $\RR(\cdot)\in\{\norm{\cdot}_2,\norm{\cdot}_1\}$.
\end{proof}

We get an $\Omega((k/\eps)^2)$ LB against logistic regression with $\norm*{\cdot}_2$.

\begin{lemma}\label{lem:lowerbound_logreg}
    Consider the distribution $\mathcal{P}$ with $P(X=e_j)=1/d$, where $e_j$ is the $j$-th unit vector.
        Let $\varepsilon \in (0, 1/10)$. Consider any sample $(a_i, w_i)_{i \in [m]}$  such that each $(a_i, w_i)$ is set to $a_i=e_j$ and $w_i=1/(dp_j)$ with some probability $p_j \in [0,  1]$.
        Further let $g : \mathbb{R} \rightarrow \mathbb{R}_{\geq 0}$ be the logistic loss function with
        \begin{align*}
            g(r)=\ln(1+\exp(-r))\,.
        \end{align*}
        If $\RR(x)=\Vert x\Vert_2$, $d=2(k\ln(2)/40\eps)^2$ and $m\leq d/2$ then there exists $x \in \mathbb{R}^d$ such that
        \[
            \left| \frac 1 d \sum_{j=1}^d g(x e_j) - \frac{1}{m}\sum_{i=1}^m w_ig(x a_i)\right| >  \varepsilon \left(\frac 1 d  \sum_{j=1}^d g(x e_j) + \frac{\RR(x)}{k}\right).
        \]
    \end{lemma}

    \begin{proof}
        We consider the case $\RR(x)=\Vert x\Vert_2$, $d=2(k\ln(2)/40\eps)^2$. For notational convenience, we scale the loss and regularizer by a factor $1/\ln(2)$. This obviously preserves the multiplicative error guarantee. For the proof consider $g(r)=\ln(1+\exp(-r))/\ln(2)$.

        Let $m\leq d/2$. Without loss of generality, we assume that there are no samples $a_i=e_j$ for any $j \leq d/2$.
        Let $x=\sum_{j=1}^{d/2}e_j$. Note that for half of the unit vectors, we have $g(e_jx)=g(0)=\ln(2)/\ln(2)=1$, and for the other half it holds that $g(e_jx)=g(1)=\ln(1+\exp(-1))/\ln(2)$.
        We define $c=\ln(1+\exp(-1))/(2\ln(2))$ and observe that $c \approx 0.225971 < \frac{1}{4}$.
        
        The (rescaled) loss of $x$ thus equals
        \begin{align*}
            \frac{1}{d}\sum_{j=1}^d g(e_j x) = \frac12 \left( 1+\frac{\ln(1+\exp(-1))}{\ln(2)}\right) = \frac12 + c \,.
        \end{align*}
        For the (rescaled) regularization term, we have
        \[
            \frac{\RR(x)}{k\ln(2)}=\frac{\sqrt{d/2}}{k\ln(2)} = {\frac{\sqrt{(k\ln(2)/40\eps)^2}}{k\ln(2)}}=\frac{1}{40\eps}\,.
        \]
        Overall, we get that
        \begin{align*}
            \frac{1}{d}\sum_{j=1}^d g(e_j x) + \frac{\RR(x)}{k\ln(2)} = \frac12 + c + \frac{1}{40\eps}
        \end{align*}
        and the loss of taking $x=0$ instead equals
        $\frac{1}{d}\sum_{j=1}^d g(0) = 1 \,.$
        
        By construction of $x$ it always holds that $a_ix = 0$ for the subsampled version and thus
        \begin{align*}
            \frac{1}{m}\sum_{i=1}^m w_ig(a_i x)=\frac{1}{m}\sum_{i=1}^m w_ig(0)\,.
        \end{align*}
        Now assume $\frac{1}{m}\sum_{i=1}^m w_ig(0)\in [1-\varepsilon, 1+\varepsilon] $. Then using $\eps\leq 1/10$ and $c < 1/4$ we have that
        \begin{equation}\label{eq:help}
            \frac12-c-\eps > \frac12-\frac14-\eps > \frac14-\frac{1}{10} > \frac{1}{10} = \frac{1}{20} + \frac{1}{40} + \frac{1}{40} > \frac{\eps}{2} + \eps c + \frac{1}{40}\,.
        \end{equation}
        It follows that
        \begin{align*}
            \frac{1}{m}\sum_{i=1}^m w_ig(a_i x) + \frac{\RR(x)}{k\ln(2)}
            &\;\;\,\,\geq \;\;\,\,\, 1-\eps+\frac{1}{40\eps} 
            = \frac{1}{2}+c+\frac{1}{40\eps} + \frac{1}{2} - c - \eps \\
            &\overset{\cref{eq:help}}{>}  \frac{1}{2}+c+\frac{1}{40\eps} + \frac{\eps}{2} + \eps c + \frac{1}{40}
            = (1+\eps) \left(\frac{1}{2}+c+\frac{1}{40\eps}\right)
            \,.
        \end{align*}
        We conclude that for at least one choice of $x' \in \{x, 0\}$ it holds for the original loss and regularizer that
        \[
            \left| \frac{1}{d} \sum_{j=1}^d g(e_j x') - \frac{1}{m}\sum_{i=1}^m w_ig(a_i x')\right| >  \varepsilon \left(\frac{1}{d}\sum_{j=1}^d g(e_j x') + \frac{\RR(x')}{k}\right) \,.
        \]
    \end{proof}

    We get the same bound against sigmoid loss with $\norm*{\cdot}_2$ in a very similar way; only some constants change.

    \begin{corollary}\label{cor:lowerbound_sigmoid}
        Consider the distribution $\mathcal{P}$ with $P(X=e_j)=1/d$, where $e_j$ is the $j$-th unit vector.
        Let $\varepsilon \in (0, 1/10)$. Consider any sample $(a_i, w_i)_{i \in [m]}$  such that each $(a_i, w_i)$ is set to $a_i=e_j$ and $w_i=1/(dp_j)$ with some probability $p_j \in [0,  1]$.
        Further let $g : \mathbb{R} \rightarrow \mathbb{R}_{\geq 0}$ be the sigmoid loss function with
        \begin{align*}
            g(r)=\frac{1}{1+\exp(r)}\,.
        \end{align*}
        If $\RR(x)=\Vert x\Vert_2$, $d=2((k/2)/50\eps)^2$ and $m\leq d/2$ then there exists $x \in \mathbb{R}^d$ such that
        \[
            \left| \frac 1 d \sum_{j=1}^d g(x e_j) - \frac{1}{m}\sum_{i=1}^m w_ig(x a_i)\right| >  \varepsilon \left(\frac 1 d  \sum_{j=1}^d g(x e_j) + \frac{\RR(x)}{k}\right).
        \]
    \end{corollary}
    \begin{proof}
        The proof is exactly as the proof of \Cref{lem:lowerbound_logreg}. Only some constants change. In particular the scaling factor becomes $1/g(0)=2$ instead of $1/\ln(2)$, and the value of $c$ changes to $c=1/(1+\exp(1))<0.3$ instead of $c<1/4$.
        Using $d=2((k/2)/50\eps)^2$ instead of $d=2(k\ln(2)/40\eps)^2$ allow the remaining calculations to work verbatim.
    \end{proof}

    The lemma below gives $\Omega((k/\eps)^2)$ against hinge with both $\ell_2$ regularizers.
    \begin{lemma}\label{lem:lowerboundallhinges}
    Consider the distribution $\mathcal{P}$ with $P(X=v_j)=1/(d-1)$ where $v_j=e_d+e_j/\sqrt{2}$.
        Consider any sample $(a_i, w_i)_{i \in [m]}$ such that each $(a_i, w_i)$ is set to $a_i=v_j$ and $w_i=1/((d-1) p_j)$ with some probability $p_j \in [0,  1]$ and let $\varepsilon \in (0, 1/4)$.
        Further let $g : \mathbb{R} \rightarrow \mathbb{R}_{\geq 0}$ be the function with
        \begin{align*}
            g(r)=\begin{cases}
                \text{arbitrary} &\text{if $r < 0$}\\
                1-r &\text{if $r\in [0, 1]$}\\
                0 &\text{otherwise.}
            \end{cases}
        \end{align*}
        If $\RR(x)\in\{\|x\|_2, \|x\|_2^2\}$, $d=(k/6\eps)^2+1$ and $m\leq (d-1)/2$ then there exists an $x \in \mathbb{R}^d$ such that
        \[
            \left| \frac{1}{d-1} \sum_{j=1}^{d-1} g(v_j x) - \frac{1}{m}\sum_{i=1}^m w_ig(a_i x)\right| > \varepsilon \left( \frac{1}{d-1}\sum_{j=1}^{d-1} g(v_j x)  + \frac{\RR(x)}{k}\right).
        \]
    \end{lemma}

    \begin{proof}
        Without loss of generality we assume that there are no samples $a_i=v_j$ for any $j \leq (d-1)/2$.
        Let $x=e_d-\sum_{i=1}^{(d-1)/2}e_i/\sqrt{(d-1)/2}$.
        We make the following observations:
        \begin{enumerate}
            \item We have for both choices that $\RR(x)\leq 2$ since
            \begin{enumerate}
                \item $\RR(x)=\|x\|_2^2 = 1^2 + \sum_{i=1}^{(d-1)/2} (1/\sqrt{(d-1)/2})^2 = 2$\,,
                \item $\RR(x)=\|x\|_2 = \sqrt{\|x\|_2^2} = \sqrt{2} < 2$\,.
            \end{enumerate}
            \item Also note for $j>(d-1)/2$ that $v_j^T x = 1$ and thus $g(v_j^T x)=1-1=0$,\\
            by the assumption it also holds for all samples that $a_ix = 1$ resp. $g(a_ix)=g(1)=0$.
            \item Finally, for $j\leq (d-1)/2$ we have $v_j^T x = 1 - 1/\sqrt{2} \cdot 1/\sqrt{(d-1)/2} = 1 - 1/\sqrt{(d-1)}$ and \\thus $g(v_j^T x)=1 - 1 + 1/\sqrt{(d-1)} =  1/\sqrt{(d-1)}$.
        \end{enumerate}
        These observations together with the choice of $d-1=(k/6\eps)^2$ yield that
        \begin{align*}
            \left( \frac{1}{d-1} \sum_{j=1}^{d-1} g(v_j x) + \frac{\RR(x)}{k}\right)
            =\frac{1}{d-1}\frac{d-1}{2}\frac{1}{\sqrt{(d-1)}}+\frac{\RR(x)}{k}
            \leq \frac{6\eps}{2k} + \frac{2}{k} = \frac{4+6\eps}{2k}\,.
        \end{align*}
        Further we have that
        $$
            \frac{1}{m}\sum_{i=1}^m w_i g(a_i x)=\frac{1}{m}\sum_{i=1}^m w_i g(1) = 0\,.
        $$
        Using $\eps < 1/4$ we conclude that 
        \[
             \left| \frac{1}{m}\sum_{i=1}^m w_ig(a_i x)- \frac{1}{d-1}\sum_{j=1}^{d-1} g(v_j x) \right|
             = \left|\frac{1}{d-1}\sum_{j=1}^{d-1} g(v_j x)\right|
             = \frac{6\eps}{2k} > {\eps}\cdot\left( \frac{4+6\eps}{2k}\right).
        \]
    \end{proof}

    The lemma below gives $\Omega((k/\eps)^2)$ against ReLU with both $\ell_2$ regularizers.
    \begin{lemma}\label{lem:lowerboundallrelus}
    Consider the distribution $\mathcal{P}$ with $P(X=e_j)=1/d$.
        Consider any sample $(a_i, w_i)_{i \in [m]}$ such that each $(a_i, w_i)$ is set to $a_i=v_j$ and $w_i=1/(d p_j)$ with some probability $p_j \in [0,  1]$ and let $\varepsilon \in (0, 1/4)$.
        Further let $g : \mathbb{R} \rightarrow \mathbb{R}_{\geq 0}$ be the function with
        \begin{align*}
            g(r)=\begin{cases}
                -r &\text{if $r < 0$}\\
                0 &\text{otherwise.}
            \end{cases}
        \end{align*}
        If $\RR(x)\in\{\|x\|_2, \|x\|_2^2\}$, $d=(k/6\eps)^2$ and $m\leq d/2$ then there exists an $x \in \mathbb{R}^d$ such that
        \[
            \left| \frac{1}{d} \sum_{j=1}^{d} g(v_j x) - \frac{1}{m}\sum_{i=1}^m w_ig(a_i x)\right| > \varepsilon \left( \frac{1}{d}\sum_{j=1}^{d} g(v_j x)  + \frac{\RR(x)}{k}\right).
        \]
    \end{lemma}

    \begin{proof}
        Without loss of generality we assume that there are no samples $a_i=e_j$ for any $j \leq d/2$.
        Let $x=-\sum_{i=1}^{d/2}e_i/\sqrt{d}$. 
        We make the following observations:
        \begin{enumerate}
            \item 
            We have for both choices that $\RR(x)= 1$ since
            \begin{enumerate}
                \item $\RR(x)=\|x\|_2^2 = \sum_{i=1}^{d} (1/\sqrt{d})^2 = 1$\,,
                \item $\RR(x)=\|x\|_2 = \sqrt{\|x\|_2^2} = 1$\,.
            \end{enumerate}
            \item Also note for $j>d/2$ that $v_j^T x = 0$ and thus $g(v_j^T x)=g(0)=0$,\\
            by the assumption it also holds for all samples that $a_ix = 0$ resp. $g(a_ix)=g(0)=0$.
            \item Finally, for $j\leq d/2$ we have $v_j^T x = - 1/\sqrt{d}$ and \\thus $g(v_j^T x)=1/\sqrt{d}$.
        \end{enumerate}
        These observations together with the choice of $d=(k/6\eps)^2$ yield that
        \begin{align*}
            \left( \frac{1}{d} \sum_{j=1}^{d} g(v_j x) + \frac{\RR(x)}{k}\right)
            =\frac{1}{d}\frac{d}{2}\frac{1}{\sqrt{d}}+\frac{\RR(x)}{k}
            \leq \frac{6\eps}{2k} + \frac{1}{k} = \frac{2+6\eps}{2k}\,.
        \end{align*}
        Further we have that
        $$
            \frac{1}{m}\sum_{i=1}^m w_i g(a_i x)=\frac{1}{m}\sum_{i=1}^m w_i g(0) = 0\,.
        $$
        Using $\eps < 1/4$ we conclude that 
        \[
             \left| \frac{1}{m}\sum_{i=1}^m w_ig(a_i x)- \frac{1}{d}\sum_{j=1}^{d} g(v_j x) \right|
             = \left|\frac{1}{d}\sum_{j=1}^{d} g(v_j x)\right|
             = \frac{6\eps}{2k} > {\eps}\cdot\left( \frac{2+6\eps}{2k}\right).
        \]
    \end{proof}

\subsection{Linear lower bounds}\label{sec:LBlinear}
Let $\mathcal{P}$ be a distribution on $\mathbb{R}^{d}$,
$g\colon\mathbb{R}\to[0,\infty)$ a non–negative loss function, and $\mathcal{R}\colon \mathbb{R}^d\longleftarrow \mathbb{R}_{\geq 0}$ a regularizer function. 
For $k\in\mathbb{N}$ define
 $$
    f(x) = \underbrace{\int g(\langle a,x\rangle)\, d\mathcal{P}(a)}_{=:f_{0}(x)}
     + \frac{\mathcal{R}(x)}{k},
    \qquad x\in\mathbb{R}^{d}.
 $$
Given a non–negative sampling score $s(a)$ with
$S=\int s(a)\,d\mathcal{P}(a)<\infty$, we draw
$m$ points $a_{1},\dots,a_{m}\sim\mathcal{Q}$
where
$
d\mathcal{Q}(a)=\frac{s(a)}{S}\, d\mathcal{P}(a),
$
and attach the importance weights
$
w_{i}=w(a_{i})=\frac{S}{s(a_{i})}.
$
A \emph{relative–error approximation} of sample size $m$ must satisfy, for given
$\varepsilon,\delta\in(0,1)$,
 \begin{equation}\label{def:coreset}
    \Pr \Bigl[
       \forall x\in\mathbb{R}^{d}\colon
        \bigl| f_{0}(x)-\hat{f}_{0}(x)\bigr|
        \le\varepsilon f(x)
    \Bigr]
     \geq 1-\delta,
    \quad
    \hat{f}_{0}(x)=\frac1m\sum_{i=1}^{m}
                     w_{i} g \bigl(\langle a_{i},x\rangle\bigr).   
 \end{equation}

\thmlinearlowerbounds*
\begin{proof}
1. For ReLU loss with $\RR(\cdot)=\norm{\cdot}_1$ regularization, $m = \Omega (\eps^{-2} k\log k)$ is proven in \Cref{lem:relu-lb-klogk}. For hinge/logistic with the same regularizer, the bound follows by reduction from the case of ReLU, treated in \Cref{lem:reduction}.
2. For sigmoid loss with $\RR(\cdot)=\norm{\cdot}_1$, $m = \Omega (\eps^{-2} k/\log k)$ is proven in \Cref{lem:sigmoid-lb-all-regularization}. For logistic loss with $\RR(\cdot)=\norm{\cdot}_2^2$ regularization, we have the same bound proven in \Cref{lem:logistic-lb-all-regularization}
3. For sigmoid loss with $\RR(\cdot)=\norm{\cdot}_2^2$ regularization, $m = \Omega (\eps^{-2} k/(\log k)^3)$ is proven in \Cref{lem:sigmoid-lb-all-regularization}.
\end{proof}

\begin{lemma}\label{lem:reduction}
Assume that $a_1,\ldots,a_m$ satisfy Eq.~\ref{def:coreset} with a regularizer $\RR$ such that 
$\RR(rx) = r\,\RR(x)$ for all $r \geq 0$, and that 
\[
\lim_{\beta \to \infty} \frac{g(\beta t)}{\beta} = \hat{g}(t).
\]
Then $a_1,\ldots,a_m$ also satisfy Eq.~\ref{def:coreset} when $g$ is replaced by $\hat{g}$.
\end{lemma}
\begin{proof}
    Since $a_1,\ldots,a_m$ satisfy  Eq.~\ref{def:coreset}, we have 
     \begin{equation*}
    \Pr \Bigl[
       \forall x\in\mathbb{R}^{d}\colon
        \bigl| f_{0}(x)-\hat{f}_{0}(x)\bigr|
        \le\varepsilon f(x)
    \Bigr]
     \geq 1-\delta,   
 \end{equation*}
 where  
 $$f(x) = \underbrace{\int g(\langle a,x\rangle)\, d\mathcal{P}(a)}_{=:f_{0}(x)}
     + \frac{\RR(x)}{k},
    \qquad x\in\R^{d}$$ 
and 
 $\hat{f}_{0}(x)=\frac1m\sum_{i=1}^{m}
                     w_{i} g \bigl(\langle a_{i},x\rangle\bigr).$
Note that $\forall x\in\mathbb{R}^{d}\colon
        \bigl| f_{0}(x)-\hat{f}_{0}(x)\bigr|
        \le\varepsilon f(x)$ is equivalent to
        $\forall x\in\R^{d}\colon
        \bigl| \frac{1}{\beta}f_{0}(\beta x)-\frac{1}{\beta}\hat{f}_{0}(\beta x)\bigr|
        \le\varepsilon \frac{1}{\beta}f(\beta x)$.
        Now, sending $\beta\rightarrow\infty$ concludes the desired assertion. 
\end{proof}

Due to our reduction, we start with a lower bound against ReLU with $\RR(\cdot)\in\{\|\cdot\|_1,\|\cdot\|_2\}$ regularization.

\begin{lemma}\label{lem:relu-lb-klogk}
Fix $0<\varepsilon\le 1/4$ and $0<\delta\le 1/4$. Let $\RR\in\{\|\cdot\|_1,\|\cdot\|_2\}$ and define the ReLU loss $g(t)=[-t]_+=\max\{0,-t\}$. Then there exists a distribution $\mathcal P$ on $\mathbb R^d$ such that the following holds: if for an importance sampling we have 
$$\Pr\Big[\,\forall x\in\mathbb R^{d}:\ \bigl|\hat f_0(x)-f_0(x)\bigr|\ \le\ \varepsilon\, f(x)\,\Big]\ \ge\ 1-\delta\,,$$
then the sample size must satisfy
$$m\ \ge\ \Omega\left(\frac{k\log(k/\delta)}{\varepsilon^{2}}\right)\,.$$
\end{lemma}
\begin{proof}
For $d=k$ place equal mass on the $2k$ signed canonical basis vectors:
 $$
    \mathcal{P}\bigl\{+e_{j}\bigr\}
    =\mathcal{P}\bigl\{-e_{j}\bigr\}
     = \frac1{2k}\,,
    \qquad j=1,\ldots,k\,.
 $$
Write $S=1$ without loss of generality
(rescaling $s$ leaves $\mathcal{Q}$ and $w$ unchanged).
As we assumed $S=1$, we have 
$$\frac{1}{k}\sum\limits_{j=1}^k s(+e_j) + \frac{1}{k}\sum\limits_{j=1}^k s(-e_j)=2\,,$$
hence one of those two sums is at most $1$.
Set $\sigma\in\{+1,-1\}$ such that 
$$\frac{1}{k}\sum\limits_{j=1}^k s(\sigma e_j)\leq 1\,.$$
Define 
$$J=\{j\in[k]\colon s(\sigma e_j)\leq 2\}$$
and note that we must have $|J|\geq k/2$.
Let us check the approximation inequality at $x=-\sigma e_j$, for some $j\in J$:
$$g(\langle \sigma e_j, x\rangle) = g(-1) = 1\,, \qquad\qquad g(\langle -\sigma e_j, x\rangle) = g(1) = 0\,,$$
and $g(\langle \pm e_i, x\rangle) = 0$ for all $i\neq j$. 
Put these all together, for $\RR\in\{\|\cdot\|_1, \|\cdot\|_2\}$, we have 
$$f_0(x) = \frac{1}{2k}, \quad \frac{\RR(x)}{k} = \frac{1}{k}, \quad f(x)= \frac{3}{2k}\,.$$
Let $N_j$ be the number of occurrence of $\sigma e_j$ in the $m$ samples. Since $\QQ(\sigma e_j) = q_j = \frac{s(\sigma e_j)}{2k}$, we obtain 
$$N_j\sim {\rm Bionomial}(m, q_j)\,.$$
Set $\mu_j= mq_j$. For selected $x$, we have
$$\hat{f}_0(x)=\frac{1}{m}. \frac{1}{s(\sigma e_j)} N_j\,.$$
This indicates that the approximation property at $x$ is violated if and only if 
$$|\hat{f_0}(x)-f_0(x)| = \left|\frac{N_j}{ms(\sigma e_j)} -\frac{1}{2k}\right| > \eps f(x) = \eps\frac{3}{2k}\,.
$$
Using that $ms(\sigma e_j) = s(\sigma e_j) \mu_j/q_j = 2k \mu_j$ we have the condition $\left|\frac{N_j}{ms(\sigma e_j)} - \frac{1}{2k} \right| > \frac{3 \eps}{2k}$, multiplied by $2k \mu_j$, is equivalent to 
$$\left|N_j-\mu_j\right|> 3\eps \mu_j.$$
Let us call this event $F_j$ which can be interpreted as the event that our desired approximation is violated at $x=-\sigma e_j$. Also, define $F'_j$ as the event that 
$N_j-\mu_j> 3\eps \mu_j$ and notice $F'_j\subseteq F_j$, so $F'_j$ implies that the approximation fails. 

Note that for each $j\in J$, we have $s(\sigma e_j)\leq 2$ which concludes $q_j\leq \frac{1}{k}\leq \frac{1}{2}$ and thus $\frac{q_j}{1-q_j}\leq \frac{2}{k}$. 

Note that $\Var(N_j) = mq_j(1-q_j)$. Using Lemma~\ref{lem:feller} with $t= 3 \eps m q_j$ with $\eps \leq 1/600$, we obtain 
$$\Pr(F_j)\geq \Pr(F_j')\geq 
c\exp\left(-\frac{9\eps^2 m^2 q_j^2}{3mq_j(1-q_j)}\right) = c\exp\left(-3\eps^2\frac{ m q_j}{(1-q_j)}\right)\geq c\exp\left(-6\eps^2\frac{m}{k}\right) = p\,$$
which is independent of $j$. 

Define, for each $j\in J$,
$$F'_j \;=\; \{\,N_j > z_j\,\}\,,\qquad 
z_j = (1+3\varepsilon)\,\mu_j\,,\qquad \mu_j = \mathbb E[N_j]=m q_j\,.$$
Then
\begin{align*}
\Pr\Big(\bigcup_{j\in J} F'_j\Big)
&= 1 - \Pr\Big(\bigcap_{j\in J} (F'_j)^c\Big) \\
&\geq 1 - \prod_{j\in J} \Pr\big((F'_j)^c\big)
\qquad\text{\emph{(negative association of multinomial coordinates)}}\\
&= 1 - \prod_{j\in J} \bigl(1-\Pr(F'_j)\bigr) \\
&\geq 1 - (1-p)^{|J|} \\
&\geq 1 - (1-p)^{k/2}\,.
\end{align*}

The vector of counts $(N_j)_{j\in J}$ is a subcollection of a multinomial count vector, which is known to be \emph{negatively associated}. 
Here $(F'_j)^c=\{N_j\le z_j\}$ is decreasing in $N_j$, so
$\Pr\big(\bigcap_{j\in J} (F'_j)^c\big)\le \prod_{j\in J}\Pr\big((F'_j)^c\big)$ (see~\cite{Joag-Dev-Proschan1983}),
which yields the second line above. 
On the other hand 
$1-(1-p)^{\frac{k}{2}}\leq \Pr(\text{approximation failure})\leq \delta$ which concludes 
$$1-\delta\leq (1-p)^{\frac{k}{2}}\leq \exp\left(-\frac{kp}{2}\right)$$
and thus $kp\leq -2\ln (1-\delta).$
Plugging the value of $p$, we have
$$\exp\left(-6\eps^2\frac{m}{k}\right)\leq  \frac{-2\ln (1-\delta)}{ck} \Longrightarrow m \geq\frac{k}{6\eps^2}\ln\left(\frac{ck}{-2\ln (1-\delta)} \right)=\Theta\left(\frac{k\log(k/\delta)}{\eps^2}\right)\,.$$
\end{proof}

Using our previous reduction, the bound holds for logistic and hinge loss as well.

\begin{corollary}
\label{thm:lower-bound_logistic}
Fix $0<\varepsilon\le\frac14$ and $\delta\le\frac14$.
For either the 
logistic loss $g_{\mathrm{L}}(t)=\ln(1+e^{-t})$ or hinge-loss $g_{\mathrm{H}}(t)=\max\{0, 1-t\}$ and regularizer term $\RR(x)\in \{\|x\|_1, \|x\|_2\}$, 
there exists a distribution $\mathcal{P}$ on $\mathbb{R}^{d}$ such that
\emph{any} importance‑sampled approximation that attains
 $$
    \bigl|\hat{f}_{0}(x)-f_{0}(x)\bigr|
     \le \varepsilon f(x),
    \quad\forall x\in\mathbb{R}^{d},
 $$
with probability at least $1-\delta$
must have size
 $$
    m = \Omega \Bigl(\frac{k \log(k/\delta)}{\varepsilon^{2}}
              \Bigr).
 $$
\end{corollary}
\begin{proof}
    Note that $\lim\limits_{\beta\mapsto\infty} \frac{g_L(\beta t)}{\beta} = \lim\limits_{\beta\mapsto\infty} \frac{g_H(\beta t)}{\beta} = \max\{0,-t\}$. Now, using \Cref{lem:relu-lb-klogk} and \Cref{lem:reduction} we have completed the proof.
\end{proof}

We still need to cover the case of logistic regression with $\RR(\cdot)=\|\cdot\|_2^2$ regularization. We note that our general proof also covers other regularization variants, and that the proof outline can be adapted to sigmoid loss as well.
\begin{lemma}
\label{lem:logistic-lb-all-regularization}
Fix $0<\varepsilon\le\frac14$ and $\delta\le\frac14$. Consider the logistic loss $g(r) = \ln(1 + e^{-r})$.
For $p \in \{1,2\}$ and $q \in \{1,2\}$, set
\[
  \mathcal{R}^p_q(x) \coloneqq \frac{\|x\|_q^p}{k}, \qquad
  f_0(x) \coloneqq \mathbb{E}_{a \sim \mathcal{P}} \big[ g(\langle a, x \rangle) \big], \qquad
  f(x) \coloneqq f_0(x) + \RR^p_q(x).
\]

Then there exists a distribution $\mathcal{P}$ on $\R^{d}$ such that the following holds.
Any i.i.d.\ importance–weighted approximation $\hat f$ of size $m$ that satisfies
\[
  \Pr\Big[ \sup_{x \in \R^{d}} 
    \big| \hat f(x) - f(x) \big|
     \le  \varepsilon\, f(x) \Big]
   \ge  1 - \delta
\]
for every $k\ge 4$, accuracy parameter $\varepsilon \in (0,1)$, and failure probability
$\delta \in (0, 1/2]$, 
must obey
\[
  m  \ge 
  \begin{cases}
    \Omega \left(
      \dfrac{k}{\varepsilon^{2}}
      \,\log \left(\dfrac{k}{\delta}\right)
    \right)
      & \text{if } p = 1,\ q \in \{1,2\}, \\[1.1em]
    \Omega \left(
      \dfrac{k}{\varepsilon^{2} \log(k)}
      \,\log \left(\dfrac{1}{\delta}\right)
    \right)
      & \text{if } p = 2,\ q \in \{1,2\}.
  \end{cases}
\]
\end{lemma}
\begin{proof}
First note that since $\|\cdot\|_1\geq \|\cdot\|_2$, we only need to prove the statement for $q=1$. 
Work in $d=k+1$ and let $\PP$ be uniform on the $k$ atoms
$$p_i=(e_i,1)=e_i+e_{k+1}\in\R^{k+1}\qquad(i=1,\ldots,k).$$
An i.i.d.\ importance sampler uses nonnegative scores $s_i=s(p_i)$ normalized by $\frac1k\sum_{i=1}^k s_i=1$ and draws $a_1,\dots,a_m\sim\QQ$ with
$\QQ\{p_i\}=q_i=\frac{s_i}{k}.$
The unbiased importance–weighted estimator is
$$\hat f_0(x)=\frac{1}{m}\sum_{t=1}^m \frac{1}{s(a_t)}\,g(\langle a_t,x\rangle),
\qquad \E_{x\sim \QQ}[\hat f_0(x)]=f_0(x).$$
Let $Z_i:=|\{t:\ a_t=p_i\}|$ so $(Z_1,\dots,Z_k)\sim\mathrm{Mult}(m;q_1,\dots,q_k)$ and $Z_i\sim\mathrm{Bin}(m,q_i)$ with expectation $\mu_i=\E(Z_i)=m q_i$ and variance
$\sigma_i^2:=\Var(Z_i)=m q_i(1-q_i).$
Set $H:=\{i\in[k]:\ s_i\le 2\}$ and note that 
Markov’s inequality and (as rescaling $s$ leaves $\mathcal{Q}$ and $w$ unchanged, we may assume that) $S=\frac1k\sum_i s_i=1$ give $|H|\ge \frac{k}{2}.$
For $j\in H$, we have $q_j=s_j/k\le 2/k \leq 1/2$ and hence
\begin{equation}\label{eq:var-lb}
\sigma_j^2=m q_j(1-q_j)\ \ge\ \frac12\,\frac{m s_j}{k}.
\end{equation}
Fix $\alpha\ge 1$ to be chosen shortly. For each $j\in[k]$, define the adversarial query 
$$x^{(j)}=\alpha(-2e_j+e_{k+1}).$$
Then $\langle p_j,x^{(j)}\rangle=-\alpha$ and $\langle p_i,x^{(j)}\rangle=\alpha$ for $i\ne j$.
The logistic identity
\begin{equation}\label{eq:log-id}
    g(-r)-g(r)=r\qquad \forall r\in\R
\end{equation}
yields
\begin{equation}\label{eq:f0-at-xj}
f_0(x^{(j)})=\frac1k\,g(-\alpha)+\frac{k-1}{k}\,g(\alpha)=g(\alpha)+\frac{\alpha}{k}.
\end{equation}
Moreover $3\alpha=\|x^{(j)}\|_1\leq \|x^{(j)}\|_1^2=9\alpha^2$, hence for $p\in\{1,2\}$
\begin{equation}\label{eq:reg-at-xj}
\RR^p_1(x^{(j)})=\frac{\|x^{(j)}\|_1^p}{k}\ \le\ \frac{9\alpha^p}{k}.
\end{equation}
Let $\bar S_m=\frac1m\sum_{t=1}^m \frac{1}{s(a_t)}\ge 0$.
We may assume that $\bar S_m\in[1-\eps,1+\eps]$, otherwise the approximation fails on $x=0$ and we are done.
Using \eqref{eq:log-id} and $Z_i=\sum_t \mathbf 1\{a_t=p_i\}$,
$$
\hat f_0(x^{(j)})=\frac{Z_j}{m s_j}\,g(-\alpha)+\sum_{i\ne j}\frac{Z_i}{m s_i}\,g(\alpha)
\geq \alpha\,\frac{Z_j}{m s_j}+\bar S_m\,g(\alpha)\geq \alpha\,\frac{Z_j}{m s_j}+(1-\eps)\,g(\alpha)\,,
$$
which concludes 
\begin{equation}\label{eq:onesided}
\hat f_0(x^{(j)})- f_0(x^{(j)})
\ \ge\ \alpha\Big(\frac{Z_j}{m s_j}-\frac{1}{k}\Big)\ -\ \eps g(\alpha).
\end{equation}
In view of \eqref{eq:f0-at-xj} and \eqref{eq:reg-at-xj}, we have 
\begin{equation}\label{eq:f-upper}
f(x^{(j)})=f_0(x^{(j)})+\RR_1^p(x^{(j)})\ \le\ g(\alpha)+\frac{\alpha}{k}+\frac{9\alpha^p}{k}
\leq g(\alpha)+\frac{10\alpha^p}{k}
\end{equation}
Thus, we see that a \emph{sufficient} condition for
$$
\hat f(x^{(j)})-f(x^{(j)})>\varepsilon f(x^{(j)})
$$
is
$$
\alpha\Big(\frac{Z_j}{m s_j}-\frac{1}{k}\Big)-\eps g(\alpha)
\ \ge\ \varepsilon\Big(
g(\alpha)+\frac{10\alpha^p}{k}\Big),
$$
which is satisfied if $Z_j \ge m q_j + t_j$ for 
$$
t_j = 
\varepsilon mq_j\Big(
{2g(\alpha)k}
+ 10\alpha^{p-1}\Big).
$$
Choose 
$\alpha= \ln\left(\frac{1}{e^{1/k}-1}\right)$ and note that 
${g(\alpha)k}=1$ and  $\alpha\in\Big[\tfrac12\log k,\log k\Big]=\Theta(\log k)$. This implies 
$$t_j= \varepsilon mq_j\Big(
2
+ 10\alpha^{p-1}\Big)\leq 12\varepsilon mq_j\alpha^{p-1}.$$
Define the (increasing) event
$$
E_j=\{ Z_j\ge m q_j+ t_j \}.
$$
If $E_j$ occurs, the approximation fails at the single query $x^{(j)}$. Note $Z_j\sim\mathrm{Bin}(m,q_j)$ and $\sigma_j^2=mq_j(1-q_j)$.
Using $q_j=s_j/k$ and $1-q_j\ge \tfrac12$, we have $\sigma_j^2  \ge  \frac{m s_j}{2k}=\frac{m q_j}{2}$ and then (using $\alpha>1$)
\begin{align*}
    \frac{t_j}{\sigma_j^2}
  & \le  24\varepsilon \alpha^{p-1} \le  200,
\end{align*}
provided that 
$\varepsilon\leq \frac{25}{3\alpha^{p-1}}$.
Also, note that 
\begin{align*}
    \frac{t_j^2}{\sigma_j^2}
  & \le 288\varepsilon^2 \frac{m s_j}{k}\alpha^{2p-2}\\
  & \le 576 \varepsilon^2 \frac{m}{k}\alpha^{2p-2}.
\end{align*}
Hence, using Lemma~\ref{lem:feller}, for each $j\in H$,
\begin{equation}\label{eq:pstar}
\Pr(E_j)  \ge  p^*  =  c \exp\Big(-576 \varepsilon^2 \alpha^{2p-2} \frac{m}{k}\Big).
\end{equation}

The multinomial vector $(Z_1,\ldots,Z_k)$ is \emph{negatively associated} (see~\cite{Joag-Dev-Proschan1983}),
in particular, for increasing or decreasing events depending on disjoint coordinates we have
$$
\Pr\Big(\bigcap_{j\in H} A_j\Big)  \le  \prod_{j\in H}\Pr(A_j).
$$

Since each $E_j^c$ depends only on $Z_j$ and is decreasing, we obtain
$$
\Pr\Big(\bigcap_{j\in H} E_j^c\Big)  \le  \prod_{j\in H}\Pr(E_j^c)
  \le \Big(1- p^*\Big)^{|H|}\leq \Big(1- p^*\Big)^{k/2}.
$$
Accordingly, 
\[
\Pr\Big(\bigcup_{j\in H}E_j\Big) = 1-\Pr\Big(\bigcap_{j\in H} E_j^c\Big)
\ge\ 1-\Big(1- p^*\Big)^{k/2}.
\]
For the approximation to succeed with probability at least $1-\delta$ we must have
$$1-\Big(1- p^*\Big)^{k/2}\leq \delta,$$
which implies 
$$p^*\leq 1-(1-\delta)^{2/k}= 1-e^{-\frac{2}{k}\ln \frac{1}{1-\delta}}\leq \frac{2}{k}\ln \frac{1}{1-\delta}\leq \frac{2}{k} \frac{\delta}{1-\delta}\leq \frac{4\delta}{3k}$$
where the first two inequalities follow from $e^{x} \ge 1 + x$ for $x \in\R$ and the third from the assumption $\delta \le 1/2$.
Solving 
$$p^*\leq \frac{4\delta}{3k}$$
implies 
$$
m \ge \Omega\left(\frac{k}{\varepsilon^2 \alpha^{2p-2}} \ln \frac{k}{\delta}\right).
$$
\end{proof}

As indicated above, we adapt the proof to sigmoid loss as well.
\begin{lemma}
\label{lem:sigmoid-lb-all-regularization}
Fix $0<\varepsilon\le\frac14$ and $\delta\le\frac14$. Consider the sigmoid loss $g(r) = 1/(1 + e^{r})$.
For $p \in \{1,2\}$ and $q \in \{1,2\}$, set
\[
  \mathcal{R}^p_q(x) \coloneqq \frac{\|x\|_q^p}{k}, \qquad
  f_0(x) \coloneqq \mathbb{E}_{a \sim \mathcal{P}} \big[ g(\langle a, x \rangle) \big], \qquad
  f(x) \coloneqq f_0(x) + \RR^p_q(x).
\]

Then there exists a distribution $\mathcal{P}$ on $\R^{d}$ such that the following holds.
Any i.i.d.\ importance–weighted approximation $\hat f$ of size $m$ that satisfies
\[
  \Pr\Big[ \sup_{x \in \R^{d}} 
    \big| \hat f(x) - f(x) \big|
     \le  \varepsilon\, f(x) \Big]
   \ge  1 - \delta
\]
for every $k\ge 4$, accuracy parameter $\varepsilon \in (0,1)$, and failure probability
$\delta \in (0, 1/2]$, 
must obey
\[
  m  \ge 
  \begin{cases}
    \Omega \left(
      \dfrac{k}{\varepsilon^{2} \log(k)}
      \,\log \left(\dfrac{1}{\delta}\right)
    \right)
      & \text{if } p = 1,\ q \in \{1,2\}, \\[1.1em]
    \Omega \left(
      \dfrac{k}{\varepsilon^{2} \log^3(k)}
      \,\log \left(\dfrac{1}{\delta}\right)
    \right)
      & \text{if } p = 2,\ q \in \{1,2\}.
  \end{cases}
\]
\end{lemma}
\begin{proof}
First note that since $\|\cdot\|_1\geq \|\cdot\|_2$, we only need to prove the statement for $q=1$. 
Work in $d=k+1$ and let $\PP$ be uniform on the $k$ atoms
$$p_i=(e_i,1)=e_i+e_{k+1}\in\R^{k+1}\qquad(i=1,\ldots,k).$$
An i.i.d.\ importance sampler uses nonnegative scores $s_i=s(p_i)$ normalized by $\frac1k\sum_{i=1}^k s_i=1$ and draws $a_1,\dots,a_m\sim\QQ$ with
$\QQ\{p_i\}=q_i=\frac{s_i}{k}.$
The unbiased importance–weighted estimator is
$$\hat f_0(x)=\frac{1}{m}\sum_{t=1}^m \frac{1}{s(a_t)}\,g(\langle a_t,x\rangle),
\qquad \E_{x\sim \QQ}[\hat f_0(x)]=f_0(x).$$
Let $Z_i:=|\{t:\ a_t=p_i\}|$ so $(Z_1,\dots,Z_k)\sim\mathrm{Mult}(m;q_1,\dots,q_k)$ and $Z_i\sim\mathrm{Bin}(m,q_i)$ with expectation $\mu_i=\E(Z_i)=m q_i$ and variance
$\sigma_i^2:=\Var(Z_i)=m q_i(1-q_i).$
Set $H:=\{i\in[k]:\ s_i\le 2\}$ and note that 
Markov’s inequality and (as rescaling $s$ leaves $\mathcal{Q}$ and $w$ unchanged, we may assume that) $S=\frac1k\sum_i s_i=1$ give $|H|\ge \frac{k}{2}.$
For $j\in H$, we have $q_j=s_j/k\le 2/k \leq 1/2$ and hence
\begin{equation}\label{eq:var-lb2}
\sigma_j^2=m q_j(1-q_j)\ \ge\ \frac12\,\frac{m s_j}{k}.
\end{equation}
Fix $\alpha\ge 1$ to be chosen shortly. For each $j\in[k]$, define the adversarial query 
$$x^{(j)}=\alpha(-2e_j+e_{k+1}).$$
Then $\langle p_j,x^{(j)}\rangle=-\alpha$ and $\langle p_i,x^{(j)}\rangle=\alpha$ for $i\ne j$.

Choose 
$\alpha=\ln(k-1) > 1$ and note that 
${g(\alpha)k}=1$ and  $\alpha= \Theta(\log k)$.
The sigmoid identity
\begin{equation}
    g(-r)=1-g(r)\qquad \forall r\in\R
\end{equation}
yields that 
\begin{equation}\label{eq:sigm-id}
    h(r)=g(-r)-g(r)=1-2g(r) \qquad \forall r\in\R
\end{equation}
and thus
\begin{equation}\label{eq:sigm-id2}
    h(\alpha)=1-2g(\alpha)=1-\frac 2k \in \left[0.5,1\right]\,.
\end{equation}
Therefore we have that
\begin{equation}\label{eq:f0-at-xj2}
f_0(x^{(j)})=\frac1k\,g(-\alpha)+\frac{k-1}{k}\,g(\alpha)=\frac{1}{k} \big(g(-\alpha)-g(\alpha)\big) + g(\alpha) = \frac{h(\alpha)}{k} + g(\alpha).
\end{equation}
Moreover $3\alpha=\|x^{(j)}\|_1\leq \|x^{(j)}\|_1^2=9\alpha^2$, hence for $p\in\{1,2\}$
\begin{equation}\label{eq:reg-at-xj2}
\RR^p_1(x^{(j)})=\frac{\|x^{(j)}\|_1^p}{k}\ \le\ \frac{9\alpha^p}{k}.
\end{equation}
Let $\bar S_m=\frac1m\sum_{t=1}^m \frac{1}{s(a_t)}\ge 0$.
We may assume that $\bar S_m\in[1-\eps,1+\eps]$, otherwise the approximation fails on $x=0$ and we are done.
Using \eqref{eq:sigm-id} and $Z_i=\sum_t \mathbf 1\{a_t=p_i\}$,
\begin{align*}
\hat f_0(x^{(j)})=\frac{Z_j}{m s_j}\,g(-\alpha)+\sum_{i\ne j}\frac{Z_i}{m s_i}\,g(\alpha)
&= \frac{Z_j}{m s_j} \big(g(-\alpha)-g(\alpha)\big) + \bar S_m\,g(\alpha)\\
&\geq \frac{Z_j}{m s_j} h(\alpha)+(1-\eps)\,g(\alpha)\,,
\end{align*}
which concludes by \eqref{eq:sigm-id2} that
\begin{equation}\label{eq:onesided2}
\hat f_0(x^{(j)})- f_0(x^{(j)})
\ \ge\ \Big(\frac{Z_j}{m s_j}-\frac{1}{k}\Big) h(\alpha) - \eps g(\alpha) \ge \Big(\frac{Z_j}{2 m s_j}-\frac{1}{2k}\Big) - \eps g(\alpha) .
\end{equation}
In view of \eqref{eq:f0-at-xj} and \eqref{eq:reg-at-xj}, and using \eqref{eq:sigm-id2}, we have 
\begin{equation}\label{eq:f-upper2}
f(x^{(j)})=f_0(x^{(j)})+\RR_1^p(x^{(j)})\ \le\ g(\alpha)+\frac{h(\alpha)}{k}+\frac{9\alpha^p}{k}
\leq g(\alpha)+\frac{10\alpha^p}{k}
\end{equation}
Thus, we see that a \emph{sufficient} condition for
$$
\hat f(x^{(j)})-f(x^{(j)})>\varepsilon f(x^{(j)})
$$
is
$$
\Big(\frac{Z_j}{2 m s_j}-\frac{1}{2k}\Big) - \eps g(\alpha)
\ \ge\ \varepsilon\Big(
g(\alpha)+\frac{10\alpha^p}{k}\Big),
$$
which is satisfied if $Z_j \ge m q_j + t_j$ for 
$$
t_j = 
\varepsilon mq_j\Big(
{4g(\alpha)k}
+ 20\alpha^{p}\Big).
$$
This implies 
$$t_j= \varepsilon mq_j\Big(
4
+ 20\alpha^{p}\Big)\leq 24\varepsilon mq_j\alpha^{p}.$$
Define the (increasing) event
$$
E_j=\{ Z_j\ge m q_j+ t_j \}.
$$
If $E_j$ occurs, the approximation fails at the single query $x^{(j)}$. Note $Z_j\sim\mathrm{Bin}(m,q_j)$ and $\sigma_j^2=mq_j(1-q_j)$.
Using $q_j=s_j/k$ and $1-q_j\ge \tfrac12$, we have $\sigma_j^2  \ge  \frac{m s_j}{2k}=\frac{m q_j}{2}$ and then (using $\alpha>1$)
\begin{align*}
    \frac{t_j}{\sigma_j^2}
  & \le  48\varepsilon \alpha^{p} \le  200,
\end{align*}
provided that 
$\varepsilon\leq \frac{4}{\alpha^{p}}$.
Also, note that 
\begin{align*}
    \frac{t_j^2}{\sigma_j^2}
  & \le 1152\varepsilon^2 \frac{m s_j}{k}\alpha^{2p}\\
  & \le 2304 \varepsilon^2 \frac{m}{k}\alpha^{2p}.
\end{align*}
Hence, using Lemma~\ref{lem:feller}, for each $j\in H$,
\begin{equation}\label{eq:pstar2}
\Pr(E_j)  \ge  p_*  =  c \exp\Big(-2304 \varepsilon^2 \alpha^{2p} \frac{m}{k}\Big).
\end{equation}

The multinomial vector $(Z_1,\ldots,Z_k)$ is \emph{negatively associated} (see~\cite{Joag-Dev-Proschan1983}),
in particular, for increasing or decreasing events depending on disjoint coordinates we have
$$
\Pr\Big(\bigcap_{j\in H} A_j\Big)  \le  \prod_{j\in H}\Pr(A_j).
$$

Since each $E_j^c$ depends only on $Z_j$ and is decreasing, we obtain
$$
\Pr\Big(\bigcap_{j\in H} E_j^c\Big)  \le  \prod_{j\in H}\Pr(E_j^c)
  \le \Big(1- p_*\Big)^{|H|}\leq \Big(1- p_*\Big)^{k/2}.
$$
Accordingly, 
\[
\Pr\Big(\bigcup_{j\in H}E_j\Big) = 1-\Pr\Big(\bigcap_{j\in H} E_j^c\Big)
\ge\ 1-\Big(1- p_*\Big)^{k/2}.
\]
For the approximation to succeed with probability at least $1-\delta$ we must have
$$1-\Big(1- p_*\Big)^{k/2}\leq \delta,$$
which implies 
$$p^*\leq 1-(1-\delta)^{2/k}= 1-e^{-\frac{2}{k}\ln \frac{1}{1-\delta}}\leq \frac{2}{k}\ln \frac{1}{1-\delta}\leq \frac{2}{k} \frac{\delta}{1-\delta}\leq \frac{4\delta}{3k}$$
where the first two inequalities follow from $e^{x} \ge 1 + x$ for $x \in\R$ and the third from the assumption $\delta \le 1/2$.
Solving 
$$p^*\leq \frac{4\delta}{3k}$$
implies 
$$
m \ge \Omega\left(\frac{k}{\varepsilon^2 \alpha^{2p}} \ln \frac{k}{\delta}\right).
$$
\end{proof}

\subsection{Impossibility of dimension-free or sublinear sampling for ReLU with \texorpdfstring{$\norm*{\cdot}_2^2$}{L2 squared} regularization}\label{sec:LBimpossible}

We first show a simple $\Omega(d\log d)$ lower bound against the combination of ReLU and ridge regularization $\RR(x)=\|x\|_2^2$. Then we establish an $\Omega(N \log N)$ lower bound with more sophisticated techniques. We summarize both results in the following theorem.

\thmReLUsquared*

\begin{proof}
    The theorem follows directly from \Cref{lem:ReLUdlogd} for $n=d$, and \Cref{lem:ReLUNlogN} for $n\geq d+1$.
\end{proof}

\begin{lemma}\label{lem:ReLUdlogd}
    Consider the ReLU function where $g(r)=\max\{0,-r\}$, and the $\RR(x)=\|x\|_2^2$ regularizer. Then there exists a distribution such that with high probability $m=\Omega(d\log d)$ uniform or row-norm samples are required to satisfy
    \[
        \forall x \in \mathbb{R}^d\colon \left|f_0(x)-\left(\frac{1}{m}\sum_{i=1}^m w_i g(a_i x)\right) \right| \leq \varepsilon f(x)\,.
    \]
\end{lemma}
\begin{proof}
    Consider the uniform distribution over the $d$ standard unit vectors, so $P(a = e_j)= 1/d$ for each $j\in[d]$. Set $x=\alpha e_i$ for some small scalar $\alpha>0$ and index $i\in[d]$ yet to be determined.
    
    Irrespective of the choice of parameters $\alpha$ and $i$, the original loss and regularization for $x$ equals
    \[
        \frac{1}{d} \sum_{j=1}^{d} g(e_jx) + \frac{\RR(x)}{k} = \frac{\alpha}{d} + \frac{\alpha^2}{k}\,.
    \]
    However, by the classic coupon collector's theorem \cite{ErdosR61}, if we take less than $m\leq \Omega(d\log d)$ samples, then with high probability some vector $e_{i^*}$ will not be present, and we set $i=i^*$ in the definition of $x$. Then every sample satisfies $g(a_ix) = 0$ and the loss degrades to $0 + \alpha^2/k$.
    \[  
        \frac{1}{m} \sum_{j=1}^{m} g(e_jx) + \frac{\RR(x)}{k} = 0 + \frac{\alpha^2}{k}\,.
    \]
    
    Now choosing $\alpha \coloneqq {2k}/{3d}$ and using $\eps < 1/2$, we get that 
    \[
        \eps \left(\frac{\alpha}{d} + \frac{\alpha^2}{k}\right)
        < \frac{k}{3d^2} + \frac{2k}{9d^2} = \frac{5k}{9d^2} < \frac{2k}{3d^2} = \frac{\alpha}{d}\,.
    \]
    Thus the error on $x$ is
    \[
        \left|\frac{\alpha}{d} + \frac{\alpha^2}{k} - \left(0 + \frac{\alpha^2}{k}\right)\right| = \frac{\alpha}{d} > \eps \left(\frac{\alpha}{d} + \frac{\alpha^2}{k}\right)\,.
    \]    
\end{proof}

\begin{lemma}\label{lem:ReLUNlogN}
 Consider the ReLU function where $g(r)=\max\{0,-r\}$, and the $\RR(x)=\|x\|_2^2$ regularizer. For any fixed $d$ and arbitrary $N\geq d+1$, there exists a distribution supported on $N$ points with uniform mass for which with high probability $m = \Omega(\eps^{-2}N\log N)$ samples are required to satisfy
    \[
        \forall x \in \mathbb{R}^d\colon \left|f_0(x)-\left(\frac{1}{m}\sum_{i=1}^m w_i g(a_i x)\right) \right| \leq \varepsilon f(x)\,.
    \]
\end{lemma}
\begin{proof}
Consider $N$ points $a_1,\dots,a_N$ on the moment curve $\gamma(t) = (1,t,t^2,\ldots,t^d)$ with distinct $t$’s. 
Place equal mass on these points:
 $$
    \mathcal{P}\bigl\{a_{j}\bigr\}
     = \frac1{N}\,,
    \qquad j=1,\ldots,N.
 $$
Write $S=1$ without loss of generality
(rescaling $s$ leaves $\mathcal{Q}$ and $w$ unchanged).
Define 
$$J=\{j\in[N]\colon s(a_j)\leq 2\}$$
and note that we must have $|J|\geq N/2$.

We know that every $a_j$ is a vertex of the cyclic polytope $\mathrm{conv}\{a_1,\dots,a_N\}$.
So, there must be a supporting hyperplane with a normal $x_j$ for which we have 
$$ \langle a_j,x_j\rangle<0,\qquad \langle a_i,x_j\rangle\ge0\quad (\forall i\ne j).$$
For ReLU $g(t)=\max\{0,-t\}$, this gives 
$$g(\langle a_j,x_j\rangle) = c_j>0,\qquad\text{and}\qquad g(\langle a_i,x_j\rangle) = 0\qquad\text{ for all } i\neq j.$$
Let $N_j$ be the number of occurrence of $a_j$ in the $m$ samples. 

We can now consider scaling the input set $\{x_1, \ldots, x_N\}$ uniformly down by a factor $\eta \in (0,1]$.  
We will ultimately drive $\eta \mapsto 0$; while we do this, the ReLu part in $f_0$ scales linearly with $\eta$, but because the residual has a squared term it decreases quadratically, and can be factored out of the analysis. As expected, without a residual (and no other conditions), there can be no useful sampling bound.

In particular, observe that 
$$f_0(\eta x_j)=\frac{\eta c_j}{N},\qquad
\hat{f}_0(\eta x_j)=\frac{\eta c_j}{ms(a_j)}N_j,\qquad
\frac{\|\eta x_j\|_2^2}{k}=\frac{\eta^2\|x_j\|_2^2}{k}.$$
Since $\QQ(a_j) = q_j = \frac{s(\sigma e_j)}{N}$, we obtain 
$$N_j\sim {\rm Bionomial}(m, q_j)\,.$$
So, if we write out the approximation inequality, we get
\begin{align*}
    \left|
\frac{\eta c_j}{ms(a_j)}N_j - \frac{\eta c_j}{N}
\right|\leq \eps\left(\frac{\eta c_j}{N}+\frac{\eta^2\|x_j\|_2^2}{k}\right)\qquad
 &\xLongrightarrow{\times \tfrac{1}{\eta c_j}} &
\left|
\frac{N_j}{ms(a_j)} - \frac{1}{N}
\right|\leq \eps\left(\frac{1}{N}+\frac{\eta \|x_j\|_2^2}{kc_j}\right)\\
&\xLongrightarrow{\eta \mapsto 0} &
\left|
\frac{N_j}{ms(a_j)} - \frac{1}{N}
\right|\leq \eps\frac{1}{N}\\
&\xLongrightarrow{\times ms(a_j)} &
\left|
N_j - \mu_j
\right|\leq \eps\mu_j.
\end{align*}
Define $F_j$ as the event $\left|
N_j - \mu_j
\right|> \eps\mu_j$. Note that $F_j$ implies that the approximation inequality fails. Also, define $F'_j$ as the event that $N_j-\mu_j> \eps \mu_j$ and notice $F'_j\subseteq F_j$, so $F'_j$ implies failure of obtaining the desired approximation.

Note that for each $j\in J$, we have $s(a_j)\leq 2$ which concludes $q_j\leq \frac{2}{N}\leq \frac{1}{2}$ and thus $\frac{q_j}{1-q_j}\leq \frac{4}{N}$. 

Note that $\Var(N_j) = mq_j(1-q_j)$. Using Lemma~\ref{lem:feller} with $t= \eps m q_j$ with $\eps \leq 1/600$, we obtain 
$$\Pr(F_j)\geq \Pr(F_j')\geq 
c\exp\left(-\frac{\eps^2 m^2 q_j^2}{3mq_j(1-q_j)}\right) = c\exp\left(-\eps^2\frac{ m q_j}{3(1-q_j)}\right)\geq c\exp\left(-\eps^2\frac{4}{3}\frac{m}{N}\right) = p\,$$
which is independent of $j$. Define, for each $j\in J$,
$$F'_j \;=\; \{\,N_j > z_j\,\}\,,\qquad 
z_j = (1+\varepsilon)\,\mu_j\,,\qquad \mu_j = \mathbb E[N_j]=m q_j\,.$$
Then
\begin{align*}
\Pr\Big(\bigcup_{j\in J} F'_j\Big)
&= 1 - \Pr\Big(\bigcap_{j\in J} (F'_j)^c\Big) \\
&\geq 1 - \prod_{j\in J} \Pr\big((F'_j)^c\big)
\qquad\text{\emph{(negative association of multinomial coordinates)}}\\
&= 1 - \prod_{j\in J} \bigl(1-\Pr(F'_j)\bigr) \\
&\geq 1 - (1-p)^{|J|} \\
&\geq 1 - (1-p)^{N/2}\,.
\end{align*}

The vector of counts $(N_j)_{j\in J}$ is a subcollection of a multinomial count vector, which is known to be \emph{negatively associated}. 
Here $(F'_j)^c=\{N_j\le z_j\}$ is decreasing in $N_j$, so
$\Pr\big(\bigcap_{j\in J} (F'_j)^c\big)\le \prod_{j\in J}\Pr\big((F'_j)^c\big)$ (see~\cite{Joag-Dev-Proschan1983}),
which yields the second line above. 
On the other hand 
$1-(1-p)^{\frac{N}{2}}\leq \Pr(\text{approximation failure})\leq \delta$ which concludes 
$$1-\delta\leq (1-p)^{\frac{N}{2}}\leq \exp\left(-\frac{Np}{2}\right)$$
and thus $Np\leq -2\ln (1-\delta).$
Plugging the value of $p$, we have
$$\exp\left(-\frac43\eps^2\frac{m}{N}\right)\leq  \frac{-2\ln (1-\delta)}{cN} \Longrightarrow m \geq\Theta\left(\frac{N\log N}{\eps^2}\right).$$
\end{proof}

\section{Technical lemmas}
The following classic anti-concentration result is due to Feller and will help in proving lower bounds.

\begin{lemma}[\citealp{Feller43}]\label{lem:feller}
        Let $Z$ be a sum of independent random variables, each attaining values in $[0,1]$, and let $\sigma=\sqrt{\Var(Z)}\geq 200$. Then for all $t \in [0, \frac{\sigma^2}{100}]$ we have
        \[
            \Pr[ Z\geq \E[Z]+t] \geq c \cdot \exp( -t^2/(3\sigma^2) )\,,
        \]
        where $c>0$ is some fixed constant.
\end{lemma}

\begin{lemma}\label{lem:expmaxGaussian}
    Let $Z_1,\ldots,Z_m\sim\mathcal N(0,\sigma_i)$ be $m$ Gaussians with $\sigma_i\leq \sigma_{\max}$. Then $\E(\max_{i\in[m]}|Z_i|) = O(\sigma_{\max}\sqrt{\log m})$, and $\E(\max_{i\in[m]}|Z_i|^2) = O(\sigma_{\max}^2{\log m})$.
\end{lemma}
\begin{proof}
    For the first claim, let $\lambda\coloneqq \frac{\sqrt{\ln m}}{\sigma_{\max}}$. We have by Jensen's inequality and the well-known moment generating function of Gaussians $\E\left(\exp\left(\lambda Z_j\right)\right) = \exp\left(\sigma_j^2\lambda^2/2\right)$, that
    \begin{align*}
        \exp\left(\lambda \E\left(\max_{j\in[m]} Z_j\right)\right)
        &\leq \E\left(\exp\left(\lambda \max_{j\in[m]} Z_j\right)\right) = \E\left(\max_{j\in[m]} \exp\left(\lambda Z_j\right)\right) \\
        &\leq \E\left(\sum_{j\in[m]} \exp(\lambda Z_j)\right)
        = \sum_{j\in[m]} \E\left(\exp\left(\lambda Z_j\right)\right) \\
        &= \sum_{j\in[m]} \exp\left(\sigma_j^2\lambda^2/2\right)
        \leq m\cdot\exp\left(\sigma_{\max}^2\lambda^2/2\right)\,.
    \end{align*}
    Now, taking the natural logarithm on both sides and rearranging, we get
    \begin{align*}
        &\E\left(\max_{j\in[m]} Z_j\right) 
        \leq \frac{\ln m}{\lambda} + \frac{\sigma_{\max}^2\lambda}{2} 
        = \sigma_{\max}\sqrt{\ln m} + \frac{\sigma_{\max}\sqrt{\ln m}}{2} \leq 2\sigma_{\max}\sqrt{\ln m}\,.
    \end{align*}
    By symmetry, we have the same bound for the collection of random variables $-Z_j, j\in[m]$. Thus the bound follows for $\max_{j\in[m]} |Z_j| = \max_{j\in[m]} \max\{Z_j, -Z_j\} = \max\{\max_{j\in[m]} Z_j , \max_{j\in[m]} -Z_j\} \leq \max_{j\in[m]} Z_j + \max_{j\in[m]} -Z_j \leq 2\max_{j\in[m]} Z_j$ up to another factor $2$ as well.

    For the second claim, set $\lambda\coloneqq \frac{1}{4\sigma_{\max}^2}$. Note that $|Z_j|^2 = Z_j^2$ follow a squared Gaussian distribution. We repeat the above argument using the appropriate moment generating function $\E\left(\exp\left(\lambda Z_j\right)\right) = \left(1-2\sigma_i^2\lambda\right)^{-1/2} \leq \left(1-2\sigma_{\max}^2\lambda\right)^{-1/2}$ to obtain
    \begin{align*}
        \exp\left(\lambda \E\left(\max_{j\in[m]} Z_j\right)\right)
        & \leq m\cdot\left(1-2\sigma_{\max}^2\lambda\right)^{-1/2}\,.
    \end{align*}
    Taking the natural logarithm on both sides and rearranging, we get
    \begin{align*}
        &\E\left(\max_{j\in[m]} Z_j\right) 
        \leq \frac{\ln m}{\lambda} + \frac{1}{2\lambda} \ln\left( \frac{1}{1-2\sigma_{\max}^2\lambda}\right) 
        = 4\sigma_{\max}^2{\ln m} + {2\sigma_{\max}^2} \ln\left( 2 \right)
        \leq 6\sigma_{\max}^2{\ln m}\,.
    \end{align*}
\end{proof}

\begin{lemma}\label{lem:expmaxsquared2}
    For $i\in[m]$, let $X_i\sim N(0,\sigma_i^2), Z\sim N(0,\sum_{i=1}^\ell \sigma_i^2)$. Then $\E(\max_{i\in[m]} |X_i|^2) \leq \E(|Z|^2)$.
\end{lemma}
\begin{proof}
    We have that
    \begin{align*}
        \E\left(\max_{i\in[m]} |X_i|^2 \right)
        =\E\left(\max_{i\in[m]} X_i^2 \right)
        \leq \E\left(\sum_{i\in[m]} X_i^2 \right) 
        =\sum_{i\in[m]} \E( X_i^2 )
        =\sum_{i\in[m]} \sigma_i^2
        = \E(Z^2) = \E(|Z|^2).
    \end{align*}\qedhere
\end{proof}

The following technical lemma will be helpful for bounding the entropy integral.

\begin{lemma}[\citealp{woodruffyasuda23}]\label{lem:calc}
Let $0 < \lambda \leq 1$. Then,
\[
    \int_0^\lambda \sqrt{\log\frac1t}\,dt = \lambda \sqrt{\log(1/\lambda)} + \frac{\sqrt\pi}{4}\erfc(\sqrt{\log(1/\lambda)}) \leq \lambda \parens*{\sqrt{\log(1/\lambda)} + \frac{\sqrt\pi}2}
\]
in particular for $0 < \lambda \leq 1/3$, we have that $\int_0^\lambda \sqrt{\log\frac1t}~dt \leq 2\lambda \parens*{\sqrt{\log(1/\lambda)}}$
\end{lemma}

Our proofs use the following Lipschitz contraction bound from \citep{AlishahiP24}, which generalizes Theorem 4.12 of \citet{LT1991}.

\begin{lemma}[\citealp{LT1991,AlishahiP24}]\label{lem:expsup}
        Let ${T}\subset \mathbb{R}^m$ be a bounded set. For any $L$-Lipschitz functions $\phi_i\colon \mathbb{R}\rightarrow \mathbb{R}$ with $\phi_i(0) = 0$ for each $i\in[m]$, we have 
    
    $$\underset{\sigma \sim \mathcal{U}\{-1, 1\}^m}{\mathbb{E}}
    \sup_{t\in {T}}\left|
    \sum_{i=1}^m \sigma_i \phi_i(t_i)
    \right|\leq 
    2L\underset{\sigma \sim \mathcal{U}\{-1, 1\}^m}{\mathbb{E}}
    \sup_{t\in {T}}\left|
    \sum_{i=1}^m \sigma_i t_i
    \right|.$$
\end{lemma}

One can find the following proposition in \cite{KHARE2023127545}. 
\begin{proposition}[\citealp{Kahane:CRASP-259-2577,Latala1994}]\label{prop:kahane}
For all $p, q \in [1, \infty)$, there exists a universal constant $C_{p,q} > 0$ depending only on $p,q$, such that for all choices of Banach spaces $\mathbb{B}$ equipped with norm $\norm*{\cdot}$, finite sets of vectors $x_1, \dots, x_n \in \mathbb{B}$, and independent Rademacher variables $r_1, \dots, r_n$,
\begin{equation*}
\left[\mathbb{E} \left\|\sum_{k=1}^n r_k x_k \right\|^q\right]^{1/q}
    \leq C_{p,q} \, \left[\mathbb{E} \left\|\sum_{k=1}^n r_k x_k \right\|^p\right]^{1/p}.
\end{equation*}
If moreover $1 = p \le q \le 2$, then the constant $C_{1,q} = 2^{1 - 1/q}$ is optimal.
\end{proposition}

The following corollary will be used in our derivations.

\begin{corollary}\label{cor:khintchine}
    Let $x_1,\ldots,x_n\in\mathbb{R}^d$, and let $\sigma_1,\ldots,\sigma_n$ be i.i.d. Rademacher random variables. Let $\ell>2$ be an even integer. Then it holds that
\begin{equation}
\mathbb{E} \left\|\sum_{i=1}^n \sigma_i x_i \right\|_2^\ell
    \leq {\ell}^{\ell/2} \left[\sum_{i=1}^n \|x_i\|^2_2\right]^{\ell/2}.
\end{equation}
\end{corollary}
\begin{proof}
We would like to apply \Cref{prop:kahane}. We are interested in the case where $\mathbb{B}=(\R^d, \norm{\cdot}_2)$ and $p=2 < q = \ell$, where $\ell=2\lceil\log(\delta^{-1})\rceil$ is an even integer to be used in the exponent of our moment bound. It is known that in this case $C_{2,\ell}= \sqrt{2} (\Gamma((\ell+1)/2)/\sqrt{\pi})^{1/\ell}$~\cite{Haagerup1981}.

We first show that $C_{2,\ell}^{\ell}\leq {\ell}^{\ell/2}$.
To this end, we use the fact that $\Gamma(1/2)=\sqrt{\pi}$, and for any real number $x$ it holds that $\Gamma(x+1)=x\Gamma(x)$. Since $\ell$ is an even non-negative integer, we have that
\begin{align*}
    C_{2,\ell}^{\ell} = \frac{2^{\ell/2}}{\sqrt{\pi}} {\Gamma\left(\frac{{\ell}+1}{2}\right)}
    &= \frac{2^{\ell/2}}{\sqrt{\pi}} {\prod_{i=1}^{\ell/2} \left(\frac{\ell}{2}-i+\frac{1}{2}\right)}\Gamma\left(\frac{1}{2}\right) 
    = {2}^{\ell/2} \prod_{i=1}^{\ell/2} \left(\frac{\ell}{2}-i+\frac{1}{2}\right) \\
    &= {2}^{-\ell/2} \prod_{i=1}^{\ell/2} \left(\ell-(2i-{1})\right) 
    = {2}^{-\ell/2} \frac{(\ell-1)!}{\prod_{i=1}^{\ell/2-1} \left(\ell-2i\right)} \\
    &= {2}^{-\ell/2} \frac{\ell \cdot (\ell-1)!}{2^{\ell/2}\frac{\ell}{2}\cdot \prod_{i=1}^{\ell/2-1} \left(\frac{\ell}{2}-i\right)} 
    = {2}^{-\ell} \frac{\ell!}{\frac{\ell}{2}!} 
    \leq \ell^{\ell/2} \,.
\end{align*}
Using \Cref{prop:kahane}, we conclude that
\begin{equation}
\mathbb{E} \left\|\sum_{i=1}^n \sigma_i x_i \right\|^\ell
    \leq {\ell}^{\ell/2} \left(\mathbb{E} \left\|\sum_{i=1}^n \sigma_i x_i \right\|_2^2 \right)^{\ell/2} =  {\ell}^{\ell/2} \left(\sum_{i=1}^n \|x_i\|^2_2\right)^{\ell/2}.
\end{equation}
To see that the equality holds, note that $\sigma_i$ are independent zero-mean and unit-variance random variables. This implies for $i\neq j$ that $\E(\sigma_i\sigma_j)=\E(\sigma_i)\E(\sigma_j)=0$. Thus only squared terms survive and we can use ${\E(\sigma_i^2)}=1$ on these squared terms.
\end{proof}

\end{document}